\documentclass[twocolumn,10pt,journal,twoside]{IEEEtran}
\usepackage[cmex10]{amsmath}
\usepackage{amssymb}
\usepackage{amsfonts}
\usepackage{amsthm}
\usepackage{algorithm}
\usepackage{algpseudocode}
\usepackage{graphicx}
\usepackage{epsfig}
\usepackage{cite}
\usepackage{tensor}
\usepackage[caption=false,font=footnotesize]{subfig}
\usepackage{color}

\usepackage{hyperref}
\hypersetup{hypertex=true,
            colorlinks=true,
            linkcolor=blue,
            anchorcolor=blue,
            citecolor=blue}

\usepackage{makecell}
\usepackage{float}

\usepackage{amsmath}

\graphicspath{{figures/}}
\DeclareGraphicsExtensions{.eps}
\newtheorem{proposition}{Proposition}
\newtheorem{lemma}{Lemma}
\theoremstyle{definition}
\newtheorem{remark}{Remark}
\newtheorem{assumption}{Assumption}

\newtheorem{definition}{Definition}

\newcommand*\diff{\mathop{}\!\mathrm{d}}
\DeclareMathOperator{\sgn}{sgn}

\DeclareMathOperator{\sig}{sig}
\DeclareMathOperator{\diag}{diag}

\renewcommand{\baselinestretch}{0.97}

\definecolor{user_color}{RGB}{190,10,10}

\begin{document}

\title{Contour Moments Based Manipulation of Composite Rigid-Deformable Objects with Finite Time Model Estimation and Shape/Position Control}

\author{Jiaming Qi, 
        Guangfu Ma,
        Jihong Zhu,~\IEEEmembership{Member,~IEEE,}
        Peng Zhou,~\IEEEmembership{Student Member,~IEEE,}
        Yueyong Lyu,\\
        Haibo Zhang
        and David Navarro-Alarcon,~\IEEEmembership{Senior~Member,~IEEE}
        
        \thanks{This work is supported in part by: the Key-Area Research and Development Program of Guangdong Province 2020 under project 76, the Hong Kong Research Grants Council under grant 14203917, and the PROCORE France/Hong Kong Joint Research Scheme sponsored by the RGC and the Consulate General of France in Hong Kong under grant F-PolyU503/18.}%
        
        \thanks{J. Qi, G. Ma and Y. Lyu are with Astronautics School, Harbin Institute of Technology, Heilongjiang, China, 150001.}
        
        \thanks{H. Zhang is with Beijing Institute of Control Engineering, Beijing, China.}
        
        \thanks{J. Zhu is with Cognitive Robotics, TU Delft, the Netherlands.}
        
        \thanks{D. Navarro-Alarcon and P. Zhou are with the Department of Mechanical Engineering, The Hong Kong Polytechnic University (PolyU), Kowloon, Hong Kong. D Navarro-Alarcon is also with the Research Institute for Smart Ageing (RISA), PolyU. Corresponding author's email: {\texttt{\small dna@ieee.org}}}
}
\bstctlcite{IEEEexample:BSTcontrol}

\markboth{IEEE/ASME Transactions on Mechatronics}%
{Qi \MakeLowercase{\textit{et al.}}: Contour Moments Based Manipulation of Composite Rigid-Deformable Objects with Finite Time Control}
\maketitle

\begin{abstract}
The robotic manipulation of composite rigid-deformable objects (i.e. those with mixed non-homogeneous stiffness properties) is a challenging problem with clear practical applications that, despite the recent progress in the field, it has not been sufficiently studied in the literature.
To deal with this issue, in this paper we propose a new visual servoing method that has the capability to manipulate this broad class of objects (which varies from soft to rigid) with the same adaptive strategy.
To quantify the object's infinite-dimensional configuration, our new approach computes a compact feedback vector of 2D contour moments features.
A sliding mode control scheme is then designed to simultaneously ensure the finite-time convergence of both the feedback shape error and the model estimation error.
The stability of the proposed framework (including the boundedness of all the signals) is rigorously proved with Lyapunov theory.
Detailed simulations and experiments are presented to validate the effectiveness of the proposed approach.
To the best of the author's knowledge, this is the first time that contour moments along with finite-time control have been used to solve this difficult manipulation problem.
\end{abstract}

\begin{IEEEkeywords}
	Robotics, Visual-Servoing, Deformable Objects, Sliding Mode Control, Contour Moments.
\end{IEEEkeywords}
\IEEEpeerreviewmaketitle

\section{INTRODUCTION}
\IEEEPARstart{T}{he} manipulation of  composite rigid-deformable objects (CRDO) is currently an open research problem in robotics (see our recent survey \cite{Journals:Zhu2021_RAM}).
The ubiquitous nature of CRDO motivates the development of suitable manipulation strategies which can be used in various application scenarios, e.g., 
food industry \cite{tokumoto2002deformation}, 
robot-surgery assistance \cite{han2020vision}, 
cable assembly \cite{park2005static} and household works \cite{2011Bringing}.
To automate these types of tasks, there are three main technical problems:
(i) Efficient feedback representation of shapes;
(ii) Estimation of the sensorimotor model of the robot-object system.
(iii) Design of shaping controls that can timely minimize the deformation error.
Although great progress has been achieved in this problem in recent years, the development of control methods for CRDO has not been suffiently studied in the literature.
Our aim in this work is to provide a solution to the above issues.

A critical issue in \emph{shape servoing} (i.e. the active shaping of a soft object by means of robot motions) is to design a low-dimensional feature that describes the infinite-dimensional shape in an efficient manner \cite{Wang2018Adaptive}.
Four geometric features (point, distance, angle, and curvature) were used in \cite{navarro2016automatic} to characterize deformable objects, however, these basic features can only represent a limited type of shapes.
Contour-based descriptors have been utilized for this purpose \cite{2004Review}. 
These methods compute features based on binary (intensity-based) information of an image contour, which mimics the way humans visually distinguish objects.
Truncated Fourier series was used in \cite{Navarro2018Fourier} to represent and control shapes; This idea was generalized in \cite{qi2020adaptive} in the form of parameterized regression shape features.
The classical Hu moments method \cite{hu1962visual} was extended in \cite{2009Error} to the case of visual contours by replacing the region integral with a curvilinear integral, which results in a reduced computation cost. Note that this promising approach has not yet been used in shape servoing.

To automatically manipulate CRDO, it is essential to know how the robot's motion result in changes of the object's shape; This relation is captured by the so-called deformation Jacobian matrix (DJM) \cite{Journals:Zhu2021_RAM}. 
Due to the complexity of CRDO, the exact (i.e. analytical) calculation of the DJM might be difficult to obtain, a situation that is further complicated by uncertainties in the object's mechanical properties.
Therefore, most works make use of numerical algorithms to estimate this structure in real-time.
Some examples of this approach include the Broyden update rules \cite{alambeigi2018autonomous}, online optimization methods \cite{2020Automatic}, flow-based adaptive algorithms \cite{navarro2016automatic}, deep neural networks \cite{Hu20193}\cite{li2018vision}, etc.
Note that these previous methods have been used to estimate the DJM of objects with ``mostly'' elastic properties.
In \cite{zhu2020vision}, a shape servoing controller is proposed for both rigid and deformable objects that estimates the DJM with a least-squares method over a sliding window.
However, in this and most previous works, the stability of such online algorithms is not rigorously proved (the design of their update rules is generally decoupled from the motion control laws).
This condition limits the robustness of existing methods.

The design of standard visual servoing methods (including shape servoing ones) typically relies on imposing an asymptotically stable equilibrium of the visual-guided task \cite{han2021visual} \cite{cherubini2021_frontneuro}.
However, this closed-loop stability property does not ensure that the feedback shape error will converge to zero in a finite amount of time. This is clearly undesirable for soft object manipulation tasks that demand precisely shaping actions with a deterministic convergence (e.g. in delicate surgical procedures).
Combining finite-time control with sliding mode control (SMC) \cite{2018An} is a feasible solution to deal with these issues, as it enables to specify the convergence of the error and is robust to model and parameter uncertainty \cite{2020Toward}.
Despite its valuable (and controllable) stability properties, this advanced control technique has not yet been utilized in soft object manipulation tasks \cite{Journals:Zhu2021_RAM}.

In this paper, we address the design of a visual servoing method for manipulating CRDO with a dual-arm robot.
For that, we propose an innovative finite-time control scheme that ensures the robust minimization of both the shape error (which characterizes the feedback manipulation task) and the model estimation error (which represents the estimation of the DJM).
A self-tuning adaptive update law is also presented to improve the approximation accuracy of DJM.
Although SMC has been previously used in visual servoing tasks \cite{ahmadi2021robust,li2010sliding}, this is the first time (to the best of the authors' knowledge) that it has been used in a soft object manipulation task.
The main contributions of the paper are fourfold:
\begin{enumerate}
    \item Our approach is not only suitable for elastic objects, but also for objects with anisotropic and time-varying physical properties, such as CRDO.
    \item Contour moments are used (for the first time) as a feedback signal to represent the shape of CRDO.
    \item A Finite-Time SMC (FTSMC) method is proposed (and rigorously analysed) to simultaneously perform the shape servoing task and to estimate the DJM.
    \item Detailed simulations, experiments and quantitative comparisons are conducted to validate our new method.
\end{enumerate}

\begin{figure}[h]
    \centering
    \includegraphics[scale=0.37]{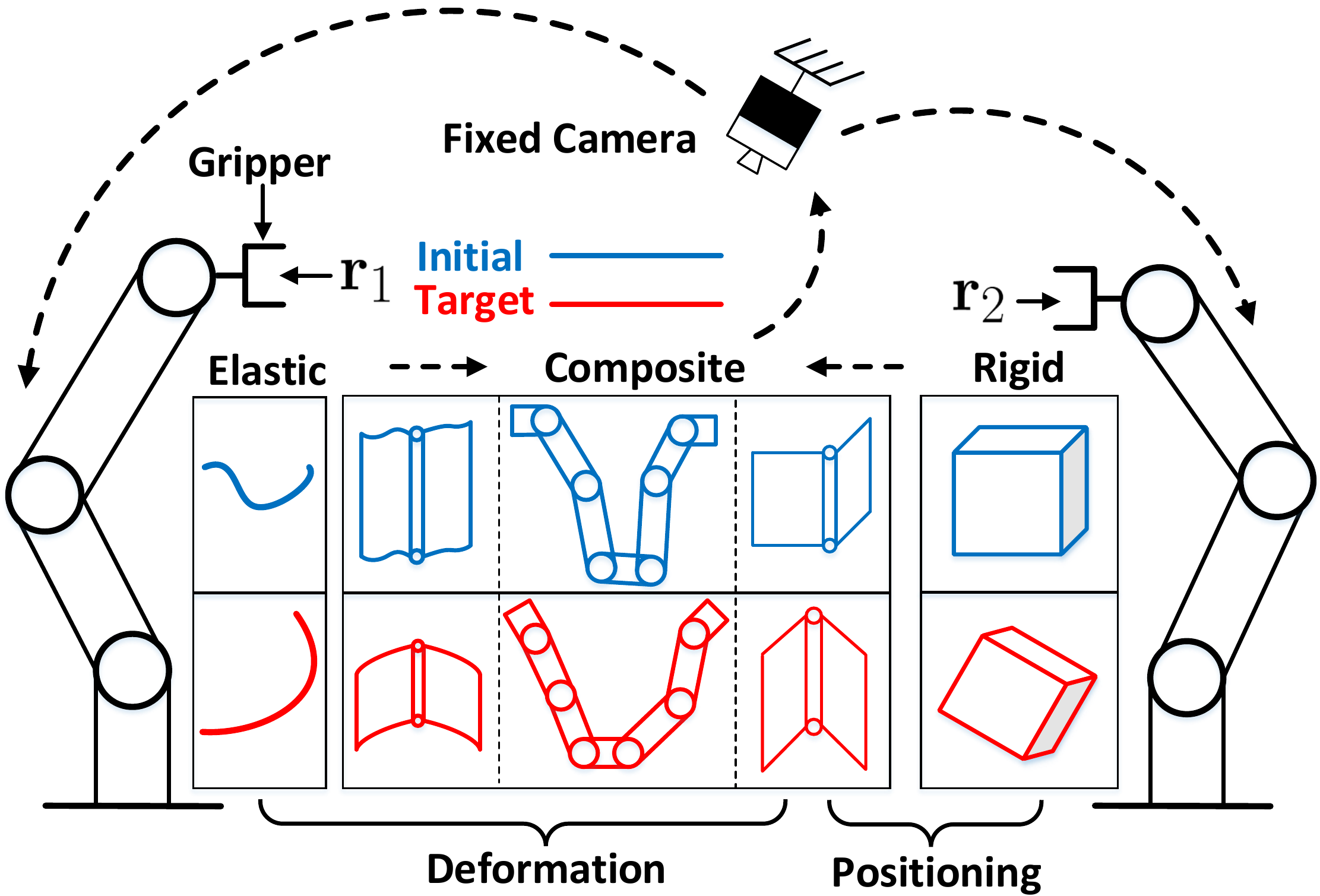}
    \caption{Representation of dual-arm robot manipulating CRDO. 
    In general, the manipulation tasks includes deformation tasks and positioning tasks. The goal is to design the velocity controller to command dual-arm robot to manipulate CRDO into the desired configurations.}
    \label{fig31}
\end{figure}


\begin{figure}[h]
    \centering
    \includegraphics[scale=0.22]{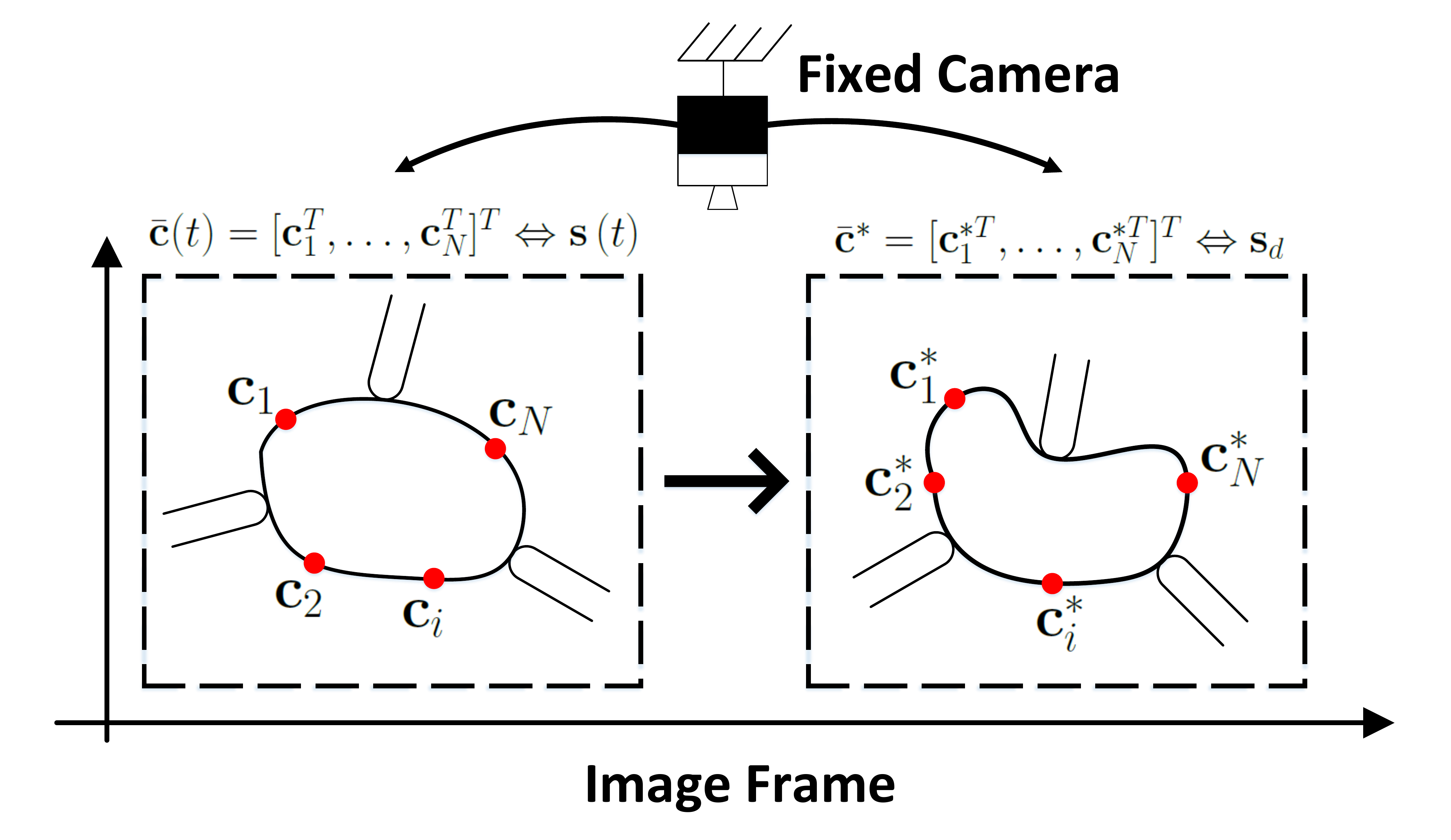}
    \caption{Contour description in the 2D manipulation space.
    The contour is defined in the image frame, where $\bar{\mathbf{c}}$ represents the real-time feedback contour and $\bar{\mathbf{c}}^*$ represents the desired contour.}
    \label{fig32}
\end{figure}

\section{PROBLEM FORMULATION}\label{sect2}
In this work, the vision-based manipulation task considered in the work is conceptually depicted in Fig. \ref{fig31}.
Column vectors are denoted with bold small letters $\mathbf{v}$, and matrices with bold capital letters $\mathbf{M}$; The symbol $(\cdot)^+$ represents the matrix pseudo-inverse.


\subsection{Dual-Arm Robot Model}
Consider the configuration of a dual-arm robot, where the vector of joint angles and end-effector pose of each robot are denoted by $\mathbf{q}_i \in \mathbb{R}^d$ and $\mathbf{r}_i=\mathbf{r}_i(\mathbf{q}_i) \in \mathbb{R}^q$, respectively, where $i=1,2$ and $q \leq d$.
We define the augmented vectors ${\mathbf{r} = {{\left[ {\mathbf{r}_1^T ,\mathbf{r}_2^T} \right]}^T} \in {\mathbb{R} ^{2q}}}$ and ${\mathbf{q} = {{\left[{\mathbf{q}_1^T,\mathbf{q}_2^T} \right]}^T} \in {\mathbb{R} ^{2d}}}$.
The differential kinematics of the dual-arm robot are ${\dot{\mathbf{r}} = \mathbf{J}\left( \mathbf{q} \right)\dot{\mathbf{q}}}$,
where $\mathbf{J}\left( \mathbf{q} \right) = \diag \left\{ {\frac{{\partial {\mathbf{r}_1}}}{{\partial {\mathbf{q}_1}}},\frac{{\partial {\mathbf{r}_2}}}{{\partial {\mathbf{q}_2}}}} \right\} \in {\mathbb{R}^{2q \times 2d}}$ represents the standard Jacobian matrix, which is assumed to be exactly known.
In this article, we consider that the robots are kinematically-controlled \cite{navarro2013model} (i.e. that the robot's joint positions/velocities can be exactly determined).
Furthermore, we assume that the trajectories of the robot are free from collisions with either the environment or amongst its arms.



\subsection{Visual-Deformation Model}
\label{sec2b}
Let us define the state of the CRDO as $\mathbf{m} \in \mathbb{R}^{\theta}$.
The relation between the robot's pose $\mathbf r$ and $\mathbf m$ is represented by the (unknown)
nonlinear mapping $\mathbf m = \mathbf f_m(\mathbf r)$, which models the mechanical properties of the CRDO, e.g., deformable structures, rigid parts, and composites of both, etc.
In this article, the object's image contour ${\bar{\mathbf{c}}  = {{\left[ {{\mathbf{c}}_1^T, \ldots ,{\mathbf{c}}_N^T} \right]}^T} \in {{\mathbb{R}}^{2N}}}$, for ${{{\mathbf{c}}_i} = {{\left[ {{u_i},{v_i}} \right]}^T} \in {\mathbb{R}^2}}$, is utilized to represent the configuration of the CRDO, where $N$ represents the number of points that comprise the contour, and $u_i$ and $v_i$ represent the pixel coordinates of the ith $(i=1,\ldots,N)$ point in the image frame, see Fig. \ref{fig32}.
Thus, the relation between $\bar{\mathbf{c}}$ and $\mathbf{m}$ can be described by the nonlinear mapping, $\bar{\mathbf{c}} = \mathbf{f}_c (\mathbf{m})$ (which captures the perspective geometry of the object's points in the boundary).

To actively control the shape of CRDOs, the robot's end-effectors must rigidly grasp the object in advance.
Motion of the controllable grippers' pose $\mathbf r$ result in changes in the object's image contour $\bar{\mathbf{c}}$.
We model this visual-motor relation with the following expression:
\begin{equation}
\bar{\mathbf{c}} = \mathbf{f}_c( \mathbf f_m ( \mathbf{r})) = 
\mathbf{f}_c(\mathbf{f}_m(\mathbf{r}_1,\mathbf{r}_2))
\end{equation}

Note that the dimension $2N$ of the observed contour $\bar{\mathbf c}$ is generally large, thus, it is inefficient to directly use it in a shape controller as it contains redundant information.
In our approach, we use the contour information $\bar{\mathbf c}$ to construct a compact feature vector, here denoted by $\mathbf{s} \in \mathbb{R}^p$, for $p\ll2N$, that characterizes the object's shape but with significantly fewer feedback coordinates.
The relation between the robot's pose and such shape descriptor is modeled as follows:
\begin{equation}
\mathbf{s} = \mathbf{f}_s(\bar{\mathbf{c}}) = \mathbf{f}_s(\mathbf     f_c(\mathbf{f}_m(\mathbf{r}_1,\mathbf{r}_2)))
\label{eq51}
\end{equation}
By computing the time derivative of \eqref{eq51}, we obtain the first-order dynamic model:
\begin{equation}
\label{eq52}
\dot{\mathbf{s}} = \frac{{\partial \mathbf{f}_s}}{{\partial {\mathbf{r}_1}}}{{\dot{\mathbf{r}}}_1} + \frac{{\partial \mathbf{f}_s}}{{\partial {\mathbf{r}_2}}}{{\dot{\mathbf{r}}}_2} = {\mathbf{J}_s}\left( \mathbf{r} \right)\dot{\mathbf{r}}
\end{equation}
where ${\mathbf{J}_s}\left( \mathbf{r} \right) = \left[ {\frac{{\partial \mathbf{f}_s}}{{\partial \mathbf{r}_1}},\frac{{\partial \mathbf{f}_s}}{{\partial \mathbf{r}_2}}} \right] \in {\mathbb{R}^{p \times 2q}}$ represents DJM which relates the velocity of the end-effectors with the shape feature changes.
As the mechanical properties of CRDOs are complex and difficult to obtain, these Jacobian matrices need to be numerically estimated.

\textbf{Problem statement.}
Design a velocity-based controller for a dual-arm robot to automatically deform the composite rigid-deformable object towards the desired image shape (characterized by the feature vector $\mathbf s_d$) while simultaneously estimating the unknown matrix $\mathbf J_s (\mathbf{r})$.


\subsection{Mathematical Properties}
Before presenting our main result, some important properties are here introduced.


\begin{lemma}
\label{lemma3}
\cite{van2016finite}
Define the function $\sig^k(x) = |x|^k \sgn(x)$, which holds for any $x \in \mathbb{R}$, where $k>0$, and satisfies the relation:
\begin{equation}
\begin{array}{*{20}{c}}
{\frac{\diff}{{\diff t}}\sig^k ( x ) = k{{\left| x \right|}^{k - 1}}\dot x,}&{k
\ge 1}
\end{array}
\end{equation}
\end{lemma}



\begin{lemma}
\label{lemma1}
\cite{polycarpou1993robust}
The inequality $0 \le \left| x \right| - x\tanh \left({{x}/{\varepsilon}} \right) \le \delta \varepsilon$ holds for any $\varepsilon > 0$ and for any $x \in \mathbb{R}$, where $\delta=0.2785$ is a constant that satisfies $\delta=e^{-(\delta+1)}$.
\end{lemma}

\begin{assumption}
\label{assumption1}
The DJM is divided into two components, namely, 
$\mathbf{J}_s(\mathbf{r}) = \hat{\mathbf{J}}_s  + \tilde{\mathbf{J}}_s$, where $\hat{\mathbf{J}}_s$ is the estimation of $\mathbf{J}_s(\mathbf{r})$ and $\tilde{\mathbf{J}}_s$ is the approximation error.
\end{assumption}

\begin{assumption}
\label{assumption2}
The approximation error of the DJM is bounded 
${|| \tilde{\mathbf{J}}_s ||}^2 \le \beta$, for $\beta$ as a unknown positive constant.
\end{assumption}

\begin{definition}
\label{definition1}
Given an arbitrary vector $\mathbf{x} \in \mathbb{R}^n$, let us introduce the following vectorial power definitions \cite{van2016finite}:
\begin{align*}
    \sig^k\left( \mathbf{x} \right) &= \left[ {{\sig^k}\left( {{x_1}} \right), \ldots ,\sig^k\left( {{x_n}} \right)} \right]^T \in {\mathbb{R}^n}\\
    {\left| \mathbf{x} \right|^k} &= \diag\left\{ {{{\left| {{x_1}} \right|}^k}, \ldots ,{{\left| {{x_n}} \right|}^k}} \right\} \in {\mathbb{R}^{n \times n}}\\
    \end{align*}
\end{definition}


\section{FEATURE EXTRACTION ALGORITHM}\label{sect3}
In this article, we use the contour moments \cite{lei2012fusion} of $\bar{\mathbf{c}}$ to construct the compact feature vector $\mathbf s$.
For digital images, the ordinary moment (of order $i+j$) is defined as follows:  
\begin{equation}
\label{eq2}
\begin{array}{*{20}{c}}
{{h_{ij}} = \sum\limits_{k = 1}^N {{u_k^i}{v_k^j}\Delta m_k}},&{i,j = 0,1, \ldots }
\end{array}
\end{equation}
where $N$ is the number of points on the contour and $\Delta {m_k} = \left\| {{\mathbf{c}_k} - {\mathbf{c}_{k - 1}}} \right\|$ represents the distance between two adjacent pixels on the contour.
The central moment of order $i+j$ of the contour is given as:
\begin{equation}
\label{eq3}
\begin{array}{*{20}{c}}
{{\eta_{ij}} = \sum\limits_{k = 1}^N {{{\left( {{u_k} - \bar u} \right)}^i}{{\left( {{v_k} - \bar v} \right)}^j}\Delta {m_k}}},&{i,j = 0,1, \ldots }
\end{array}
\end{equation}
where $\bar{u}$ and $\bar{v}$ are the central coordinates, i.e., $\bar{u} = h_{10} / h_{00}$ and $\bar{v} = h_{01} / h_{00}$.


Similar to the classical Hu moments \cite{hu1962visual}, the following seven contour moments are computed: 
\begin{align}
\label{eq6}
\bar{\varphi}_1 &= {\eta _{20}} + {\eta _{02}}, ~~\bar{\varphi}_2 = {\left( {{\eta _{20}} - {\eta _{02}}} \right)^2} + 4\eta _{11}^2 \notag \\
\bar{\varphi}_3 &= {\left( {{\eta _{30}} - 3{\eta _{12}}} \right)^2} + {\left( {3{\eta _{21}} - {\eta _{03}}} \right)^2} \notag \\
\bar{\varphi}_4 &= {\left( {{\eta _{30}} + {\eta _{12}}} \right)^2} + {\left( {{\eta _{03}} + {\eta _{21}}} \right)^2} \notag \\
\bar{\varphi}_5 &= \left( {{\eta _{30}} - 3{\eta _{12}}} \right)\left( {{\eta _{30}} + {\eta _{12}}} \right)\left( \begin{array}{l}
{\left( {{\eta _{30}} + {\eta _{12}}} \right)^2}\\
- 3{\left( {{\eta _{21}} + {\eta _{03}}} \right)^2}
\end{array} \right) \notag \\
&+ \left( {3{\eta _{21}} - {\eta _{03}}} \right)\left( {{\eta _{21}} + {\eta _{03}}} \right)\left( \begin{array}{l}
3{\left( {{\eta _{30}} + {\eta _{12}}} \right)^2}\\
- {\left( {{\eta _{21}} + {\eta _{03}}} \right)^2}
\end{array} \right) \notag \\
\bar{\varphi}_6 &= \left( {{\eta _{20}} - {\eta _{02}}} \right)\left( {{{\left( {{\eta _{30}} + {\eta _{12}}} \right)}^2} - {{\left( {{\eta _{21}} + {\eta _{03}}} \right)}^2}} \right)  \\
&+ 4{\eta _{11}}\left( {{\eta _{12}} + {\eta _{30}}} \right)\left( {{\eta _{21}} + {\eta _{03}}} \right) \notag \\
\bar{\varphi}_7 &= \left( {3{\eta _{21}} - {\eta _{03}}} \right)\left( {{\eta _{12}} + {\eta _{30}}} \right)\left( \begin{array}{l}
{\left( {{\eta _{12}} - {\eta _{30}}} \right)^2} \notag \\
- 3{\left( {{\eta _{21}} + {\eta _{03}}} \right)^2}
\end{array} \right) \notag \\
&- \left( {{\eta _{30}} - 3{\eta _{12}}} \right)\left( {{\eta _{21}} + {\eta _{03}}} \right)\left( \begin{array}{l}
3{\left( {{\eta _{12}} + {\eta _{30}}} \right)^2} \notag \\
- {\left( {{\eta _{21}} + {\eta _{03}}} \right)^2}
\end{array} \right) \notag
\end{align}
which are invariant to translation and rotation.
Note, however, that $\bar{\varphi}_k$ in \eqref{eq6} have a very large magnitude. 
A common approach to rectify these variables is to use its logarithmic form:
\begin{equation}
\label{eq44}
\begin{array}{*{20}{c}}
{\bar{s}_k = \left| {\log \left( {\left| {\bar{\varphi}_k} \right|} \right)} \right|,}&{k \in \left[ {1,7} \right]}
\end{array}
\end{equation}

Note, however, that the contour moments $\bar{s}_k$ above (for $k=1,\ldots,7$) describe the object's shape in a rotation, translation, and scale invariant manner.
For the considered bi-manual shape servoing task, such information might not be sufficient to deform and position the object into a deisred configuration.
Therefore, further incorporate the object's centroid coordinates $\bar{u}$ and $\bar{v}$ (which encodes position) within the shape feature descriptor:
\begin{equation}
\label{eq63}
\begin{array}{*{20}{c}}
     \bar{s}_8=\bar{u} ,& \bar{s}_9=\bar{v} 
\end{array}
\end{equation}
The angle of the contour's principal axis (which encodes orientation in the plane) \cite{zheng2019toward}, is also used in the shape descriptor:
\begin{equation}
\label{eq61}
\bar{s}_{10} = \frac{1}{2}{\tan ^{ - 1}}\left( {\frac{{2{\eta_{11}}}}{{{\eta_{20}} - {\eta_{02}}}}} \right)
\end{equation}

As the magnitudes of the ten shape coordinates $\bar{s}_k, k$ are different, we normalize them in a $\pm 1$ range.
For the contour moments, this is done as:
\begin{equation}
\label{eq62}
\begin{array}{*{20}{c}}
{s_i} =\frac{{\bar{s}_i - \frac{1}{7}\sum\limits_{i = 1}^7 {\bar{s}_i} }}{\max \left( {\bar{s}_i} \right) - \min \left( {\bar{s}_i} \right)}
,&{i=1,\ldots,7}
\end{array}
\end{equation}
For centroid coordinates as:
\begin{equation}
\label{eq66}
\begin{array}{*{20}{c}}
s_8 = \frac{2\bar{s}_8 - C_w}{C_w}, &
s_9 = \frac{2\bar{s}_9 - C_h}{C_h}
\end{array}
\end{equation}
where $C_w$ and $C_h$ represent the width and height of the camera, e.g., $640 \times 320$ and $1920 \times 1080$.
For the principal axis angle, it is simply done as:
\begin{equation}
\label{eq65}
    {s}_{10} =  {\bar{s}_{10}}/{\pi}
\end{equation}
Finally, the total shape feature vector is constructed as follows: $\mathbf{s}=[s_1, s_2, s_3, s_4, s_5, s_6, s_7, s_8, s_9, s_{10}]^T \in \mathbb{R}^{10}$.

As seen from \eqref{eq2}--\eqref{eq65}, these features $s_k$ only depends on the geometric image contour. 
Slight illumination and contrast changes will not significantly affect its computation. 
Therefore, these features represent a more robust alternative than the classical image moments \cite{hu1962visual}, as they replace the calculation of the region integral by a curvilinear integral.
The computation of these proposed feedback shape features is given in Algorithm \ref{algorithm2}.


\begin{algorithm}
\caption{Shape feature $\mathbf{s}$ calculation process} 
\label{algorithm2} 
\begin{algorithmic}[1] 
\Require Fixed-size contour $\bar{\mathbf{c}}$;
\State Calculate the difference-operator $\Delta m_k, k=1,\cdots,N$;
\State Calculate the ordinary moment $h_{ij} \gets$ \eqref{eq2};
\State Calculate the central moment $\eta_{ij} \gets $ \eqref{eq3};
\State Calculate contour moments $\bar{\varphi} _1, \cdots \bar{\varphi}_7 \gets$ \eqref{eq6};
\State Data compression is conducted in \eqref{eq44};
\State Calculate central coordinates $\gets$ \eqref{eq63};
\State Calculate principal angle $\gets$ \eqref{eq61};
\State Normalize shape feature $\gets$ 
       \eqref{eq62} \eqref{eq66} \eqref{eq65};
\State \Return Shape feature $\mathbf{s}$;
\end{algorithmic} 
\end{algorithm}

\section{CONTROLLER DESIGN}
\label{sec4}
Two methods are here described to control the shape of the CRDO: Linear sliding mode control (LSMC) and finite-time sliding mode control (FTSMC).
In the rest of this paper, we denote the robot's control input as $\mathbf u = \dot{\mathbf{r}}$.

\subsection{Linear Sliding Mode Control}
\label{sec4a}
Let us first define the error variables:
\begin{equation}
\label{eq55}
\begin{array}{*{20}{c}}
{{\mathbf{e}_1} = \mathbf{s} - \mathbf{s}_d}, &{{\mathbf{e}_2} = \dot{\mathbf{s}} - \hat{\mathbf{J}}_s \mathbf{u}}
\end{array}
\end{equation}
and its time derivatives
\begin{equation}
\label{eq7}
\begin{array}{*{20}{c}}
{{{\dot{\mathbf{e}}}_1} = \dot{\mathbf{s}} - {{\dot{\mathbf{s}}}_d}},&{{{\dot{\mathbf{e}}}_2} = \ddot{\mathbf{s}} - \dot{\hat{\mathbf{J}}}_s \mathbf{u} - \hat{\mathbf{J}}_s \dot{\mathbf{u}}}
\end{array}
\end{equation}
which we use to construct the linear sliding surfaces: \cite{utkin1977variable}:
\begin{equation}
\label{eq8}
\begin{array}{*{20}{c}}
{{\mathbf{\sigma} _1} = {\mathbf{K}_1}{\mathbf{e}_1} + {{\dot{\mathbf{e}}}_1}},&
{{\mathbf{\sigma} _2} = {\mathbf{K}_2}{\mathbf{e}_2} + {{\dot{\mathbf{e}}}_2}}
\end{array}
\end{equation}
for $\mathbf{s}_d$ as the desired shape feature, and $\mathbf{K}_k$ as symmetric positive-definite constant matrices. 
Considering \eqref{eq7} and Assumption \ref{assumption1}, we can compute the time derivative of $\sigma_1$ as:
\begin{equation}
\label{eq9}
{\dot{\mathbf{\sigma}} _1} = {\mathbf{K}_1}\hat{\mathbf{J}}_s \mathbf{u} + {\mathbf{K}_1}\tilde{\mathbf{J}}_s \mathbf{u} - {\mathbf{K}_1}{\dot{\mathbf{s}}_d} + {\ddot{\mathbf{e}}_1}
\end{equation}
and design the following velocity control input:
\begin{equation}
\label{eq10}
\mathbf{u} = {\hat{\mathbf{J}}^+_s} \mathbf{K}_1^{ - 1}\left( { - {\mathbf{\sigma} _1} + {\mathbf{K}_1}{{\dot{\mathbf{s}}}_d} - {{\ddot{\mathbf{e}}}_1}} \right)
\end{equation}
where $\hat{\mathbf{J}}_s^+$ denotes the pseudo-inverse of the adaptive matrix $\hat{\mathbf{J}}_s$.
To quantify the shape tracking error, we introduce the quadratic function ${V_1(\mathbf{\sigma}_1)} = \frac{1}{2}\mathbf{\sigma} _1^T{\mathbf{\sigma} _1}$, 
whose time-derivative satisfies:
\begin{align}
\label{eq13}
{{\dot V}_1}(\mathbf{\sigma}_1)
= \mathbf{\sigma} _1^T{{\dot{\mathbf{\sigma}}  }_1} 
= - \mathbf{\sigma} _1^T{\mathbf{\sigma} _1} + \mathbf{\sigma} _1^T{\mathbf{K}}_1\tilde{\mathbf{J}}_s \mathbf{u}
\end{align}
From Assumption \ref{assumption2} and considering Young’s inequality, we can obtain the following relation:
\begin{equation}
\label{eq14}
\mathbf{\sigma} _1^T{\mathbf{K}_1}\tilde{\mathbf{J}}_s \mathbf{u} \le {\lambda_{\mathbf{K}_1}}{\left\| {{\mathbf{\sigma} _1}} \right\|^2}/{4} + \beta {\left\| \mathbf{u} \right\|^2}
\end{equation}
where $\lambda_{\mathbf{K}_1}$ denotes the maximum eigenvalue of $\mathbf{K}_1$.
Substitution of \eqref{eq14} into \eqref{eq13} yields:
\begin{equation}
\label{eq53}
\dot{V}_1(\mathbf{\sigma}_1)
\le  - \left( {1 - {\lambda_{\mathbf{K}_1}}/{4}} \right){\left\| {{\mathbf{\sigma} _1}} \right\|^2} + \beta {\left\| \mathbf{u} \right\|^2}
\end{equation}
With this method, the DJM is adaptively computed as:
\begin{align}
\label{eq15}
\dot{\hat{\mathbf{J}}}_s  &= \left( {\ddot{\mathbf{s}} - \hat{\mathbf{J}}_s \dot{\mathbf{u}} - \mathbf{K}_2^{ - 1}\mathbf{\varpi}} \right){\mathbf{u}^ + }  \\
\mathbf{\varpi} &= { - {\mathbf{\sigma} _2} - {{\ddot{\mathbf{e}}}_2} - \mathbf{\sigma} _2^{T + }\tanh \left( {{{{{\left\| \mathbf{u} \right\|}^2}}} / {\chi }} \right)\hat \beta {{\left\| \mathbf{u} \right\|}^2}}  \notag
\end{align}
where the variable $\hat\beta$ is updated with the adaptive rule:
\begin{equation}
\label{eq49}
\begin{array}{*{20}{c}}
{\dot{\hat \beta}  = \tanh \left( {{{{{\left\| \mathbf{u} \right\|}^2}}}/{\chi}} \right){{\left\| \mathbf{u} \right\|}^2} - \gamma \hat \beta }
\end{array}
\end{equation}
for $\chi$ and $\gamma$ as positive constants.


\begin{proposition}
Consider the dynamic system \eqref{eq52} in closed-loop with the adaptive controller \eqref{eq10}, \eqref{eq15}, \eqref{eq49}.
Given desired shape vector $\mathbf{s}_d$, there exists an appropriate set of control parameters that ensure that: (1) all signals in the closed-loop system remain uniformly ultimately bounded (UUB), and (2) the deformation error $\mathbf{e}_1$ asymptotically converges to a compact set around zero.
\end{proposition}

\begin{proof}
	Consider the energy-like function 
	\begin{equation}
	\label{eq17}
	\begin{array}{*{20}{c}}
	{{V_2}(\sigma_1,\sigma_2,\tilde\beta)
	= {V_1}(\mathbf{\sigma}_1) + \frac{1}{2}\mathbf{\sigma} _2^T{\mathbf{\sigma} _2} + \frac{1}{2}{{\tilde \beta }^2}}
	\end{array}
	\end{equation}
	for $\tilde{\beta}=\beta - \hat \beta$.
	Considering the Young’s inequality, $\tilde \beta \hat \beta \le ({\beta ^2} - {{\tilde \beta }^2}) / 2$ and invoking \eqref{eq53}--\eqref{eq17}, we can show that the time derivative of $V_2$ satisfies
	\begin{align}
	\label{eq18}
	{{\dot V}_2}(\sigma_1,\sigma_2,\tilde\beta)
	& \le  - \left( {1 - \frac{\lambda_{\mathbf{K}_1}}{4}} \right){\left\| {{\mathbf{\sigma} _1}} \right\|^2} - {\left\| {{\mathbf{\sigma} _2}} \right\|^2} - \frac{\gamma }{2}{{\tilde \beta }^2} + b \notag \\
	& \le  - a{V_2}(\sigma_1,\sigma_2,\tilde\beta) + b 
	\end{align}
	where ${a = \min \left( {\left( {2 - {\lambda_{\mathbf{K}_1}}/{2}} \right),2,\gamma } \right)}$, ${b = \beta \delta \chi  + {\gamma }{\beta ^2}/2 >0}$.
	By selecting a matrix $\mathbf{K}_1$ that ensures that $a>0$, the state signals $\mathbf{\sigma}_1$, $\mathbf{\sigma}_2$ and $\tilde{\beta}$ are endowed with asymptotic stability and remain UUB \cite{2015AdaptiveNN}.
	This implies that the shape error $\mathbf{e}_1$ asymptotically converge to a compact set around zero, and ensures that the Jacobian estimation error $\tilde{\mathbf{J}}_s$ is bounded \cite{qu1998robust}.
\end{proof}

\subsection{Finite-Time Sliding Mode Control}
Many practical applications demand a tight timing performance, which cannot be accomplished by simply increasing the control gain; FTSMC is designed to addressed these issues.
The non-singular terminal sliding surface \cite{yang2011nonsingular} is given as:
\begin{align}
\label{eq23a}
{\mathbf{\sigma} _1} &= 
{\mathbf{e}_1} + 
{\alpha _1}\sig^{p_1}{\left( {{{\dot {\mathbf{e}}}_1}} \right)} + 
{\beta _1}\sig^{{q_1}}{\left( {{\mathbf{e}_1}} \right)} \\
{\mathbf{\sigma} _2} &= 
{\mathbf{e}_2} + 
{\alpha _2}\sig^{{p_2}}{\left( {{{\dot{\mathbf{e}}}_2}} \right)} + 
{\beta _2}\sig^{{q_2}}{\left( {{\mathbf{e}_2}} \right)}
\nonumber 
\end{align}
where $p_1 \in (1, 2)$, $q_1 > p_1$, $p_2 \in (1, 2)$, $q_2 > p_2$ and $\alpha_1, \alpha_2, \beta_1, \beta_2 > 0$ are all positive constants.
Invoking \eqref{eq7} and Lemma \ref{lemma3}, the time derivative of $\mathbf{\sigma}_1$ is given by:
\begin{equation}
\label{eq24}
{{\dot{\mathbf{\sigma}}}_1} = 
\hat{\mathbf{J}}_s \mathbf{u} + 
\tilde{\mathbf{J}}_s \mathbf{u} - {{\dot{\mathbf{s}}}_d} 
 + {\alpha _1}{p_1}{\left| {{{\dot{\mathbf{e}}}_1}} \right|^{{p_1} - 1}}{{\ddot {\mathbf{e}}}_1}
+ {\beta _1}{q_1}{\left| {{\mathbf{e}_1}} \right|^{{q_1} - 1}}{{\dot{\mathbf{e}}}_1}
\end{equation}
which we use to design the following velocity control input:
\begin{align}
\label{eq25}
\mathbf{u} = {\hat{\mathbf{J}}^+_s}
\left( 
\begin{array}{l}
- {\alpha _1}{p_1}{\left| {{{\dot{\mathbf{e}}}_1}} \right|^{{p_1} - 1}}{{\ddot {\mathbf{e}}}_1}
- {\varepsilon _1}\sgn\left( {{\mathbf{\sigma} _1}} \right) \\
- {\beta _1}{q_1}{\left| {{\mathbf{e}_1}} \right|^{{q_1} - 1}}{{\dot{\mathbf{e}}}_1} 
+ {{\dot{\mathbf{s}}}_d} 
\end{array}
\right)
\end{align}
With the above controller, we obtain a new expression for $\dot V_1$:
\begin{equation}
\label{eq23}
{{\dot V}_1}(\mathbf{\sigma}_1)
= \mathbf{\sigma} _1^T{{\dot{\mathbf{\sigma}}}_1} =  
- {\varepsilon _1}\mathbf{\sigma} _1^T\sgn\left( {{\mathbf{\sigma} _1}} \right) + \mathbf{\sigma} _1^T\tilde{\mathbf{J}}_s \mathbf{u}
\end{equation}
where by considering the Young’s inequality:
\begin{equation}
\label{eq29}
\mathbf{\sigma} _1^T\tilde{\mathbf{J}}_s \mathbf{u} 
\le {\left\| {{\mathbf{\sigma} _1}} \right\|^2} / {4} + \beta {\left\| \mathbf{u} \right\|^2}
\end{equation}
we can show (after some algebraic operations) that the following relation is satisfied:
\begin{align}
\label{eq28}
\dot V_1 (\mathbf{\sigma}_1)
& \le  - {\varepsilon _1}\left\| {{\mathbf{\sigma} _1}} \right\| + {\left\| {{\mathbf{\sigma} _1}} \right\|^2}/{4} + \beta {\left\| \mathbf{u} \right\|^2}
\end{align}
Finally, the adaptive rule for the DJM is given as:
\begin{align}
\label{eq31}
\dot{\hat{\mathbf{J}}}_s  &= 
\left( 
\begin{array}{l}
\ddot{\mathbf{s}} - \hat{\mathbf{J}}_s\dot{\mathbf{u}} + {\varepsilon _2}\sgn\left( {{\mathbf{\sigma} _2}} \right) + \mathbf{\varpi}  \notag \\
+ {\alpha _2}{p_2}{\left| {{{\dot{\mathbf{e}}}_2}} \right|^{{p_2} - 1}}{{\ddot{\mathbf{e}}}_2} 
+ {\beta _2}{q_2}{\left| {{\mathbf{e}_2}} \right|^{{q_2} - 1}}{{\dot{\mathbf{e}}}_2}
\end{array} \right){\mathbf{u}^ + }\\
\mathbf{\varpi} &= \mathbf{\sigma} _2^{T + }\left( {\tanh \left( {{{{{\left\| \mathbf{u} \right\|}^2}}} / {\chi }} \right)\hat \beta {{\left\| \mathbf{u} \right\|}^2} + \frac{1}{4}{{\left\| {{\mathbf{\sigma} _1}} \right\|}^2}} \right)
\end{align}

\begin{proposition}
Consider the dynamic system \eqref{eq52} in closed-loop with the adaptive controller given by  \eqref{eq25}, \eqref{eq31}, \eqref{eq49}. 
For this system, all state signals are semi-global practical finite time stable (SGPFS) \cite{2019Finite}, and the shape error $\mathbf{e}_1$ converges to a compact set within a finite time without any singularities during the task.
\end{proposition}

\begin{proof}
Consider the energy-like function \eqref{eq17} computed with the sliding surface \eqref{eq23a}.
By using the folowing square completion property:
	\begin{equation}
	\label{eq37}
	- \gamma {\tilde \beta}^2 / 2 \le {\gamma}/8 - \gamma |\tilde \beta| / 2
	\end{equation}
	along with \eqref{eq28}, we can show that the time derivative of $V_2$ satisfies the following relations:
	\begin{align}
	\label{eq34}
	\dot{V}_2 
	& \le  - {\varepsilon _1}\left\| {{\mathbf{\sigma} _1}} \right\| 
	- {\varepsilon _2}\left\| {{\mathbf{\sigma} _2}} \right\| - {\gamma }{{\tilde \beta }^2}/{2} + \left( {\beta {\delta}\chi  + {{\gamma {\beta ^2}}}/{2}} \right)  \notag \\
	& \le  - aV_2^{\frac{1}{2}} + b  
	\end{align}
	where ${a = \min \left( {2{\varepsilon _1},2{\varepsilon _2},\gamma } \right)}$, ${b = \beta {\delta}\chi + {{\gamma {\beta ^2}}}/{2} + {\gamma }/{8} > 0}$. 
	By selecting appropriate parameters that ensure $a > 0$, then $V_2$ converges to a compact set $\{ V_2 | V_2 \le \Omega^2, \text{~time} \ t \ge T_s \}$ within the convergence time $T_s$ calculated by \cite{2019Finite}
	\begin{equation}
	\label{eq27}
	\begin{array}{*{20}{c}}
    {{T_s} = \frac{2}{{av}}\left( {{V_2^{\frac{1}{2}}\left( { \mathbf{X} \left( 0 \right)} \right)}  - {\Omega}} \right)},&{{\Omega} = \frac{b}{{\left( {1 - v} \right)a}}}
    \end{array}
	\end{equation}
	where $\mathbf{X} \left( 0 \right) = {[ {\sigma _1^T\left( 0 \right),\sigma _2^T\left( 0 \right),\tilde \beta \left( 0 \right)} ]^T} \in \mathbb{R}^{2p+1} $ and $v \in [0,1]$.
	Then, $\mathbf{X}$ remains in the compact set defined by ${\Omega _v} = \left\{{\mathbf{X} ~|~ \| \mathbf{X} \| \le {\sqrt{2} \Omega}, \ t \ge {T_s}} \right\}$, which means that $\mathbf{\sigma}_1, \mathbf{\sigma}_2$ and $\tilde{\beta}$ are SGPFS.
	This implies that the shape error $\mathbf{e}_1$ and estimation error $\mathbf{e}_2$ converge to a compact set around zero within finite time \cite{yang2011nonsingular}.
	
    In the traditional terminal sliding mode controller \cite{feng2002non}, there exists $|\mathbf{e}_i|^w, i=1,2$ and $w<0$.
	Thus, $\mathbf{u}$ cannot be guaranteed to be bounded when $|\mathbf{e}_i|$ is near the origin, namely, numerical singularity.
	For the proposed velocity command \eqref{eq25}, it does not contain any negative fractional power since $p_i\in(1,2), q_i > p_i, i=1,2$; thus it is singularity-free.
	The above results validate that the CRDO can be manipulated into the target configuration within finite time along NFTSM without any singularity and all state signals as SGPFS.
\end{proof}
The workflow of the proposed framework is given in Algorithm \ref{algorithm1}.
Fig. \ref{fig34} gives the block diagram of the proposed dual-arm robot manipulation framework.

\begin{algorithm}
\caption{Workflow of the proposed framework} 
\label{algorithm1} 
\begin{algorithmic}[1] 
\Require 
$Threshold$; $Max$; $Sw$: default=1; 
\State Give a target shape feature $\mathbf{s}_d$ calculated by $\bar{\mathbf{c}}^*$;
\State Conduct small deformations around the starting configuration to initialize $\hat{\mathbf{J}}_s(0)$ and start the manipulation;
\State $k=0$, $\hat{\beta}(0) \gets 0.001$

\While {$\| \mathbf{e}_1 \| \ge Threshold$ and $k < Max$} 
\State Record the current position $\mathbf{r}$ and velocity $\mathbf{u}$;
\State Record the current contour $\bar{\mathbf{c}}$ of the object;
\State Calculate the current shape feature $\mathbf{s} \gets$ \eqref{eq2} - \eqref{eq65};
\State Calculate error signals $\mathbf{e}_1$ and $\mathbf{e}_2 \gets$ \eqref{eq55};
\State Calculate $\mathbf{\sigma}_i$, 
\eqref{eq8} $\gets Sw=1$ or \eqref{eq23a} $\gets Sw=2$;

\State Update the adaptive term $\hat{\beta} \gets$ \eqref{eq49};
\State Update $\mathbf{u}$, 
\eqref{eq10} $\gets Sw=1$ or \eqref{eq25} $\gets Sw=2$;

\State Update $\hat{\mathbf{J}}_s$, 
\eqref{eq15} $\gets Sw=1$ or \eqref{eq31} $ \gets Sw=2$;

\State Dual-arm robot moves using the updated $\mathbf{u}$;
\State $k = k + 1$;
\EndWhile 
\end{algorithmic} 
\end{algorithm}

\begin{figure}[h]
    \centering
    \includegraphics[scale=0.46]{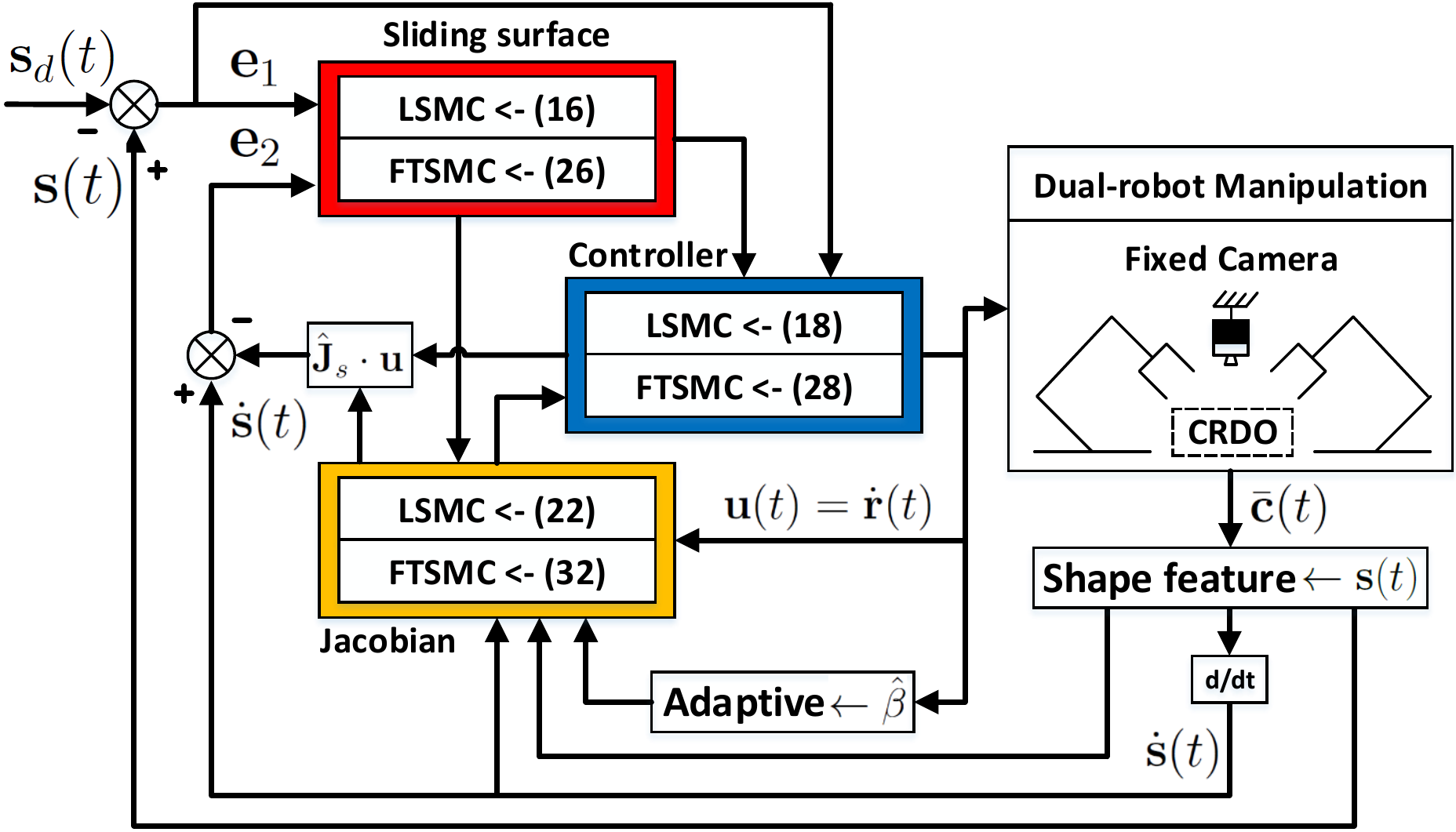}
    \caption{The block diagram of the proposed manipulation framework.}
    \label{fig34}
\end{figure}

\begin{remark}
In practice, the function $\sgn$ can be replaced with the function $\tanh$ to reduce the chattering effects.
\end{remark}


\section{NUMERICAL SIMULATIONS}\label{sec6}
In this section, we simulate the motion of a planar dual-arm robot that manipulates an elastic cable.
The velocity input is constructed as $\mathbf{u}=[\mathbf{u}_1^T, \mathbf{u}_2^T]^T \in \mathbb{R}^6$, where the coordinates of $\mathbf u_{i}^T = [u_{i1},u_{i2},u_{i3}]$ represent the linear and angular velocities along the $x$, $y$ and $z$ axes, respectively.
The maximum speed for linear motions is set to $|u_{ij}| \leq 0.06$ m/s (for $j=1,2$), and for angular motions to $|u_{i3}| \leq 0.2$ rad/s.
The simulation environment is programmed in Python; The code is available at \url{https://github.com/q546163199/shape_deformation/tree/master/python/package/shape_simulator}.

\begin{figure}
	\centering
	\subfloat[]{\includegraphics[width=11mm,height=11mm]{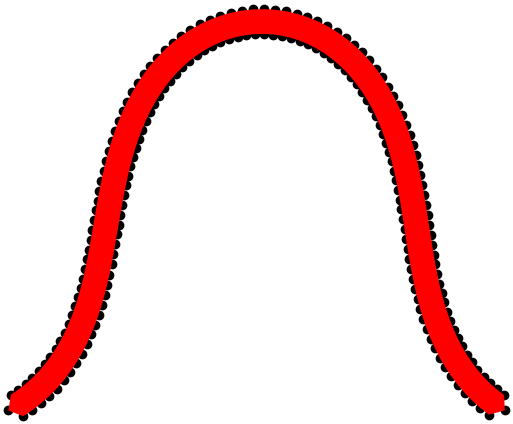}}
	\subfloat[]{\includegraphics[width=11mm,height=11mm]{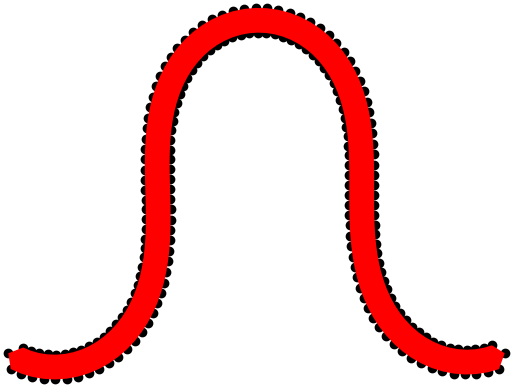}}
	\subfloat[]{\includegraphics[width=11mm,height=11mm]{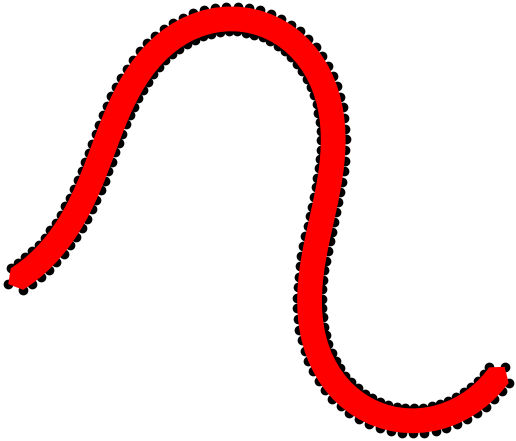}}
	\subfloat[]{\includegraphics[width=11mm,height=11mm]{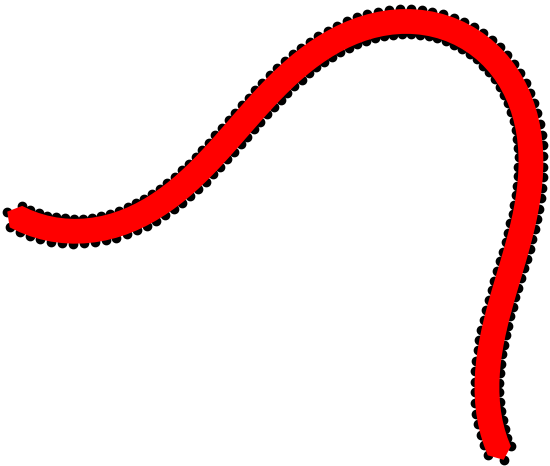}}
	\subfloat[]{\includegraphics[width=11mm,height=11mm]{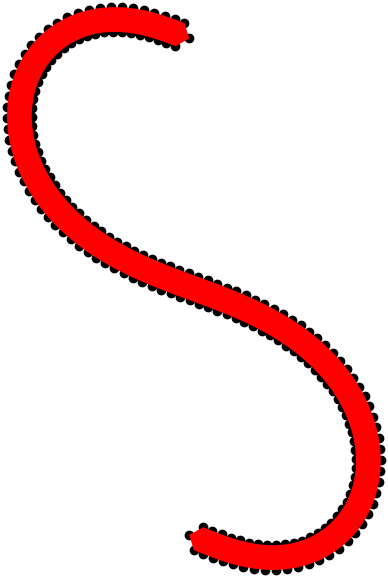}}
	\subfloat[]{\includegraphics[width=11mm,height=11mm]{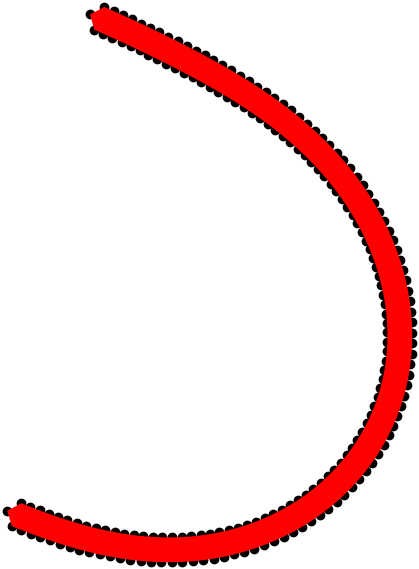}}
	\subfloat[]{\includegraphics[width=11mm,height=11mm]{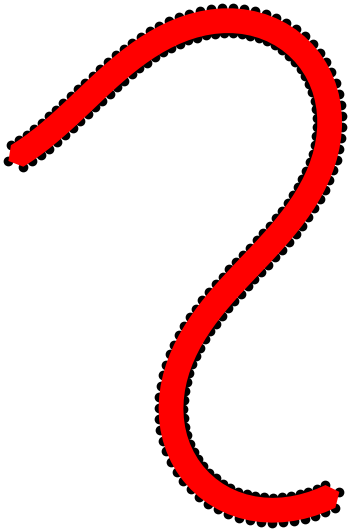}}
	\subfloat[]{\includegraphics[width=11mm,height=11mm]{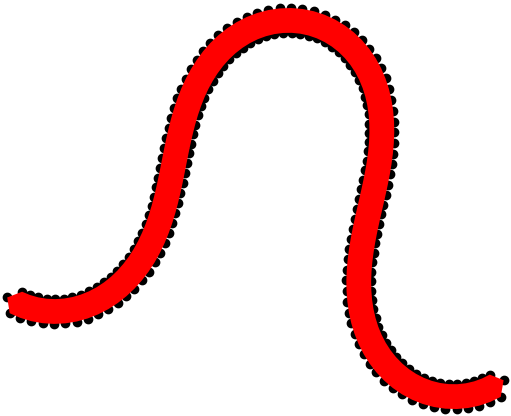}}
	
	\caption{Various shapes of the elastic cable among different conditions.}
	\label{fig4}
\end{figure}

\begin{figure}[h]
\centering
\includegraphics[scale=0.26]{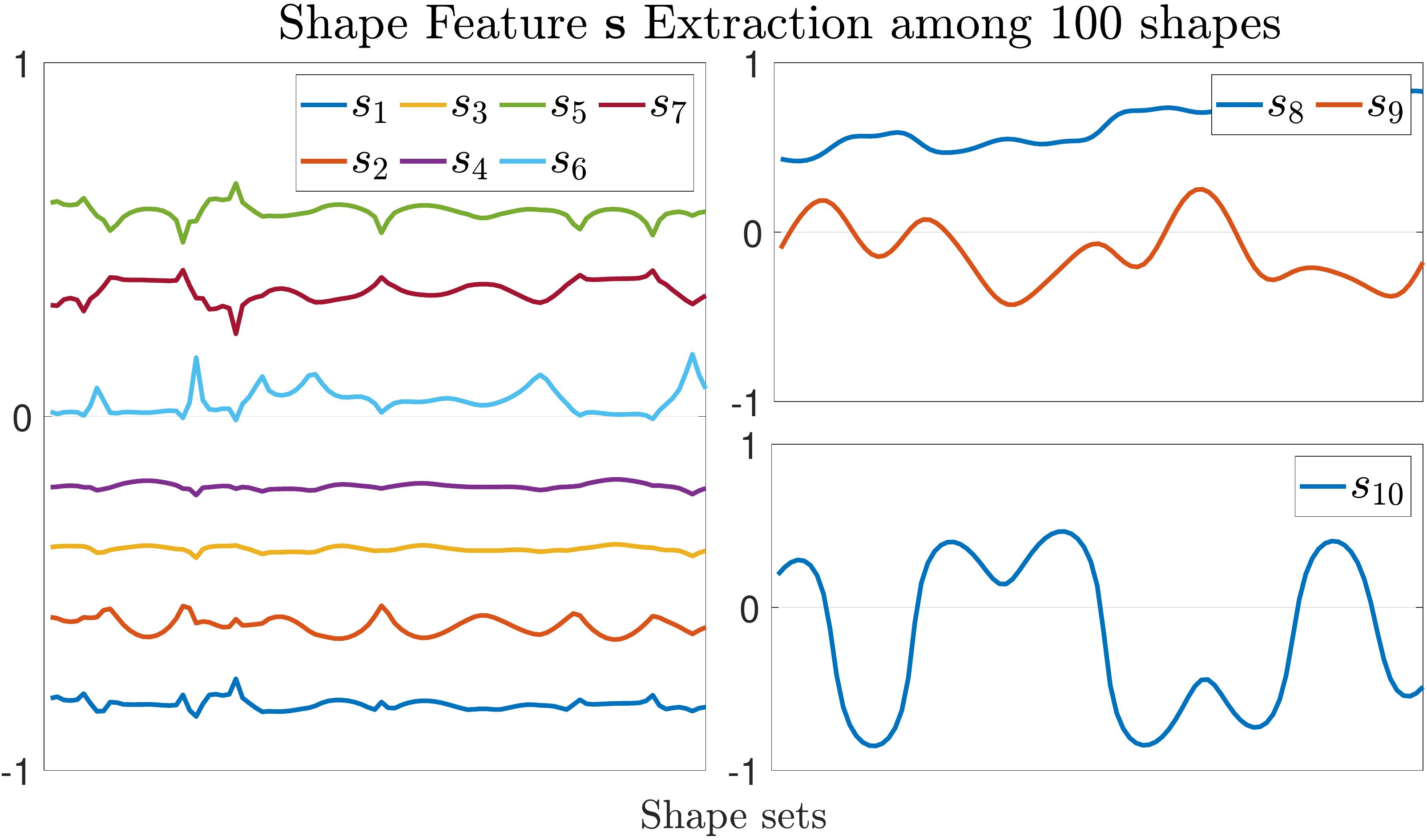}
\caption{
Contour moments extraction \eqref{eq2}-\eqref{eq65} from 100 sample shapes.}
\label{fig41}
\end{figure}

\begin{figure}[h]
\centering
\includegraphics[scale=0.256]{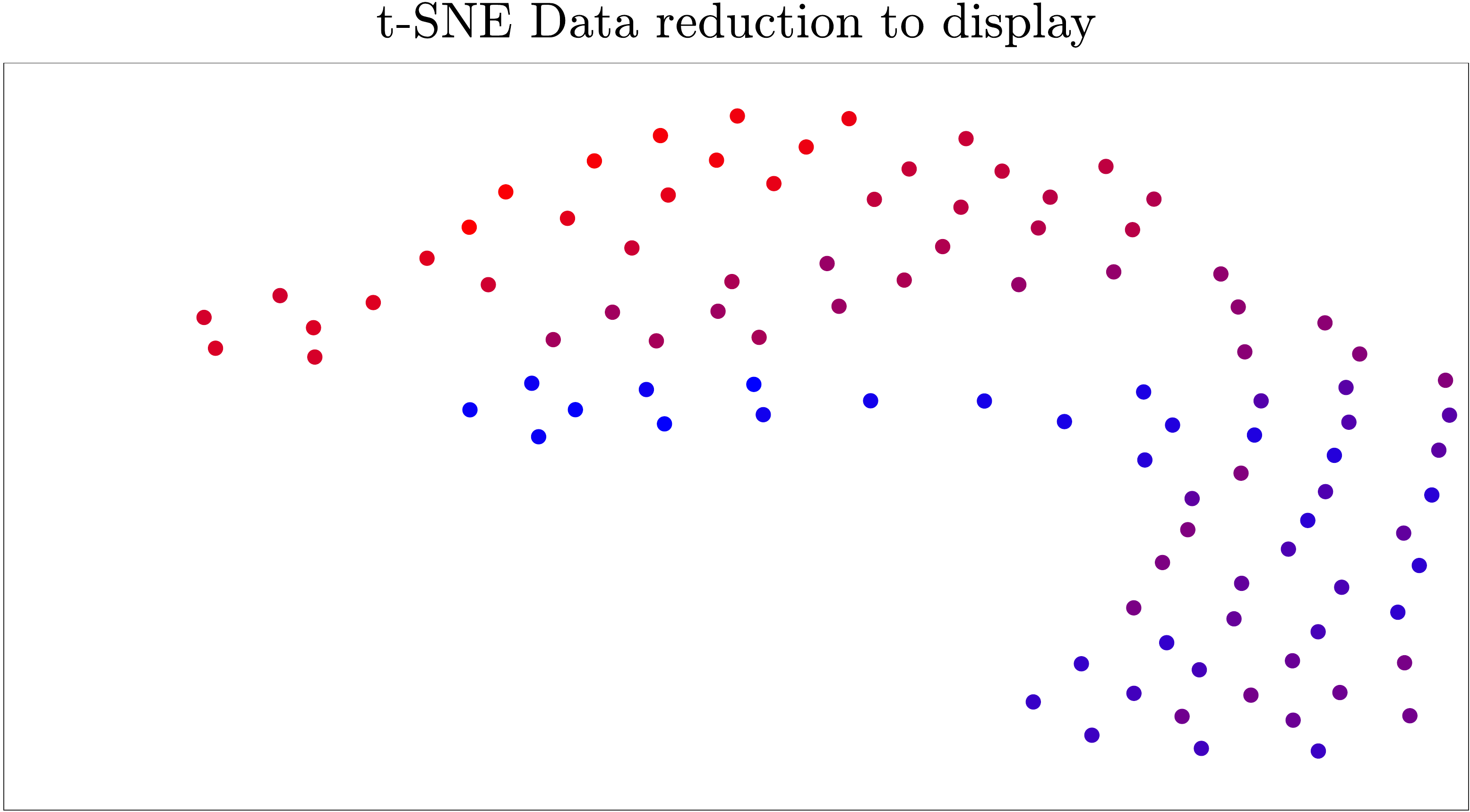}
\caption{
Data distribution t-SNE of the contour moments.
Each point represents each shape.
Different color indicates the different shape to visualize the shape changes.}
\label{fig44}
\end{figure}

\subsection{Validation of the Contour Moments Extraction}\label{sec6a}
In this section, the dual-arm robot manipulates the elastic cable along a continuous trajectory that generates multiples object shapes, see Fig. \ref{fig4}. 
The contour moments that are extracted from these robot-object motions are depicted in Fig. \ref{fig41}, which shows that as shape changes, the features' profiles also change.
We collect a 100 (ten-dimensional) data points $\mathbf s$ from this motion test and use t-SNE \cite{van2008visualizing} to visualize its evolution with a color gradient (with blue as initial and red as final), as shown in Fig. \ref{fig44}.
These results demonstrate that the proposed shape features provide a smooth (continuous) representation for the object's deformable configuration.

\begin{figure}
	\centering
	\includegraphics[scale=0.26]{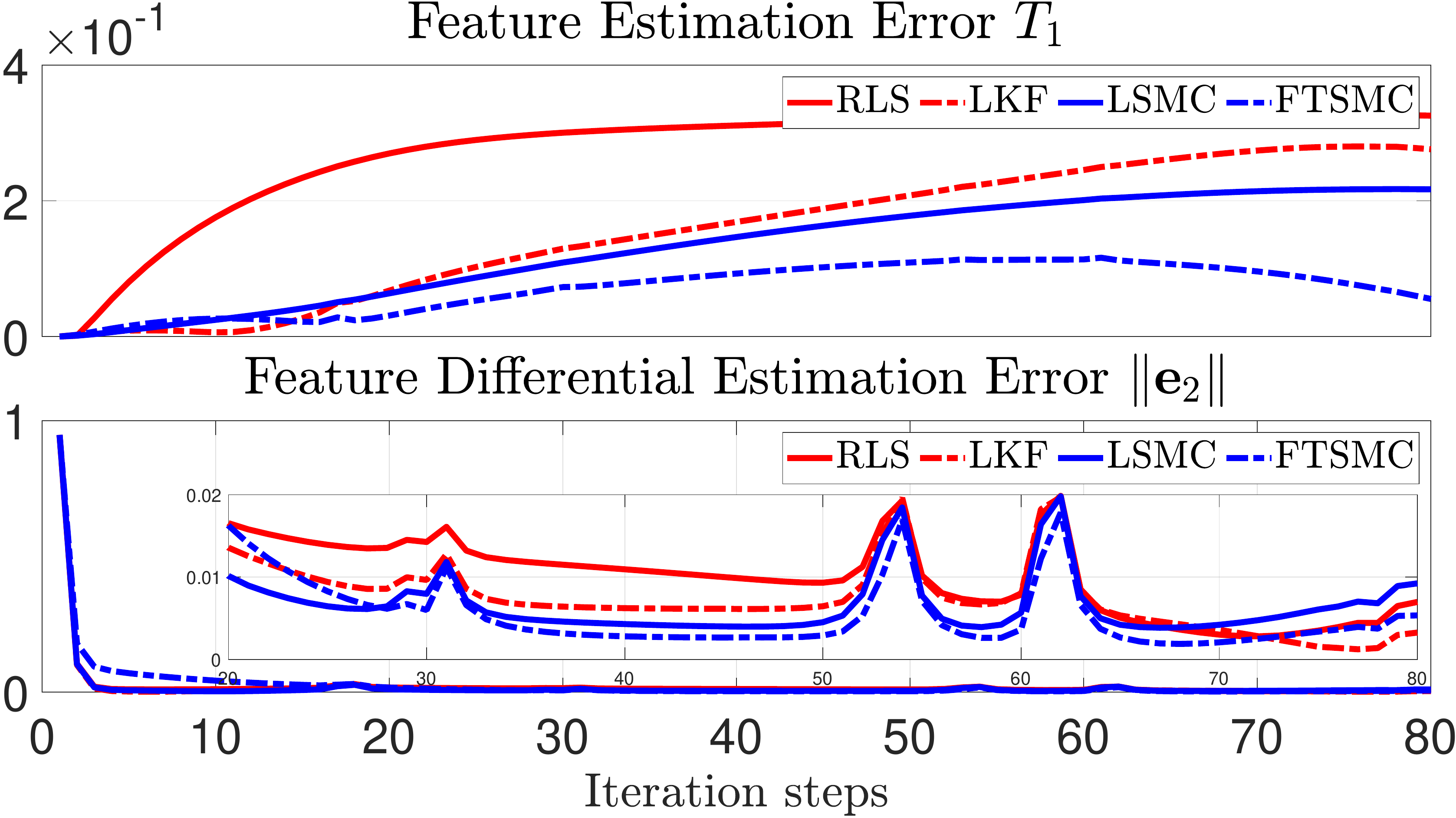}
	\caption{Profiles of the criteria $T_1$ and $\| \mathbf{e}_2 \|$ with dual-arm robot executing a given trajectory $\mathbf{r}$.
	It gives the estimation effect of $\hat{\mathbf{J}}_s$ among
	RLS \cite{hosoda1994versatile}, 
	LKF \cite{qian2002online}, 
	LSMC \eqref{eq15} and 
	FTSMC \eqref{eq31}, respectively.}
	\label{fig3}
\end{figure}

\subsection{Estimation of the DJM}
\label{sec6b}
In this section, we command the robot to execute slow and smooth trajectories $\dot{\mathbf{r}}(t)$ with the proposed LSMC and FTSMC controllers. 
The purpose of this test is to compare the accuracy of the Jacobian estimation $\hat{\mathbf{J}}_s$ obtained with these methods, and that obtained with traditional Recursive Least Square (RLS) \cite{hosoda1994versatile} and Linear Kalman Filter (LKF) \cite{qian2002online}.
For that, we initialize $\hat{\mathbf{J}}_s(0)$ with a random matrix and then perform small local motions and run the estimator to obtain an ``good-enough'' matrix $\hat{\mathbf{J}}_s$.
To quantify the accuracy of such estimation, we introduce the following metric:
\begin{equation}
\label{eq40}
{T_1} = \left\| {{\mathbf{s}} - {{\hat{\mathbf{s}}}}} \right\|
\end{equation}
where $\hat{\mathbf{s}}$ is the estimated shape feature that is updated as follows:
\begin{equation}
\label{eq41}
\begin{array}{*{20}{c}}
{\dot{\hat{\mathbf{s}}}  = \hat{\mathbf{J}}_s \cdot \mathbf{u}},&{\hat{\mathbf{s}}\left( 0 \right) = \mathbf{s} \left( 0 \right)}
\end{array}
\end{equation}
where $\hat{\mathbf{J}}_s$ is estimated by RLS, LKF, LSMC and FTSMC.
$\hat{\mathbf{s}}(0)$ and $\mathbf{s}(0)$ are the initial values of $\hat{\mathbf{s}}$ and $\mathbf{s}$, respectively.
Fig. \ref{fig3} illustrates the evolution of $T_1$ and $\| \mathbf{e}_2 \|$ for the motion $\dot{\mathbf{r}}$ executed by the robot in the previous test.
It can be seen that FTSMC provides the best approximation amongst the considered algorithms (viz. RLS, LKF, LSMC).
These results corroborate that the proposed FTSMC can accurately predict the shape features $\mathbf{s}$ and its differential change $\dot{\mathbf s}$, enabling to guide the robot with the estimated matrix $\hat{\mathbf J}_s$.

\begin{figure}[ht]
\centering
\subfloat[RLS]  {\includegraphics[scale=0.17]{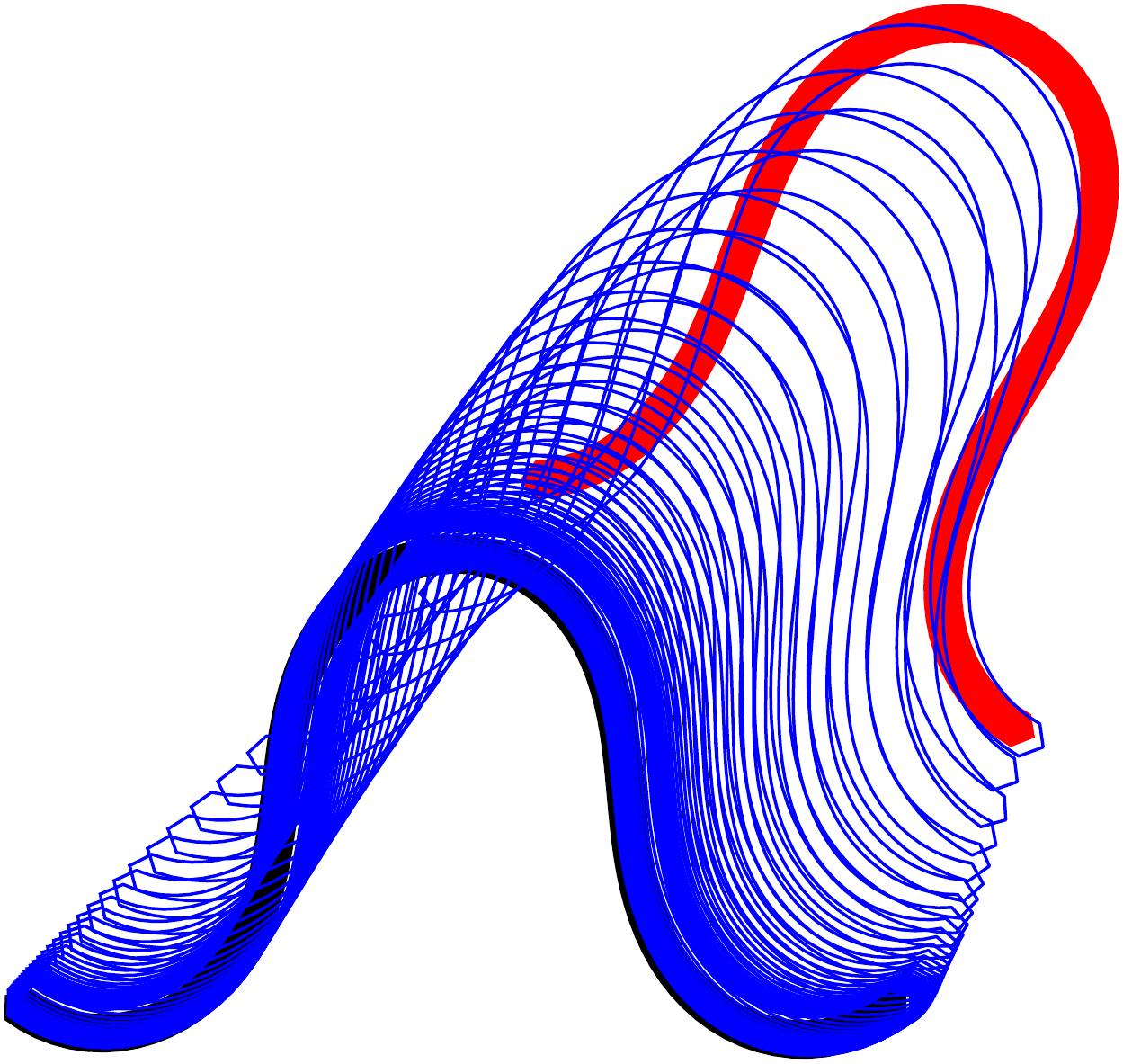}}
\subfloat[LKF]  {\includegraphics[scale=0.17]{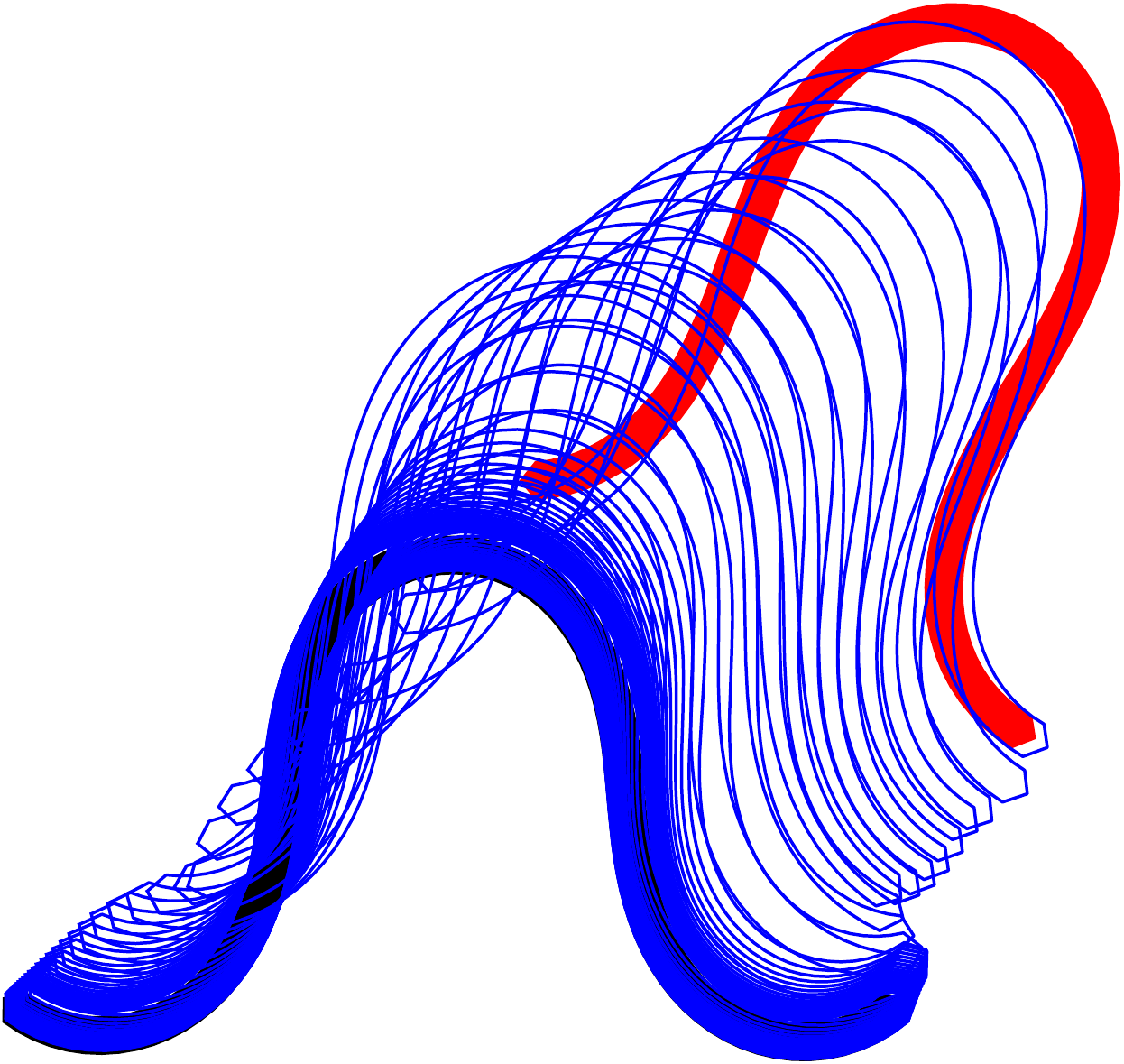}}
\subfloat[LSMC] {\includegraphics[scale=0.17]{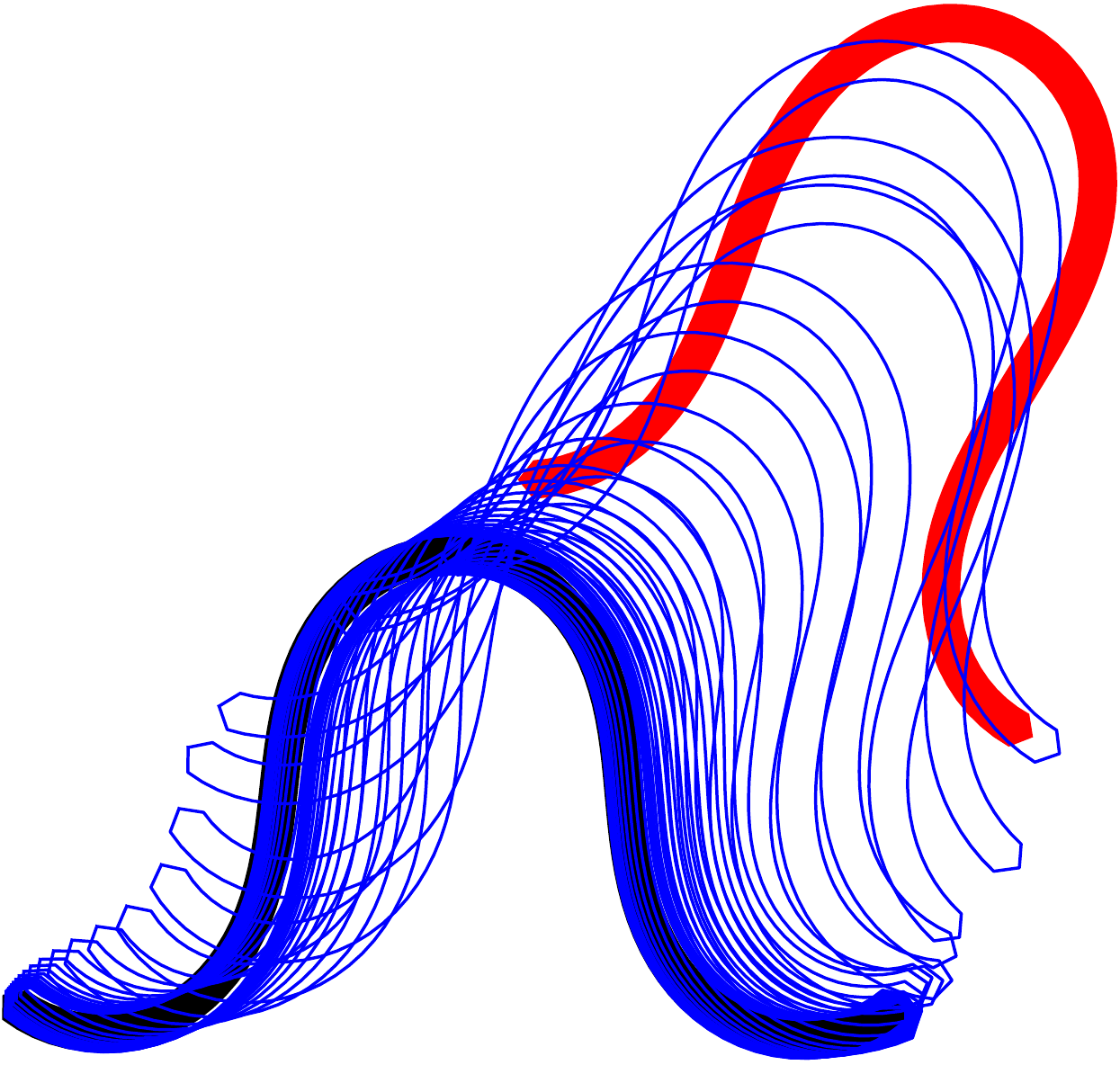}}
\subfloat[FTSMC]{\includegraphics[scale=0.17]{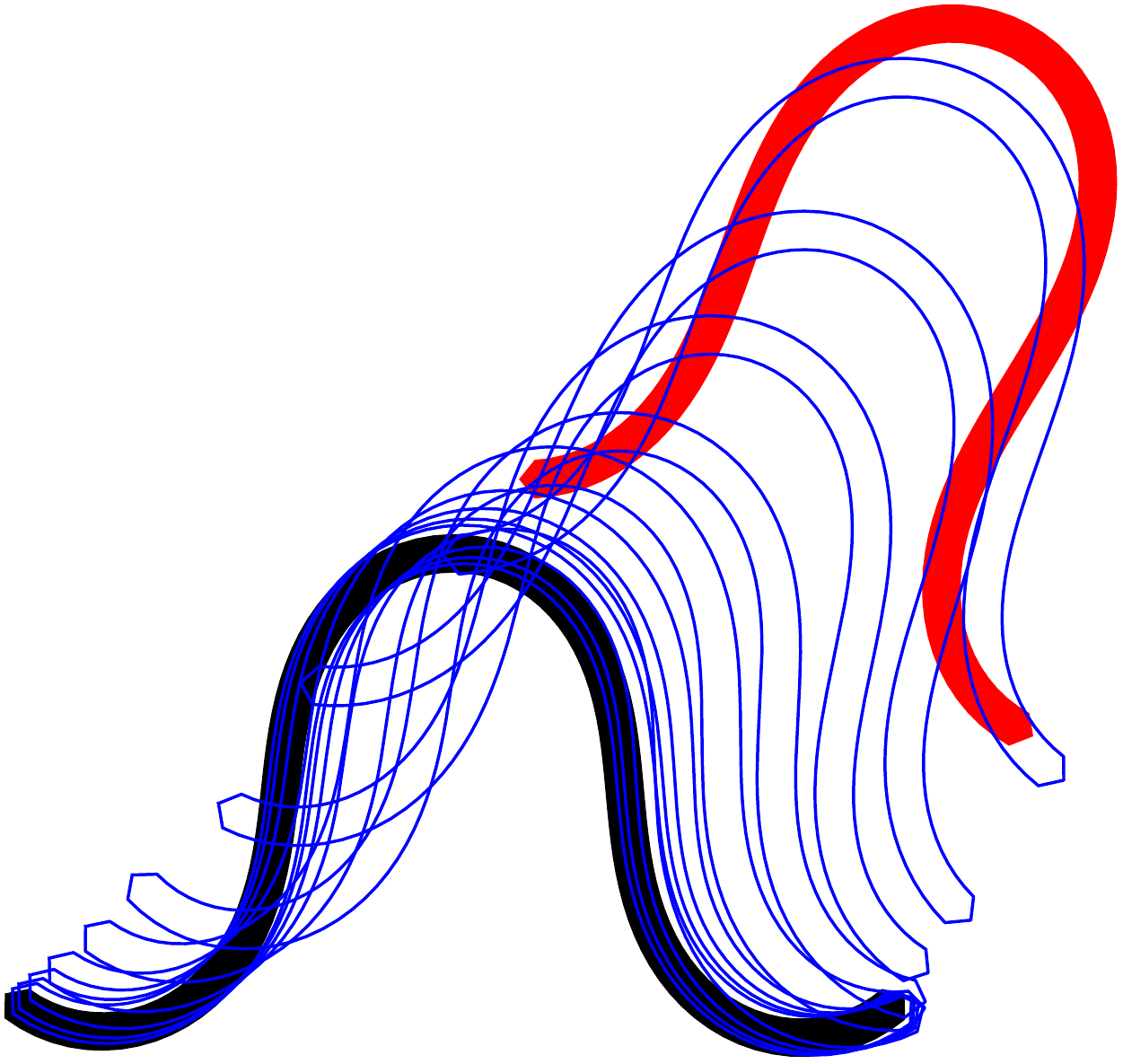}}
\caption{Profiles of the shape deformation trajectories among four methods, namely, 
RLS \cite{hosoda1994versatile}\cite{qi2020adaptive}, 
LKF \cite{qian2002online}\cite{qi2020adaptive}, 
LSMC \eqref{eq10}\eqref{eq15} and FTSMC \eqref{eq25}\eqref{eq31}. 
The red represents the initial contour, the blue represents transitional contours and the black represents the target contour $\bar{\mathbf{c}}^*$ represented by $\mathbf{s}_d$.
The deformation trajectories display every two frames.}
\label{fig6}
\end{figure}

\begin{figure}[ht]
\centering
\includegraphics[scale=0.25]{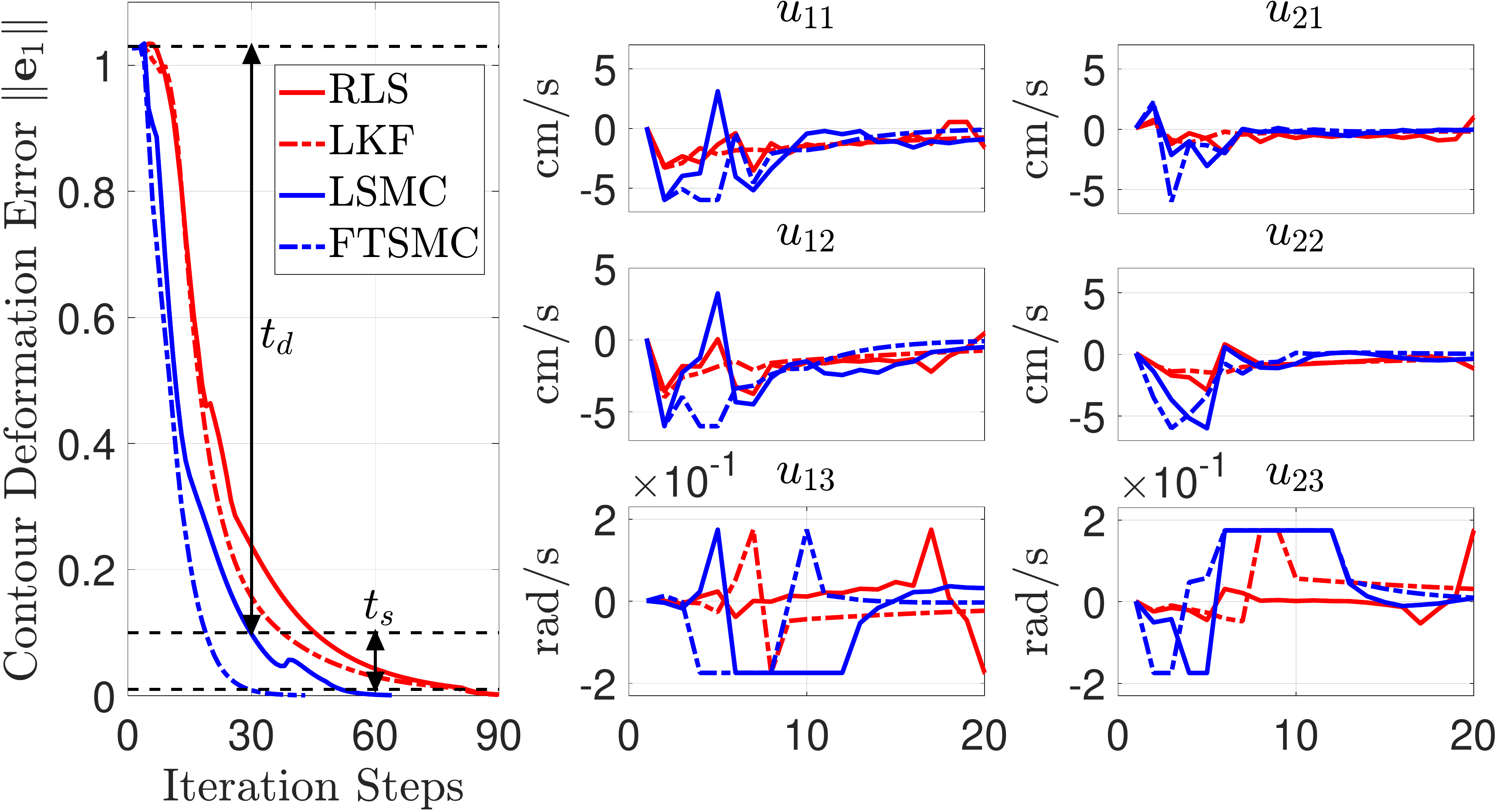}
\caption{Profiles of contour deformation error $\| \mathbf{e}_1 \|$ and velocity command $\mathbf{u}$ among four methods, namely, 
RLS \cite{hosoda1994versatile}\cite{qi2020adaptive}, 
LKF \cite{qian2002online}\cite{qi2020adaptive}, 
LSMC \eqref{eq10}\eqref{eq15} and 
FTSMC \eqref{eq25}\eqref{eq31}.}
\label{fig5}
\end{figure}

\begin{figure}[ht]
\centering
\includegraphics[scale=0.26]{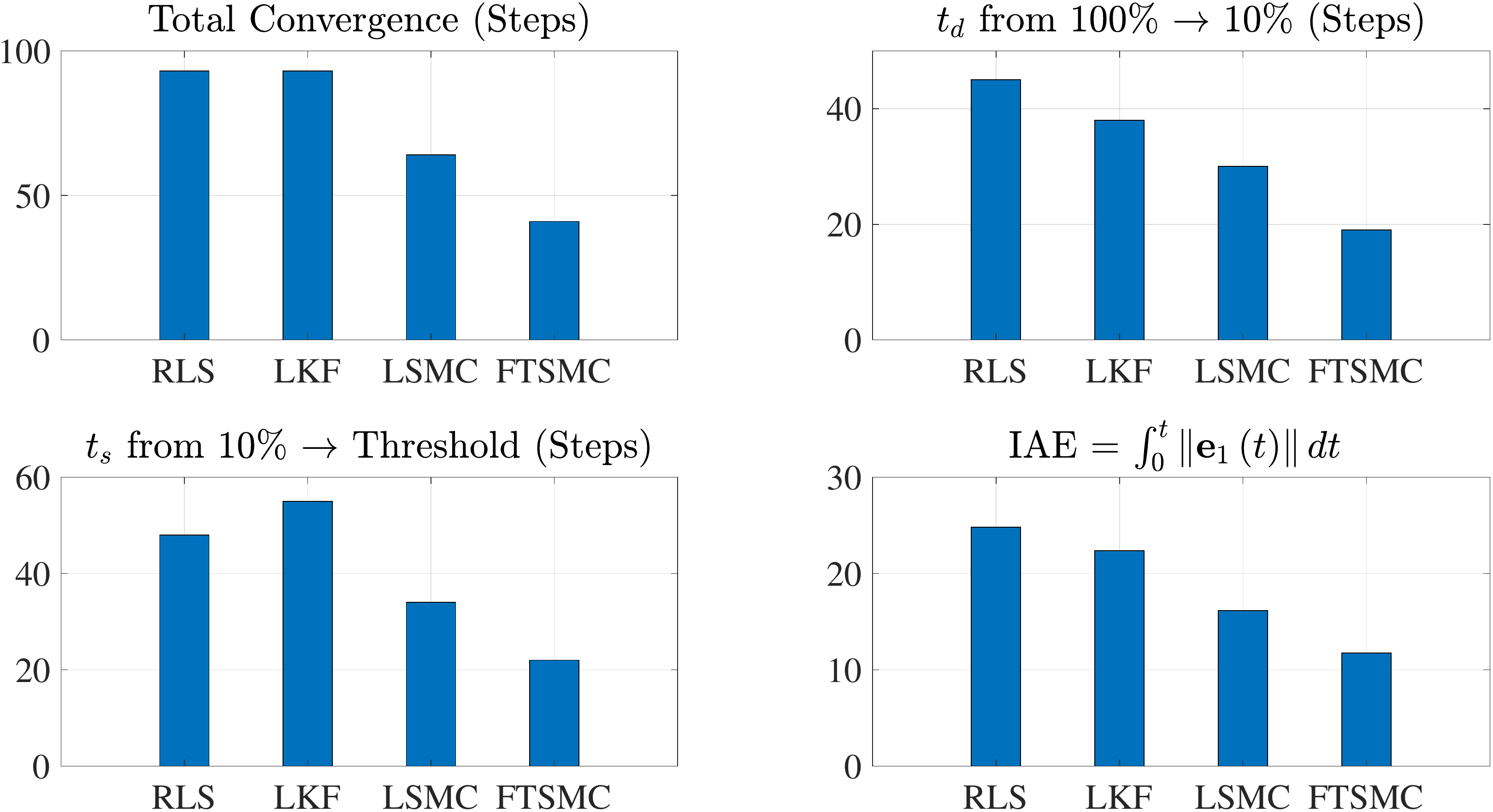}
\caption{Performance comparison of four indices among four methods, namely, 
RLS \cite{hosoda1994versatile}\cite{qi2020adaptive}, 
LKF \cite{qian2002online}\cite{qi2020adaptive}, 
LSMC \eqref{eq10}\eqref{eq15} and
FTSMC \eqref{eq25}\eqref{eq31}.}
\label{fig17}
\end{figure}

\subsection{Manipulation of an Elastic Cable}\label{sec6c}
In this section we validate the performance of the proposed shape controller by commanding the robot to deform a linear deformable object (cable) into a desired contour $\bar{\mathbf{c}}^*$ (which is transformed into a desired feature vector $\mathbf{s}_d$).
We compare the error minimization performance of the FTSMC/LSMC methods with that of a classical visual servoing controller (i.e. \'a la Chaumette) with the DJM adaptively estimated with RLS \cite{hosoda1994versatile} and LKF \cite{qian2002online}.

Additionally to the shape error $\| \mathbf{e}_1 \|$, we use the following four indices to analyze the performance of the system:
\begin{enumerate}
\item Convergence step $T_{max}$ indicates the total deformation time consumed by the system.
\item Decay time $t_d$ indicates the time required to decrease from $\| \mathbf{e}_1 \|$ to 10\% of its initial value $\| \mathbf{e}_1(0) \|$.
\item Settling time $t_s$ indicates the time required to decrease from 10\% of $\| \mathbf{e}_1(0) \|$ to 0.01. 
\item Integrated absolute error (IAE) indicates the cumulative error and shows the energy consumption of the system
\begin{equation}
\label{eq43}
\rm{IAE} = \int_0^t {\left\| {{\mathbf{e}_1}} \right\|dt}
\end{equation}
\end{enumerate}


Fig. \ref{fig6} shows the simulated deformation trajectories with the four methods.
From the analysis of the FTSMC in \eqref{eq27}, we can see that the parameters $\varepsilon_1, \varepsilon_2, \gamma, \chi$ determine the minimization performance $\mathbf{\sigma}_i \rightarrow 0$.
The parameter $\varepsilon_1$ represents a control gain with similar function to $\mathbf{K}_1$ in \eqref{eq8}.
As there exists an adjustable power term in \eqref{eq25}, the accuracy and speed of the system can be adjusted without increasing $\varepsilon_1$.
This valuable property is lacking in the other three methods, as they must increase the ``proportional-like'' control gain to achieve a comparable convergence rate (which in turn may generate a large input velocity $\mathbf{u}$).
Fig. \ref{fig5} shows that FTSMC provides the fastest error minimization and LSMC the second-best, while RLS and LKF give a similar performance.
A summary of the performance comparison amongst these control methods is shown in Fig. \ref{fig17}.

\begin{figure}[h]
\centering
\subfloat[Elastic cable]{\includegraphics[width=29mm, height=20mm]{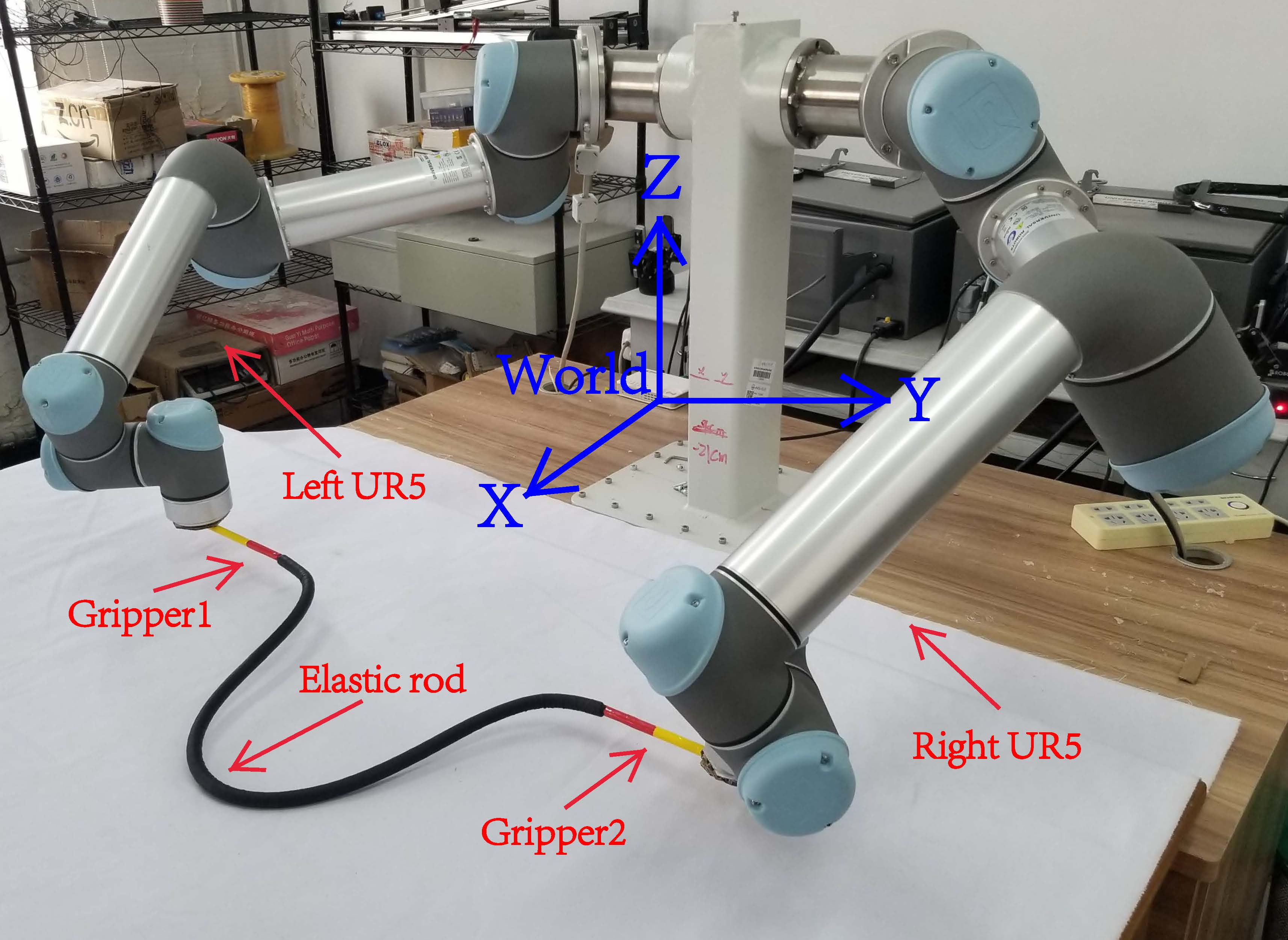}\label{fig13a}}
\subfloat[Linear sponge]{\includegraphics[width=29mm, height=20mm]{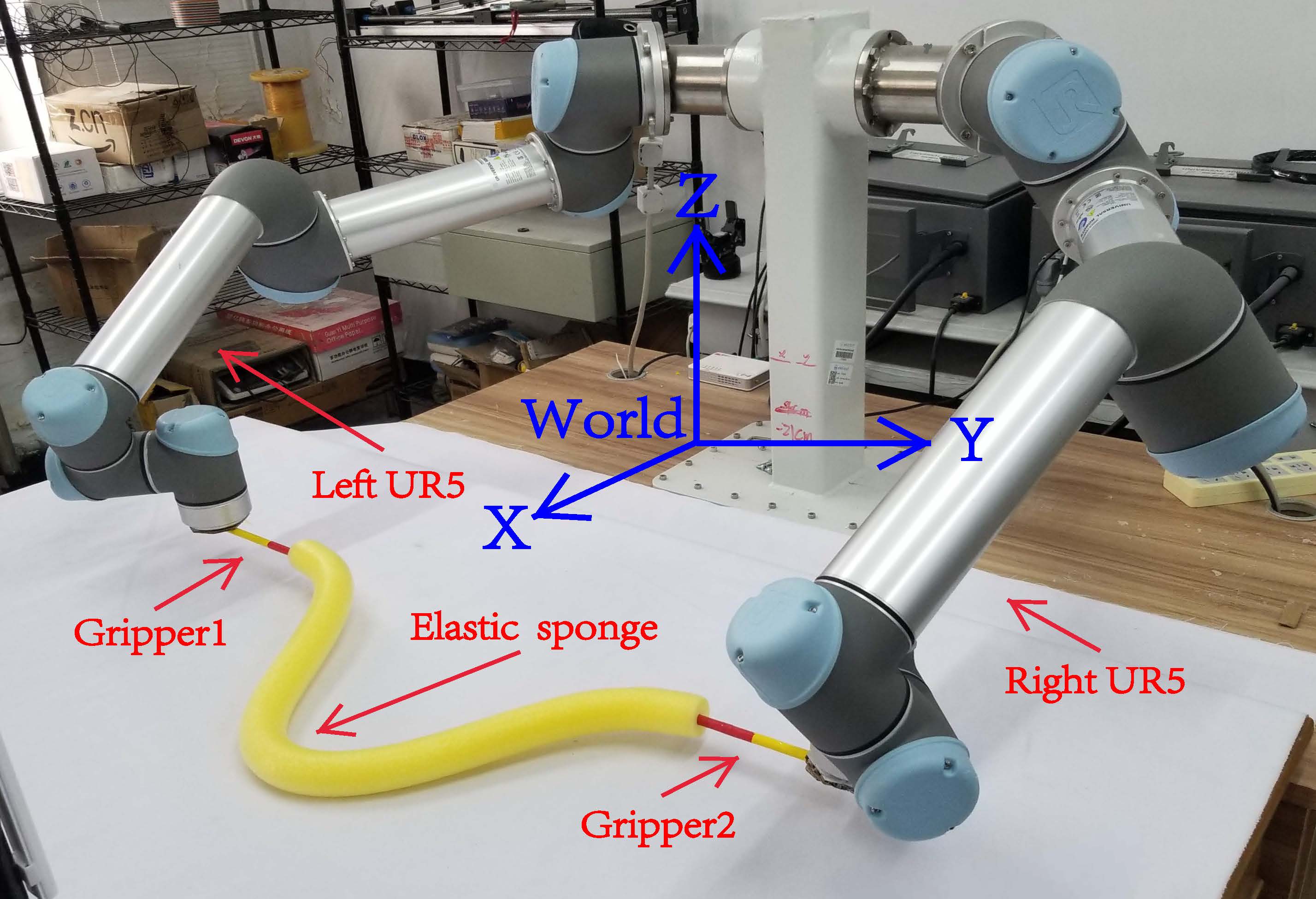}\label{fig13b}}
\subfloat[Plastic folder]{\includegraphics[width=29mm, height=20mm]{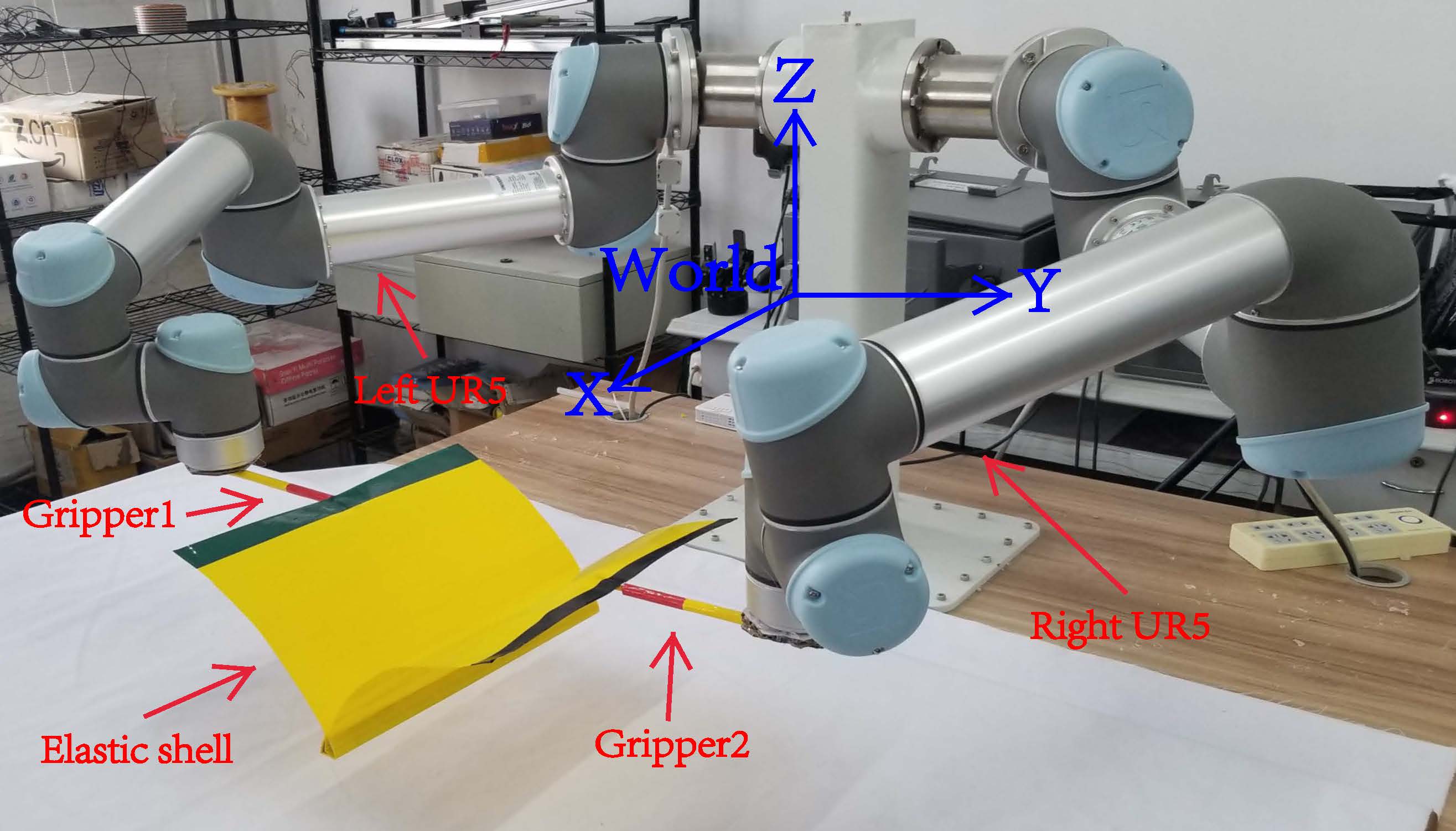}\label{fig13c}}

\subfloat[NH beam]{\includegraphics[width=29mm, height=20mm]{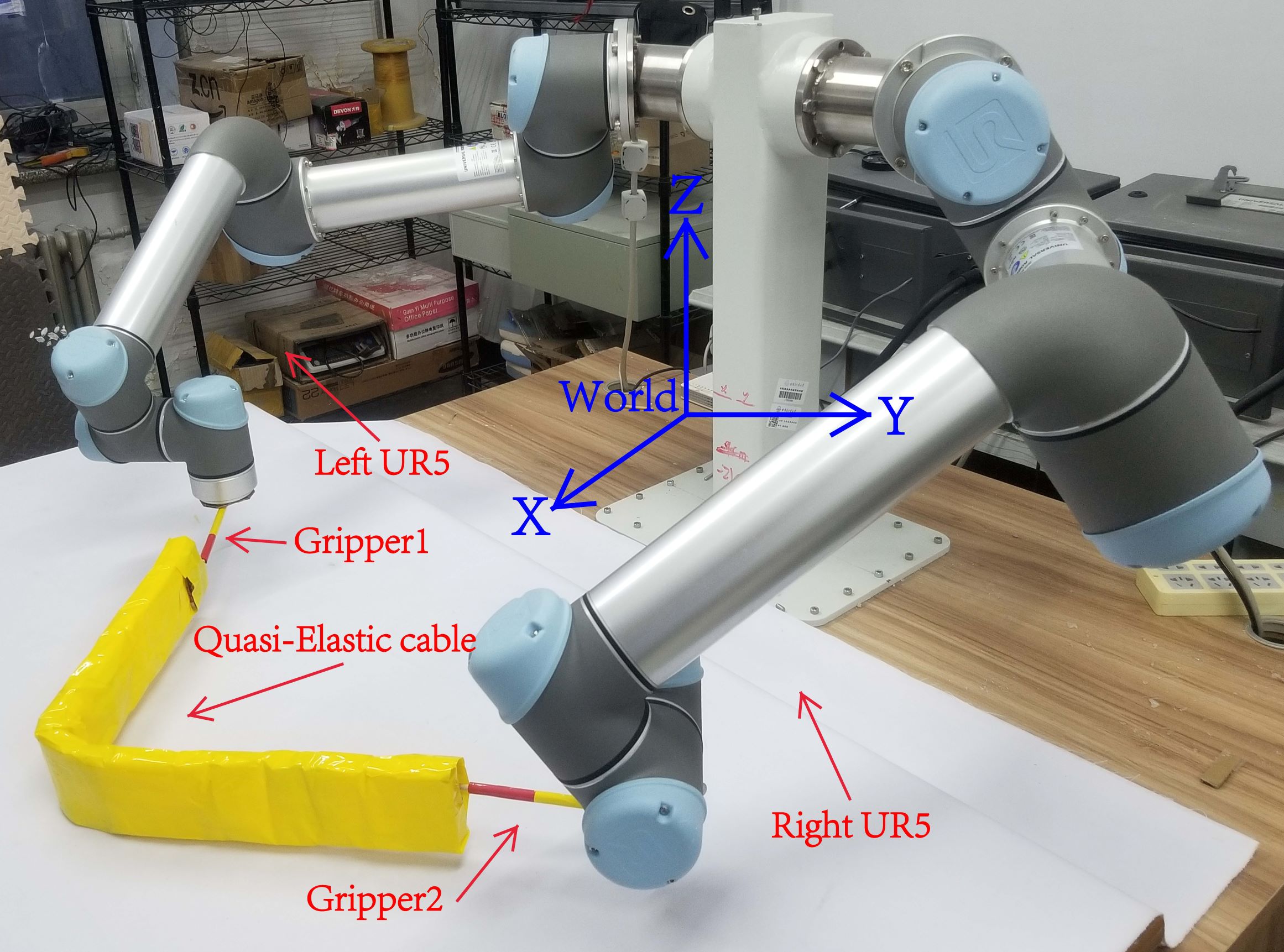}\label{fig13d}}
\subfloat[Articulated wallet]{\includegraphics[width=29mm, height=20mm]{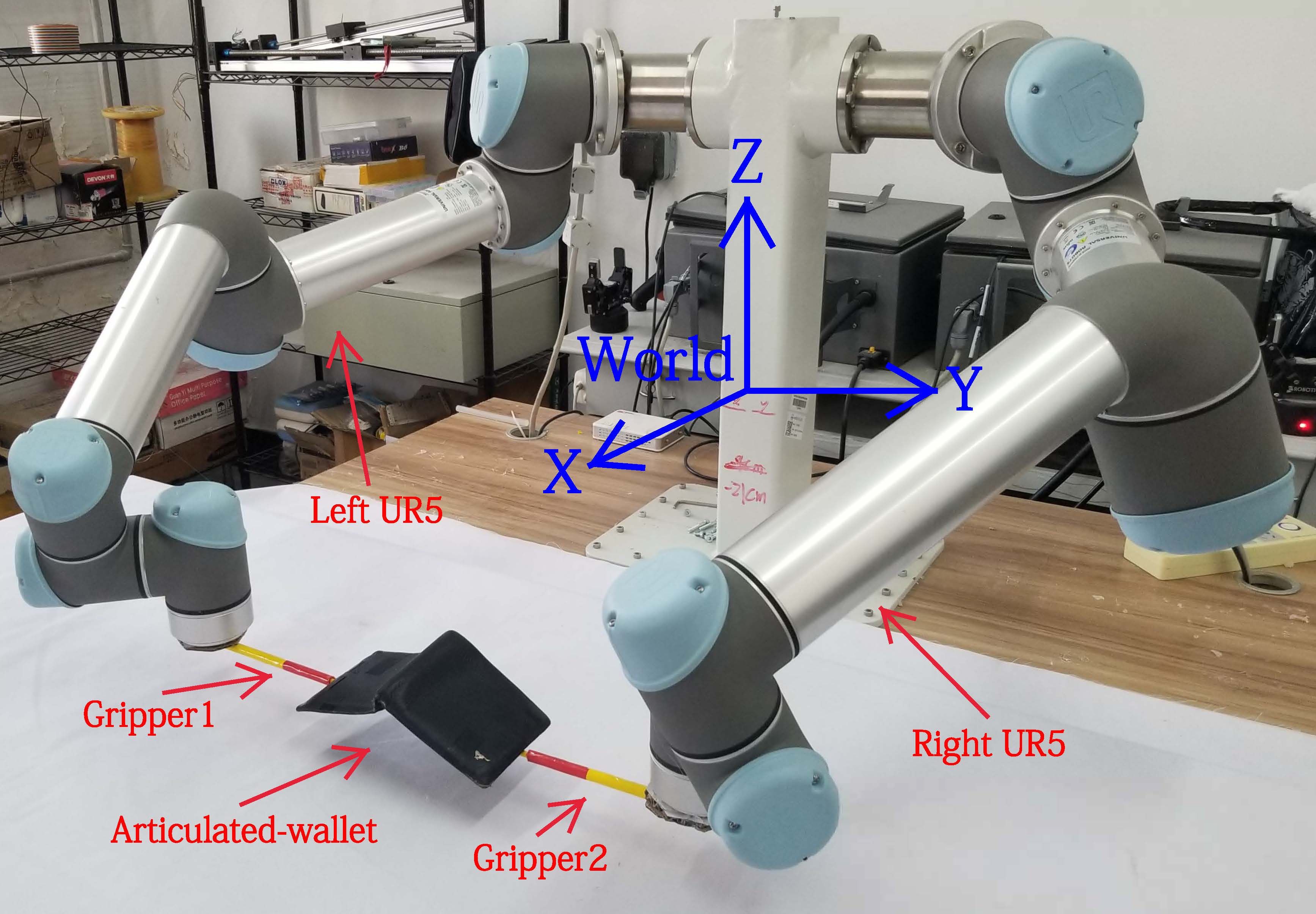}\label{fig13e}}
\subfloat[Rigid box]{\includegraphics[width=29mm, height=20mm]{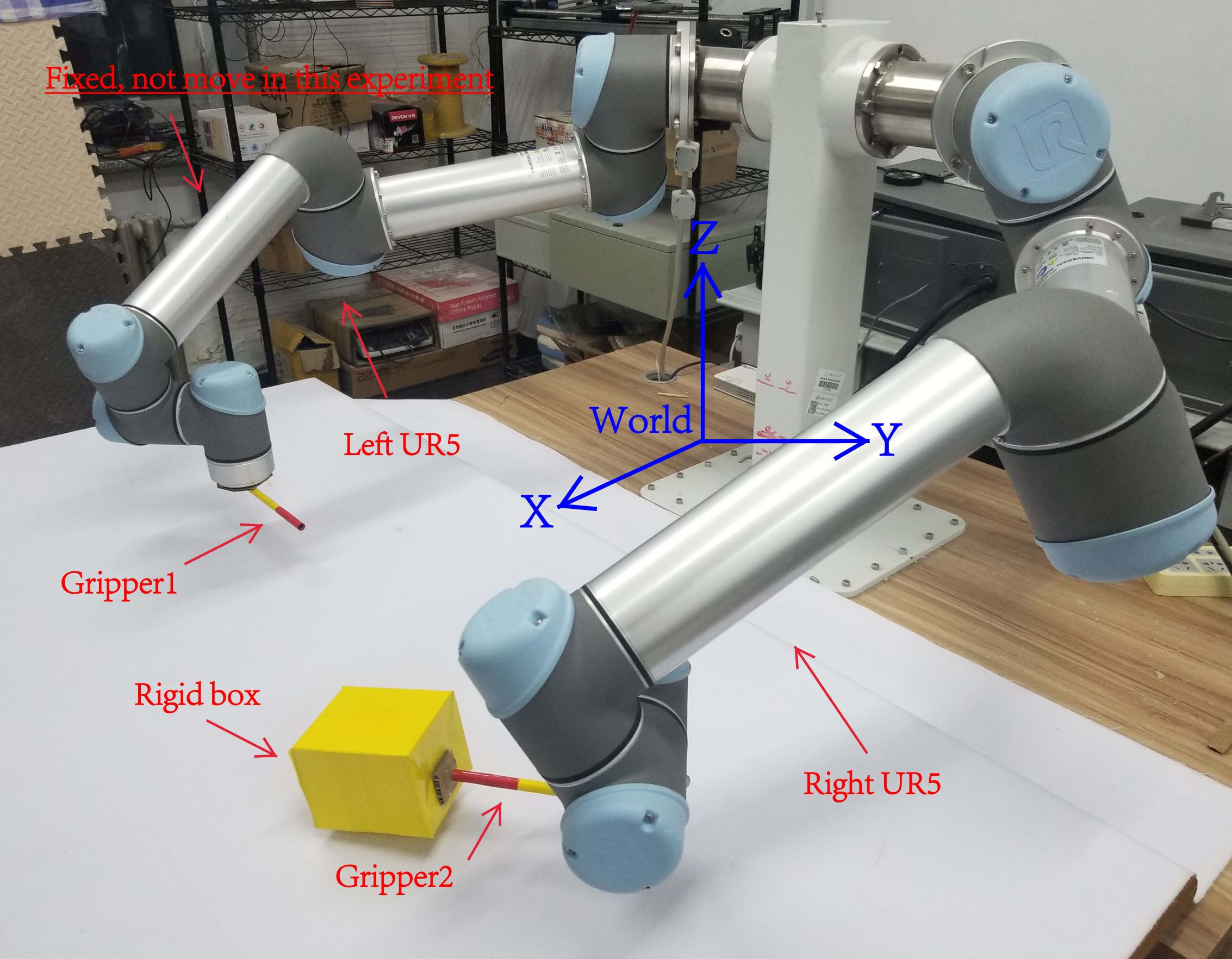}\label{fig13f}}

\caption{Dual-arm UR5 experiment platform manipulating CRDO.}
\label{fig18}
\end{figure}

\section{EXPERIMENTS}\label{sec7}
\subsection{Setup}
To further validate the effectiveness of our proposed framework, we conducted an experimental study with a dual-arm robotic platform (composed of two UR5 robots) manipulating several types of CRDO: Elastic cable, sponge, plastic folder, non-homogeneous (NH) beam, articulated wallet, and a rigid box, see Fig. \ref{fig18}.
A Logitech C270 camera is used to capture the shape of the objects; All image processing algorithms are implemented in a Linux PC at 30 FPS with OpenCV. 
The experimental video can be downloaded from:  \url{https://github.com/q546163199/experiment_video/tree/master/paper3/video.mp4}

\begin{figure}
\centering
\subfloat[ROI selection]{\includegraphics[scale=0.095]{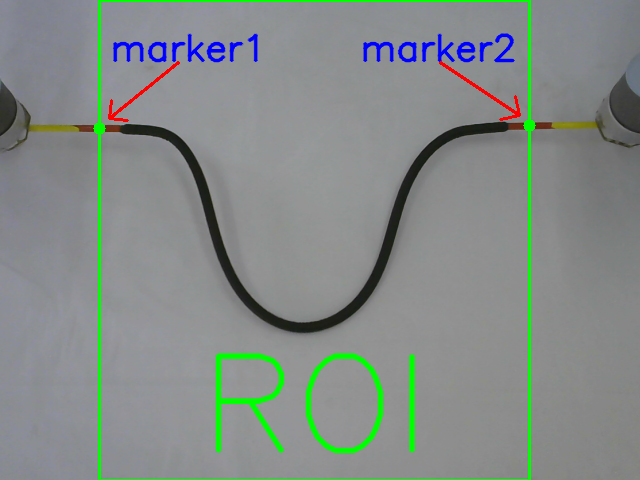}\label{fig18a}}
\subfloat[Binary]{\includegraphics[scale=0.095]{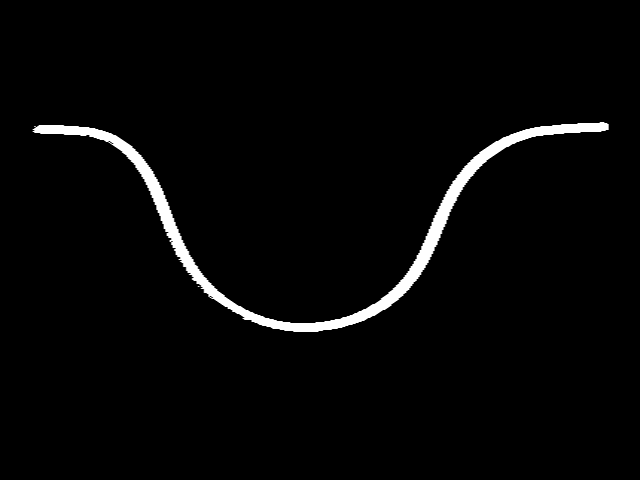}\label{fig18b}}
\subfloat[Extraction]{\includegraphics[scale=0.095]{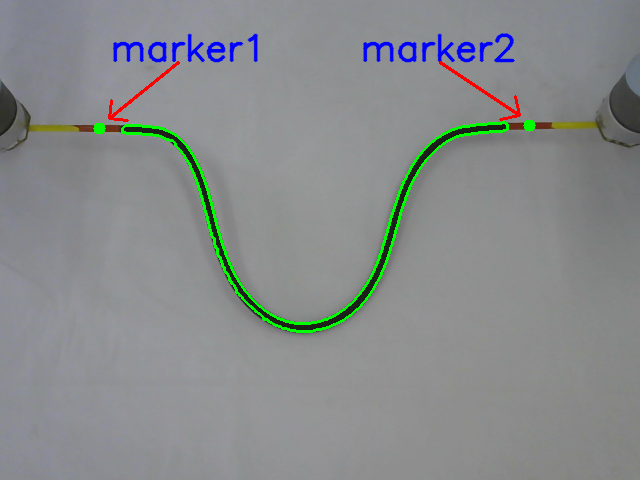}\label{fig18c}}
\subfloat[Fixed-sampled]{\includegraphics[scale=0.095]{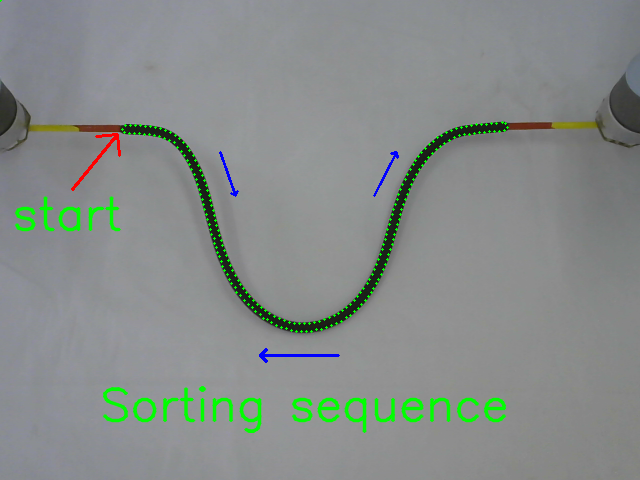}\label{fig18d}}

\subfloat[ROI selection]{\includegraphics[scale=0.095]{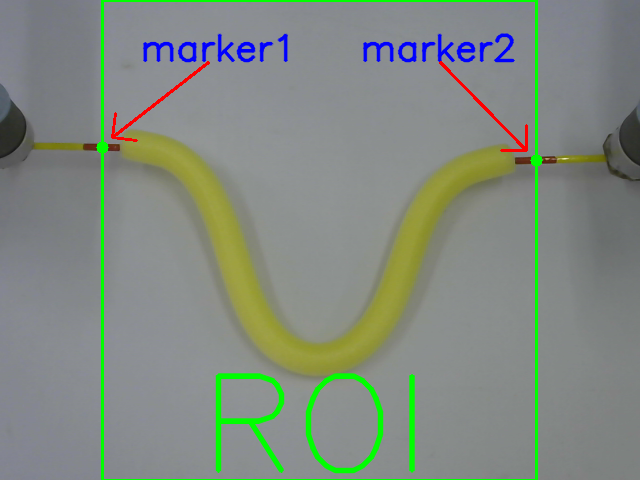}\label{fig18e}}
\subfloat[HSV]{\includegraphics[scale=0.095]{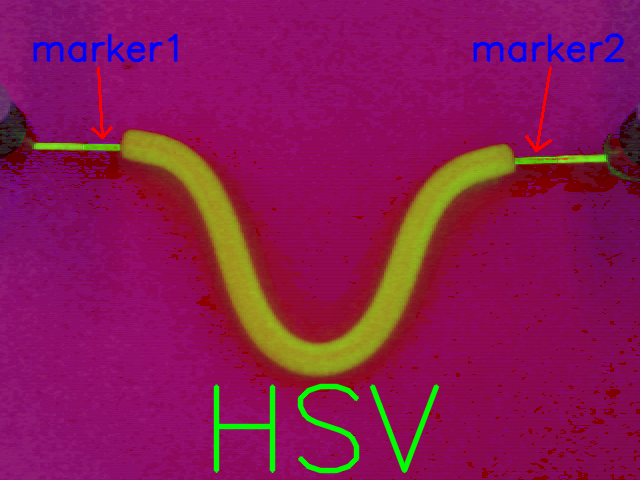}\label{fig18f}}
\subfloat[Extraction]{\includegraphics[scale=0.095]{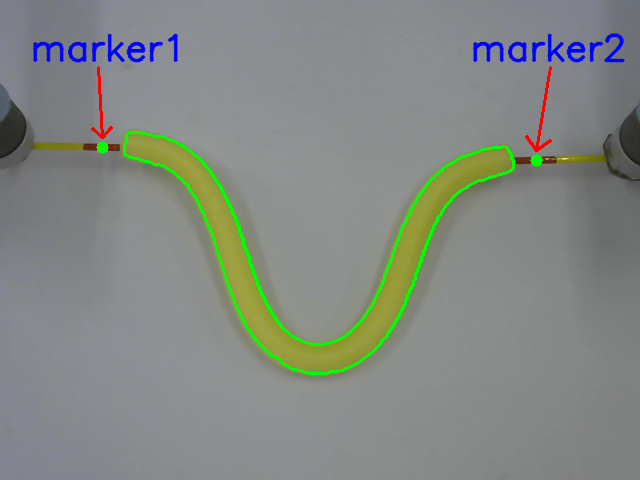}\label{fig18g}}
\subfloat[Fixed-sampled]{\includegraphics[scale=0.095]{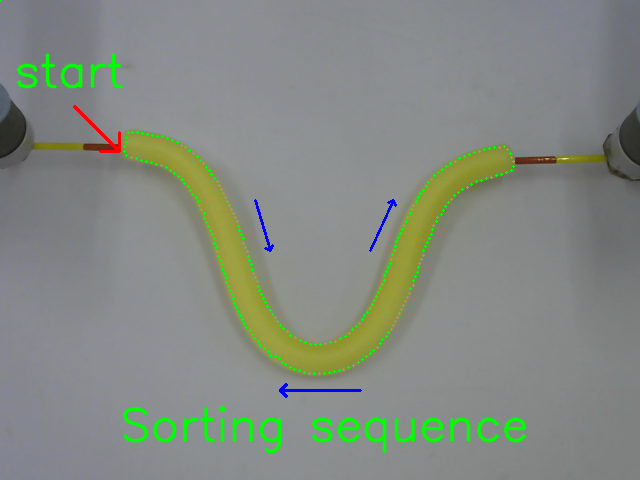}\label{fig18h}}

\caption{Image processing for generating a fixed-sampled closed contour of various objects.}
\label{fig16}
\end{figure}

\subsection{Image Processing}
The contour extraction process (depicted in Fig. \ref{fig16}) is as follows: The red areas near the grippers are first extracted, and their centroids computed and marked with green points. The object's region of interest (ROI) is defined by these points.
For the manipulated black objects, their binary image is extracted from the ROI by using OpenCV's morphological algorithms for denoising.
For the manipulated yellow objects, we transform their image into HSV and then conduct mask processing to obtain the binary image.
The contour of the objects (either black or yellow) is then simply extracted by using OpenCV's \emph{findcontour} function.
Finally, a fixed-step (set to $N=300$) algorithm \cite{qi2020adaptive} is used to construct the contour $\bar{\mathbf c}$ with a constant number of points $\mathbf c_i$.

\begin{figure}
\centering
\subfloat[Linear sponge]
{\includegraphics[scale=0.17]{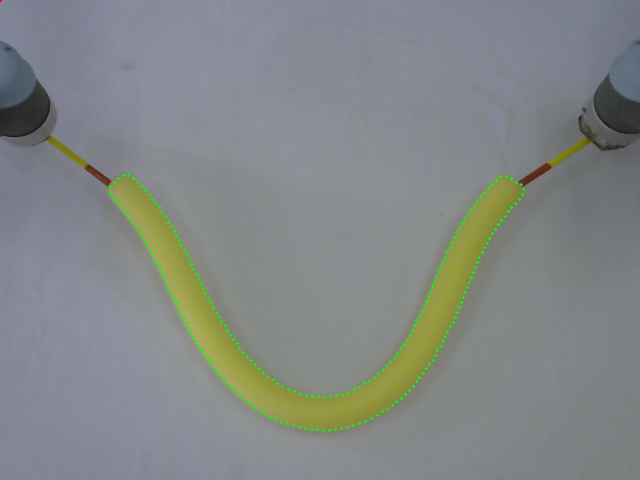}}
\subfloat[NH beam]
{\includegraphics[scale=0.17]{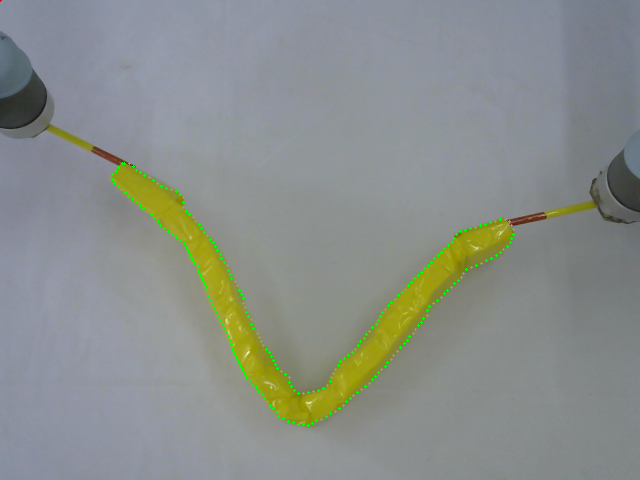}}
\subfloat[Rigid box]
{\includegraphics[scale=0.17]{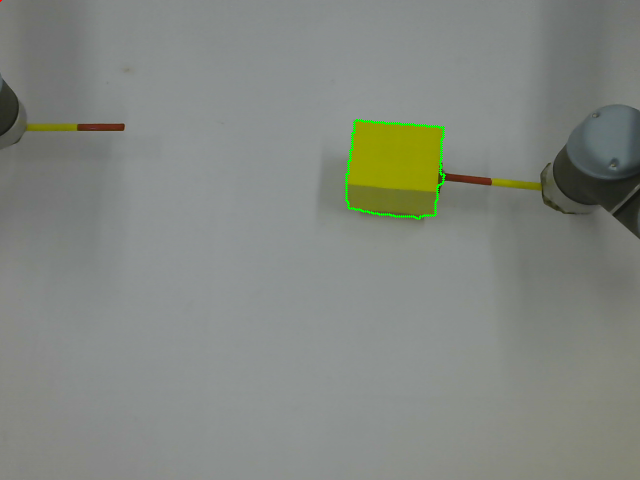}}

\subfloat[Linear sponge]
{\includegraphics[scale=0.17]{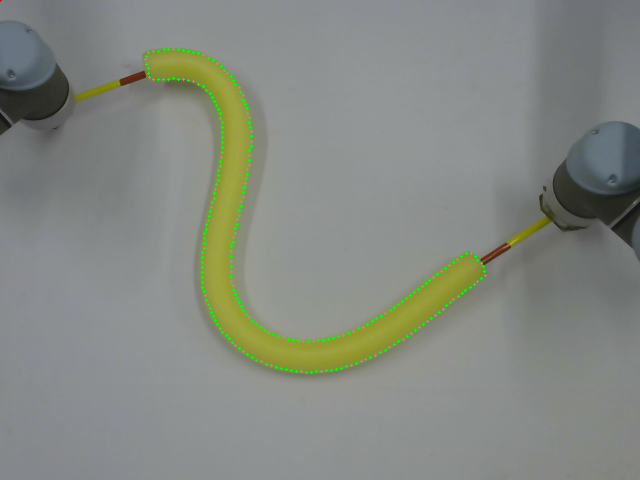}}
\subfloat[NH beam]
{\includegraphics[scale=0.17]{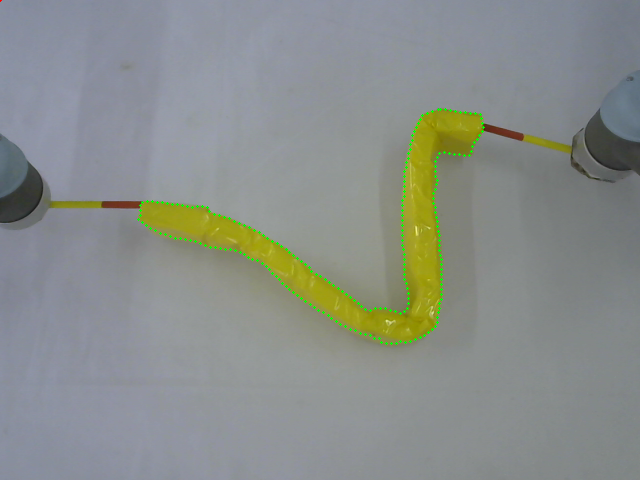}}
\subfloat[Rigid box]
{\includegraphics[scale=0.17]{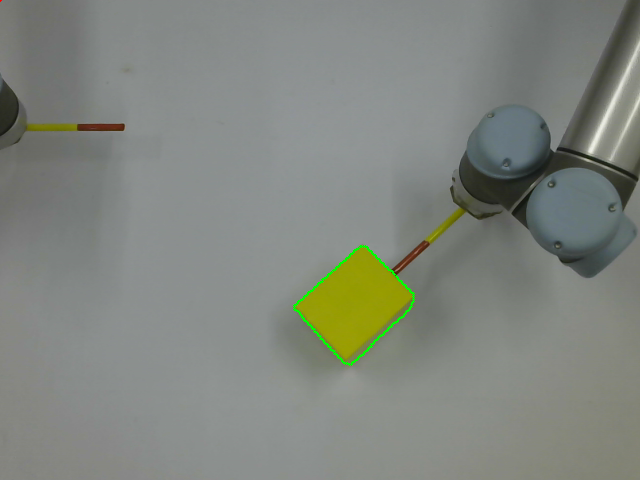}}

\caption{Different shapes of various objects manipulated by dual-arm UR5.}
\label{fig7}
\end{figure}

\begin{figure}
\centering
\includegraphics[scale=0.27]{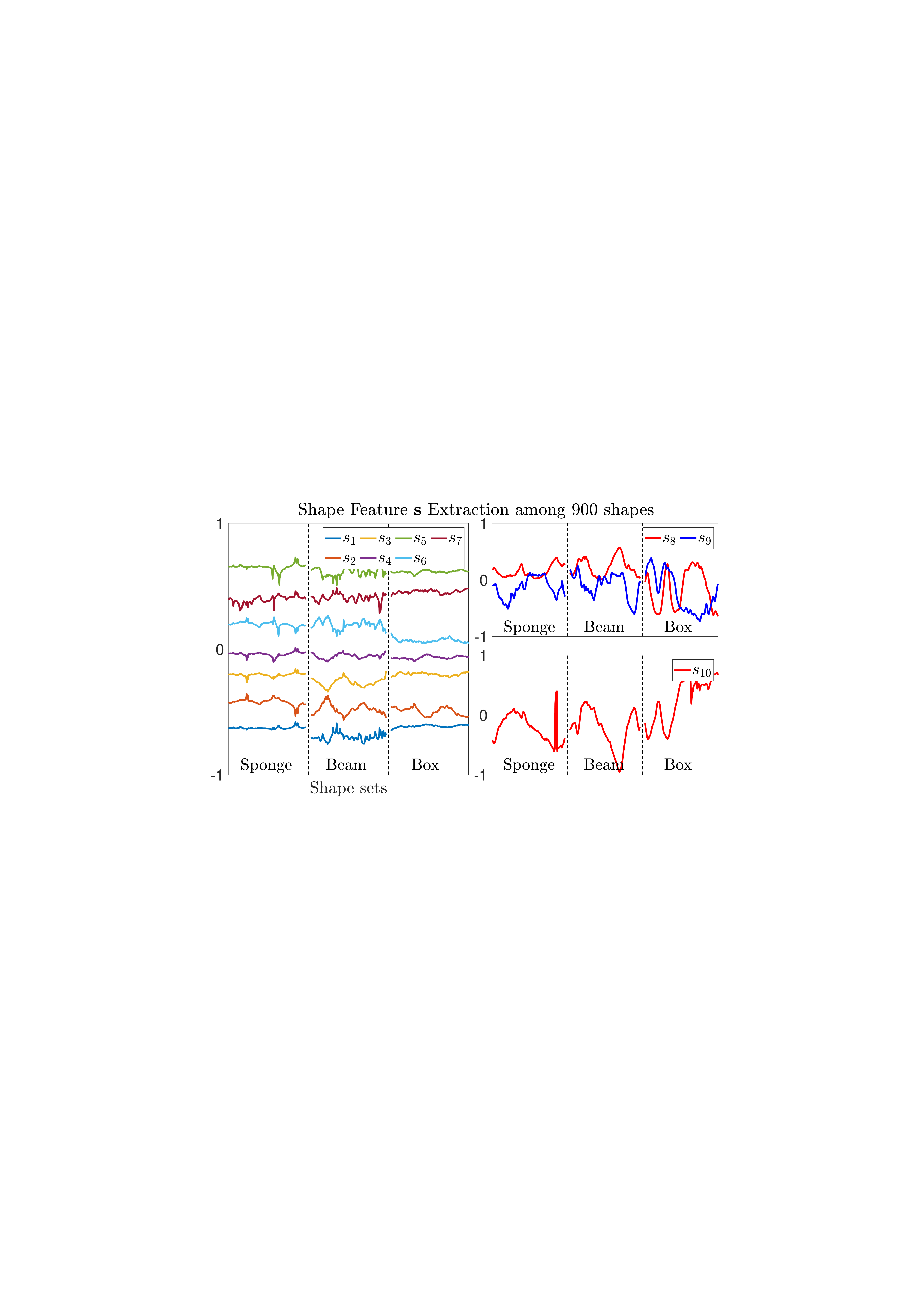}
\caption{Contour moments extraction \eqref{eq2} - \eqref{eq65} of 900 shapes among sponge, beam and box.
Each group has 300 shapes, separately.}
\label{fig42}
\end{figure}


\subsection{Evaluation of Contour Moments and DJM}\label{sec7a}
The robotic platform is commanded to continuously deform three CRDO into various shapes, see Fig. \ref{fig7} (such shaping actions are demonstrated in the accompanying multimedia file).
The profiles of the feature coordinates $s_i$ obtained from these robot motions are depicted in Fig. \ref{fig42}.
The graphs show that the contour moments can smoothly represent the infinite-dimensional object's shape, results that are consistent with its simulation counterpart.

The accuracy of the computed matrix $\hat{\mathbf J}_s$ is also validated by commanding the robot to move along an arbitrary trajectory with both grippers manipulating the elastic cable.
Fig. \ref{fig9} shows a performance comparison of the metrics $T_1$ and $\|\mathbf e_2\|$. 
These results also confirm that FTSMC provides a superior performance in the estimation of the unknown Jacobian matrix.

\begin{figure}
\centering
\includegraphics[scale=0.25]{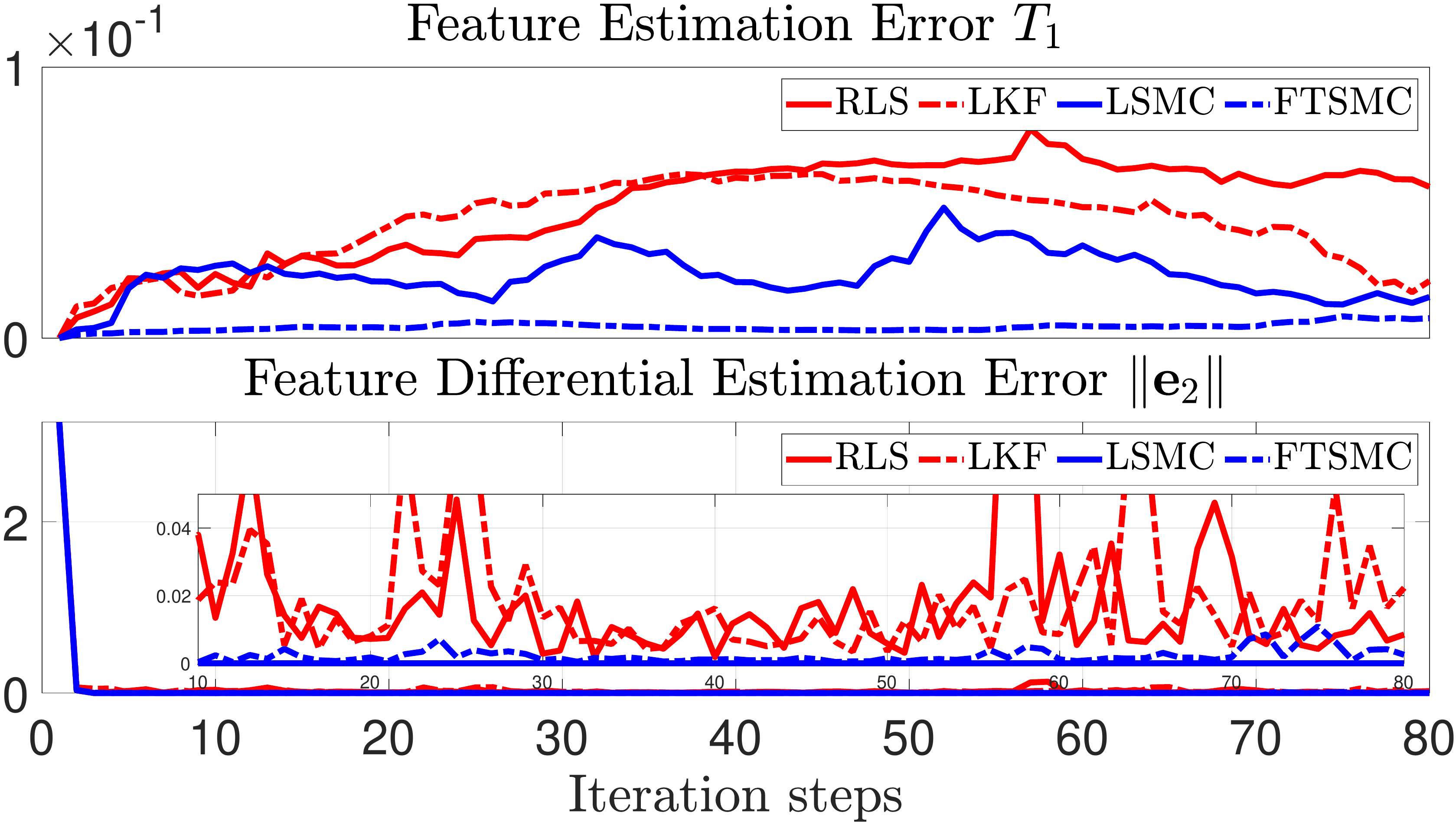}
\caption{Profiles of the criteria $T_1$ and $\| \mathbf{e}_2 \|$ with dual-arm UR5 executing a given trajectory $\mathbf{r}$.
It gives the estimation effect of $\hat{\mathbf{J}}_s$ among
RLS \cite{hosoda1994versatile}, 
LKF \cite{qian2002online}, 
LSMC \eqref{eq15} and FTSMC \eqref{eq31}, respectively.}
\label{fig9}
\end{figure}

\subsection{Shape Servoing with CRDO}
We conducted an experimental study where the robot manipulates six objects (with varying mechanical properties): Elastic cable, linear sponge, plastic folder, non-homogeneous beam, articulated wallet, rigid box. 
We denote the experiments with these objects as Exp1\ldots Exp6, respectively.
Throughout this section, we shall compare the performance of the proposed FTSMC with LSMC, LKF and RLS (a classical visual servoing controller is used with the last two).

Fig. \ref{fig10} illustrates the active shaping motions (represented by the moving green contours) of the objects towards the desired target configuration (represented by the red contour).
This figure qualitatively depicts the contour trajectories of the six objects (each in a different row) with the four controllers (each in a different column). 
Fig. \ref{fig15} quantitatively depicts the time evolution of the shape error $\|\mathbf e_1\|$ and the respective driving control inputs $\mathbf u$.
From these temporal profiles, we can see that the proposed FTSMC provides the best control performance compared to the other methods. 

A summary of the performance metrics for the conducted manipulation experiments is depicted in Fig. \ref{fig13}.
This figure quantitatively shows that the proposed control method achieves the best relative to the time instance $t_d$ (which means that the elastic cable can be coarsely deformed into the target configuration), the instance $t_s$ (a property introduced by the terminal attractor $\sig(\cdot)$), and the IAE index (which demonstrates the moderate consumption of energy).

\begin{figure*}
\centering
\subfloat[Exp1-Cable RLS]
{\includegraphics[scale=0.18]{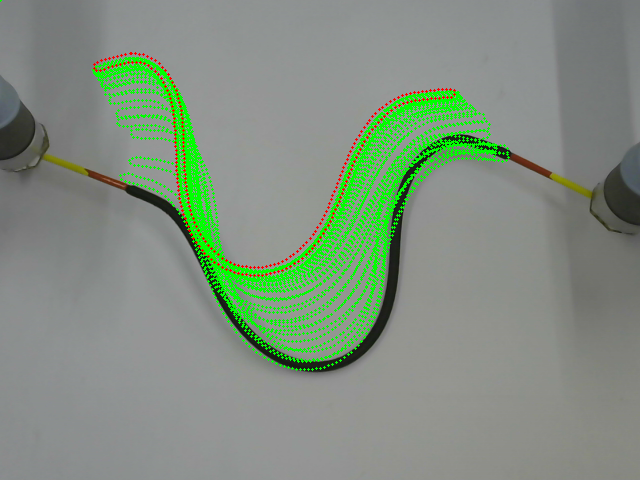}\label{fig11a1}}
\subfloat[Exp1-Cable LKF]
{\includegraphics[scale=0.18]{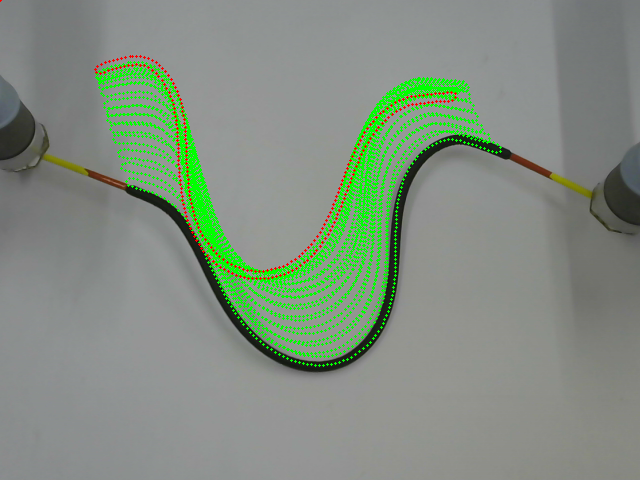}\label{fig11a2}}
\subfloat[Exp1-Cable LSMC]
{\includegraphics[scale=0.18]{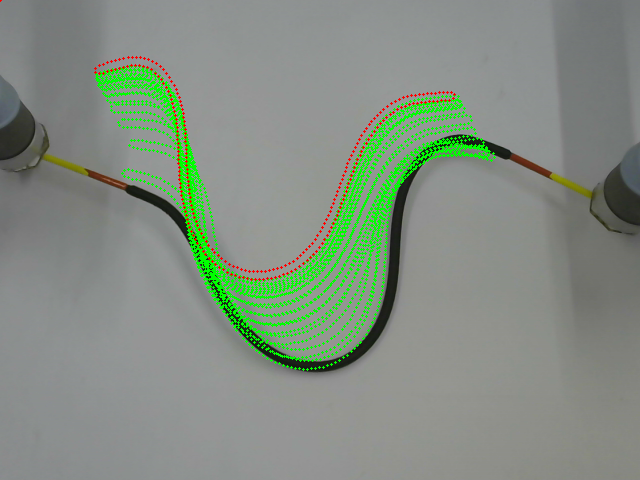}\label{fig11a3}}
\subfloat[Exp1-Cable FTSMC]
{\includegraphics[scale=0.18]{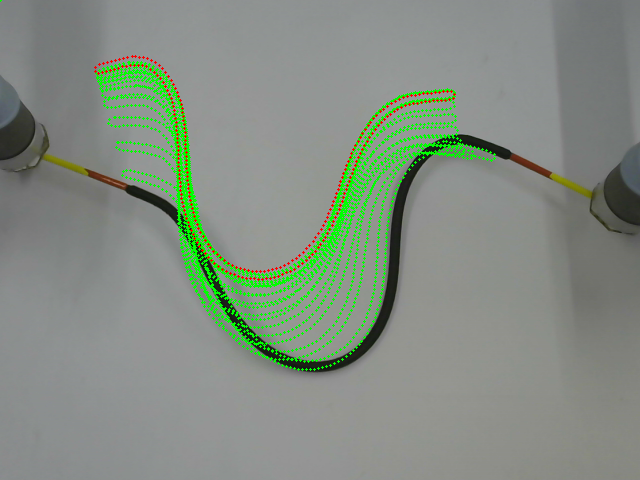}\label{fig11a4}}
	
\subfloat[Exp2-Sponge RLS]{\includegraphics[scale=0.18]{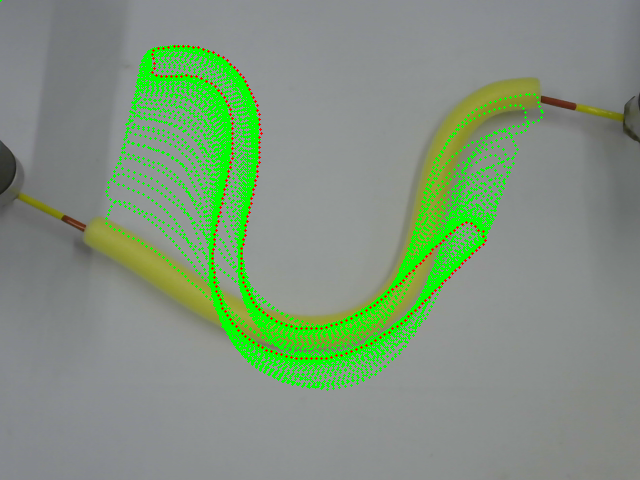}\label{fig11b1}}
\subfloat[Exp2-Sponge LKF]{\includegraphics[scale=0.18]{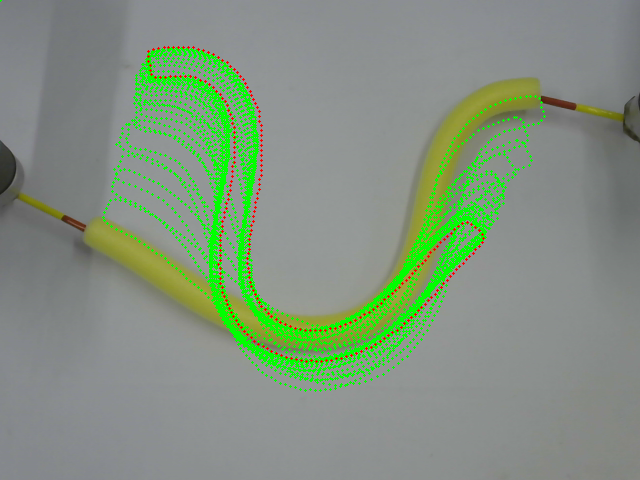}\label{fig11b2}}
\subfloat[Exp2-Sponge LSMC]{\includegraphics[scale=0.18]{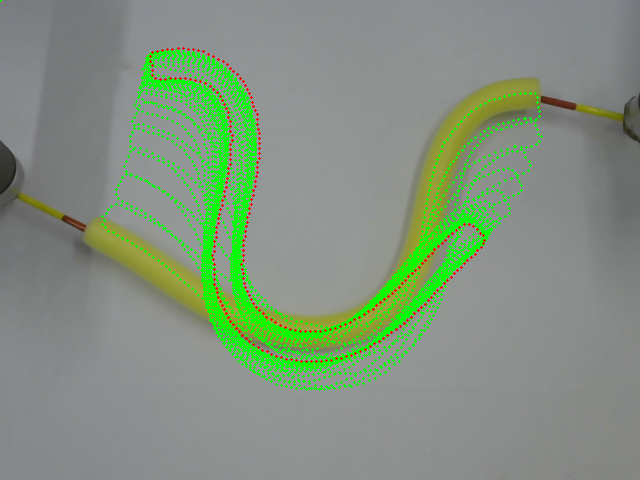}\label{fig11b3}}
\subfloat[Exp2-Sponge FTSMC]{\includegraphics[scale=0.18]{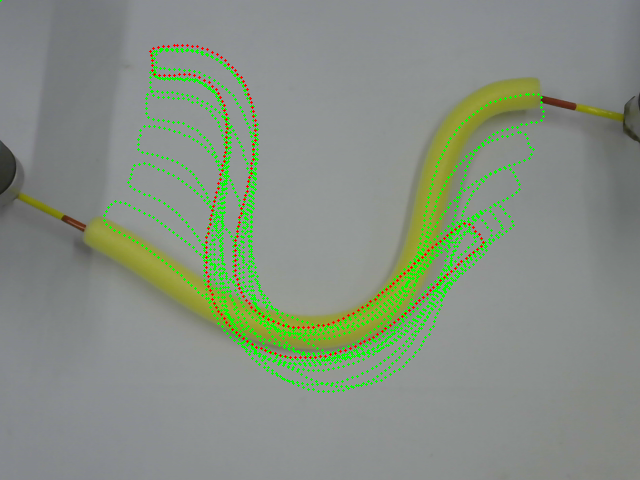}\label{fig11b4}}

\subfloat[Exp3-Folder RLS]{\includegraphics[scale=0.18]{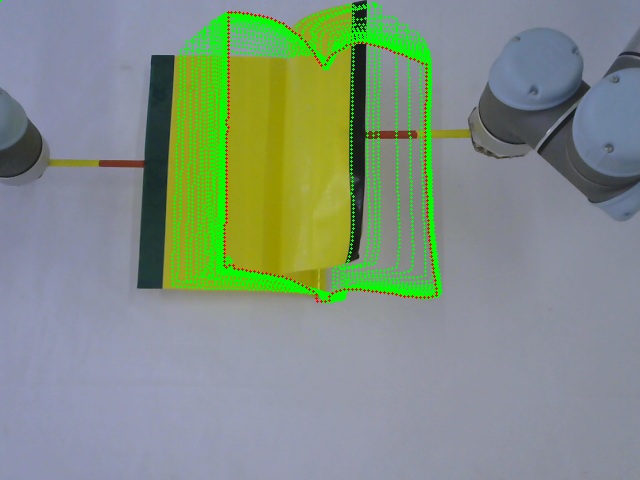}\label{fig11c1}}
\subfloat[Exp3-Folder LKF]{\includegraphics[scale=0.18]{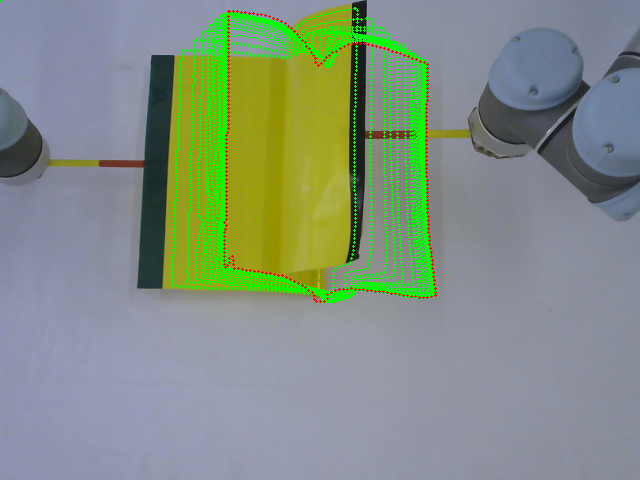}\label{fig11c2}}
\subfloat[Exp3-Folder LSMC]{\includegraphics[scale=0.18]{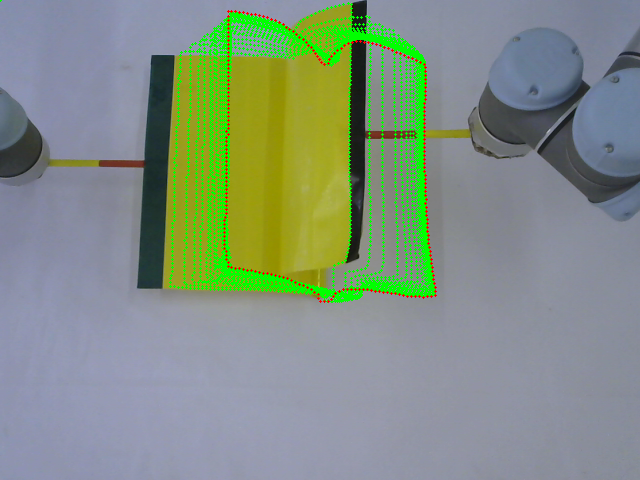}\label{fig11c3}}
\subfloat[Exp3-Folder FTSMC]{\includegraphics[scale=0.18]{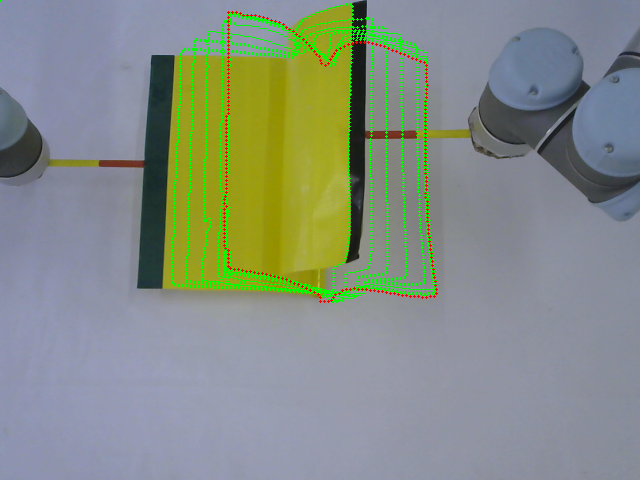}\label{fig11c4}}

\subfloat[Exp4-Beam RLS]{\includegraphics[scale=0.18]{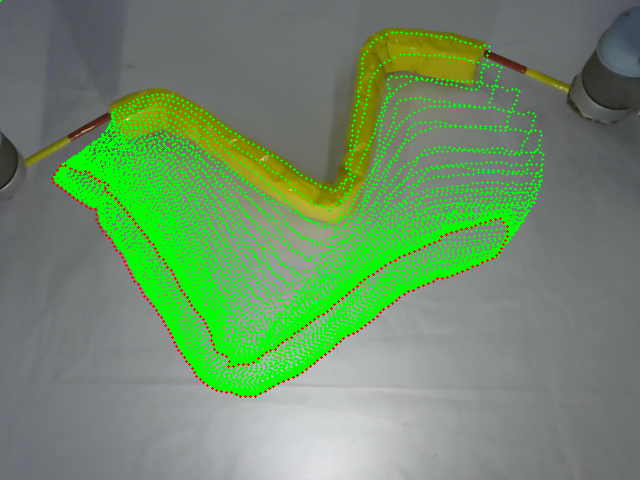}\label{fig11d1}}
\subfloat[Exp4-Beam LKF]{\includegraphics[scale=0.18]{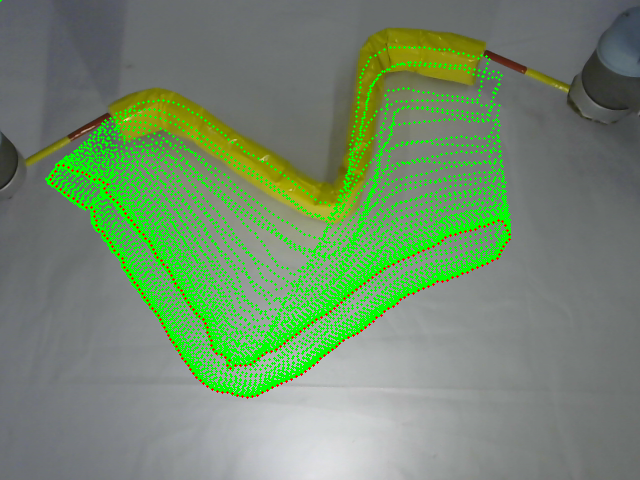}\label{fig11d2}}
\subfloat[Exp4-Beam LSMC]{\includegraphics[scale=0.18]{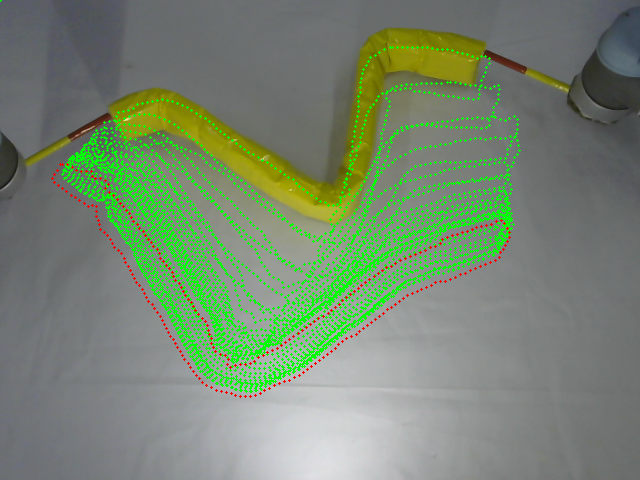}\label{fig11d3}}
\subfloat[Exp4-Beam FTSMC]{\includegraphics[scale=0.18]{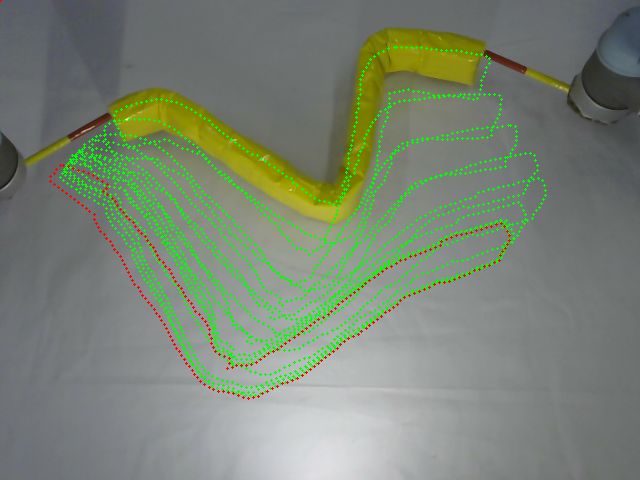}\label{fig11d4}}

\subfloat[Exp5-Wallet RLS]{\includegraphics[scale=0.18]{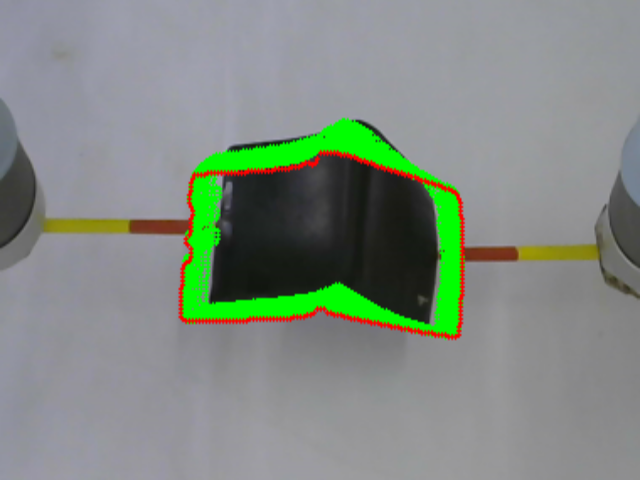}\label{fig11e1}}
\subfloat[Exp5-Wallet LKF]{\includegraphics[scale=0.18]{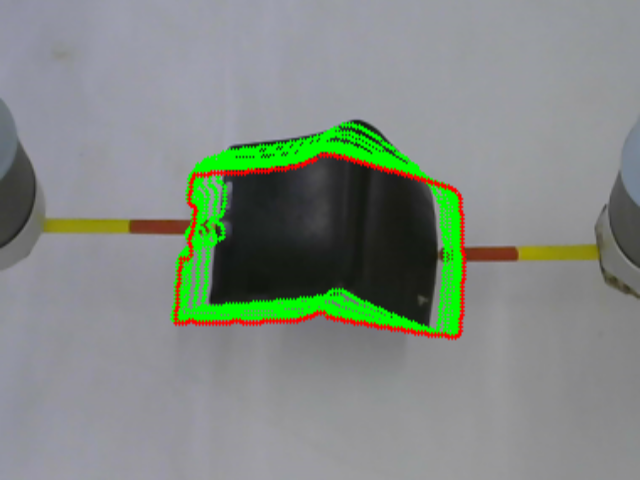}\label{fig11e2}}
\subfloat[Exp5-Wallet LSMC]{\includegraphics[scale=0.18]{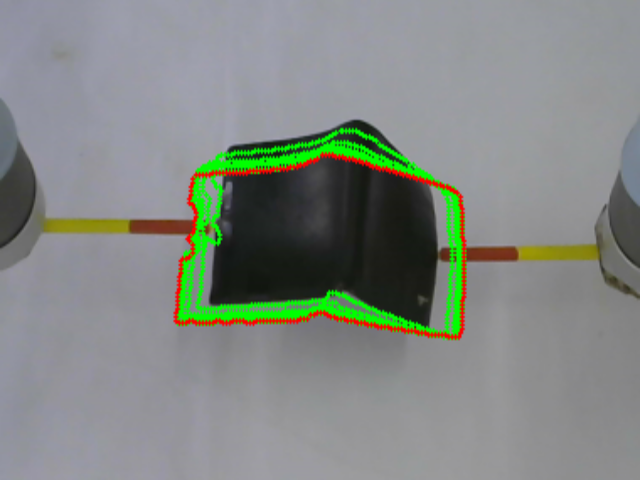}\label{fig11e3}}
\subfloat[Exp5-Wallet FTSMC]{\includegraphics[scale=0.18]{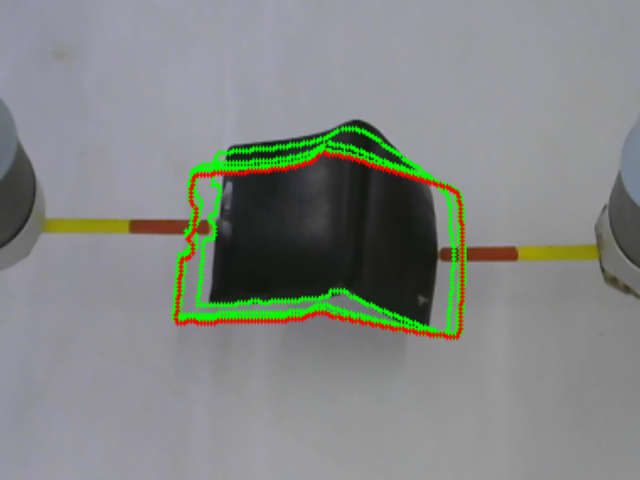}\label{fig11e4}}

\subfloat[Exp6-Rigid-Box RLS]{\includegraphics[scale=0.18]{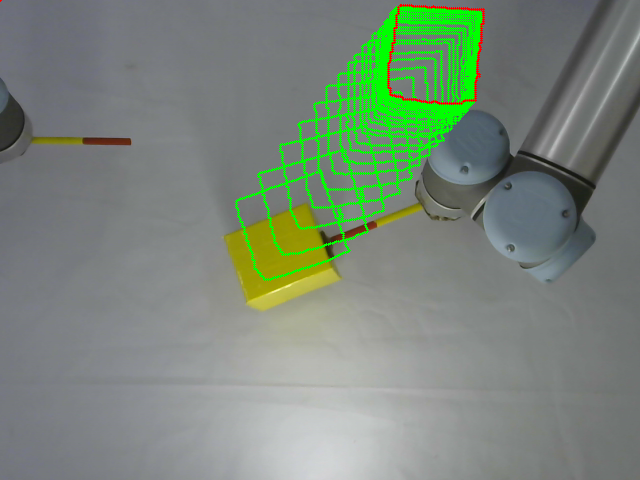}\label{fig11f1}}
\subfloat[Exp6-Rigid-Box LKF]{\includegraphics[scale=0.18]{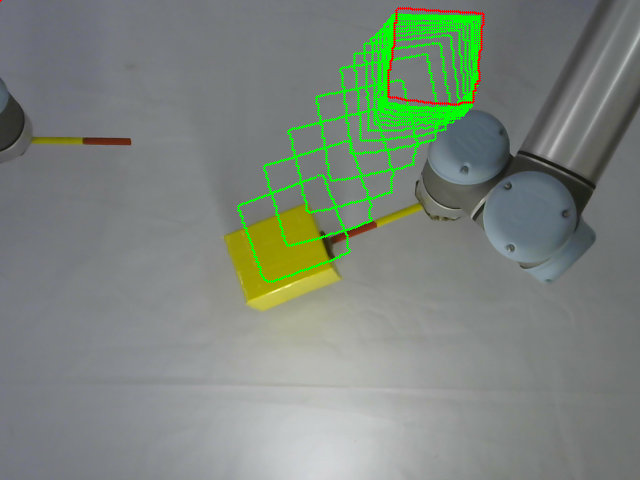}\label{fig11f2}}
\subfloat[Exp6-Rigid-Box LSMC]{\includegraphics[scale=0.18]{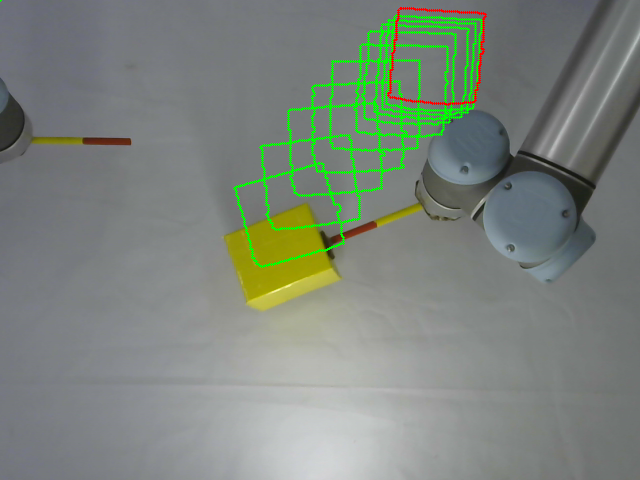}\label{fig11f3}}
\subfloat[Exp6-Rigid-Box FTSMC]{\includegraphics[scale=0.18]{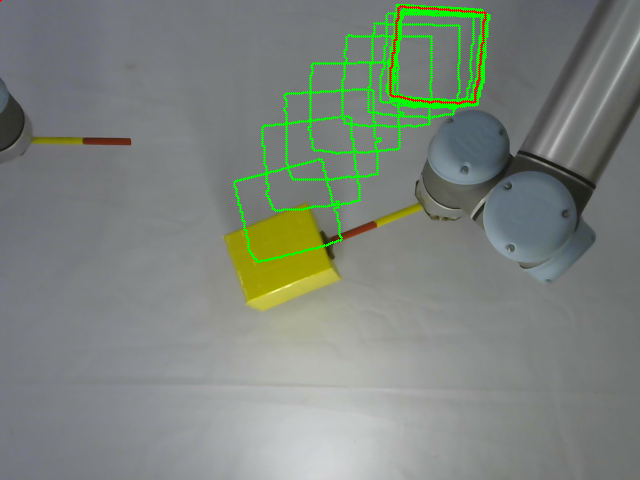}\label{fig11f4}}
	
\caption{Initial (black solid line), transition (green solid line) and target (red solid line) contour deformation trajectories from Exp1 to Exp5.
Exp6 represents the positioning of the rigid box.
All six experiments are conducted among  
RLS \cite{hosoda1994versatile}\cite{qi2020adaptive}, 
LKF \cite{qian2002online}\cite{qi2020adaptive}, 
LSMC \eqref{eq10}\eqref{eq15} and FTSMC \eqref{eq25}\eqref{eq31}.}
\label{fig10}
\end{figure*}

\begin{figure*}
\centering
\subfloat[Exp1 Result]
{\includegraphics[scale=0.17]{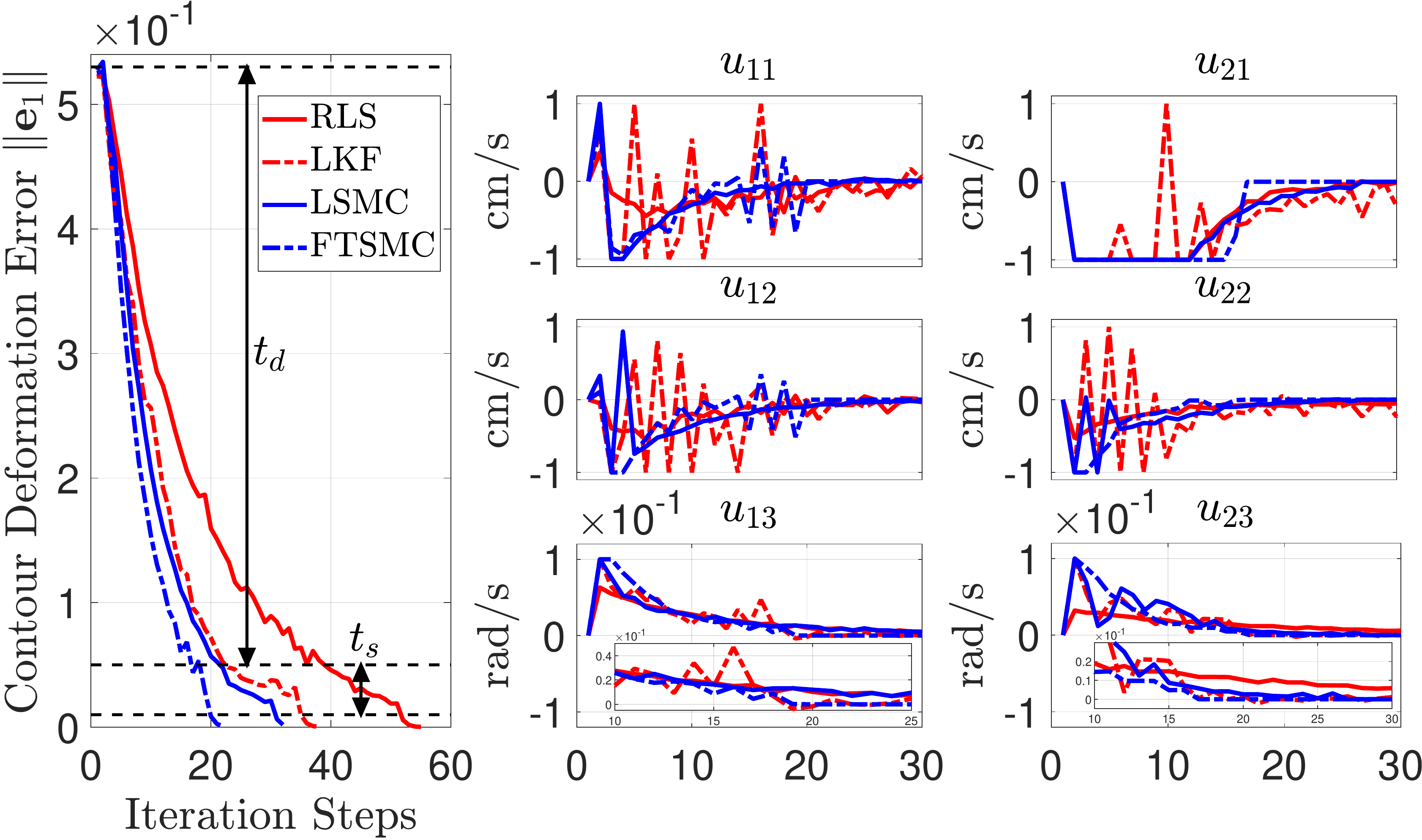}\label{fig12a1}}
\subfloat[Exp2 Result]
{\includegraphics[scale=0.17]{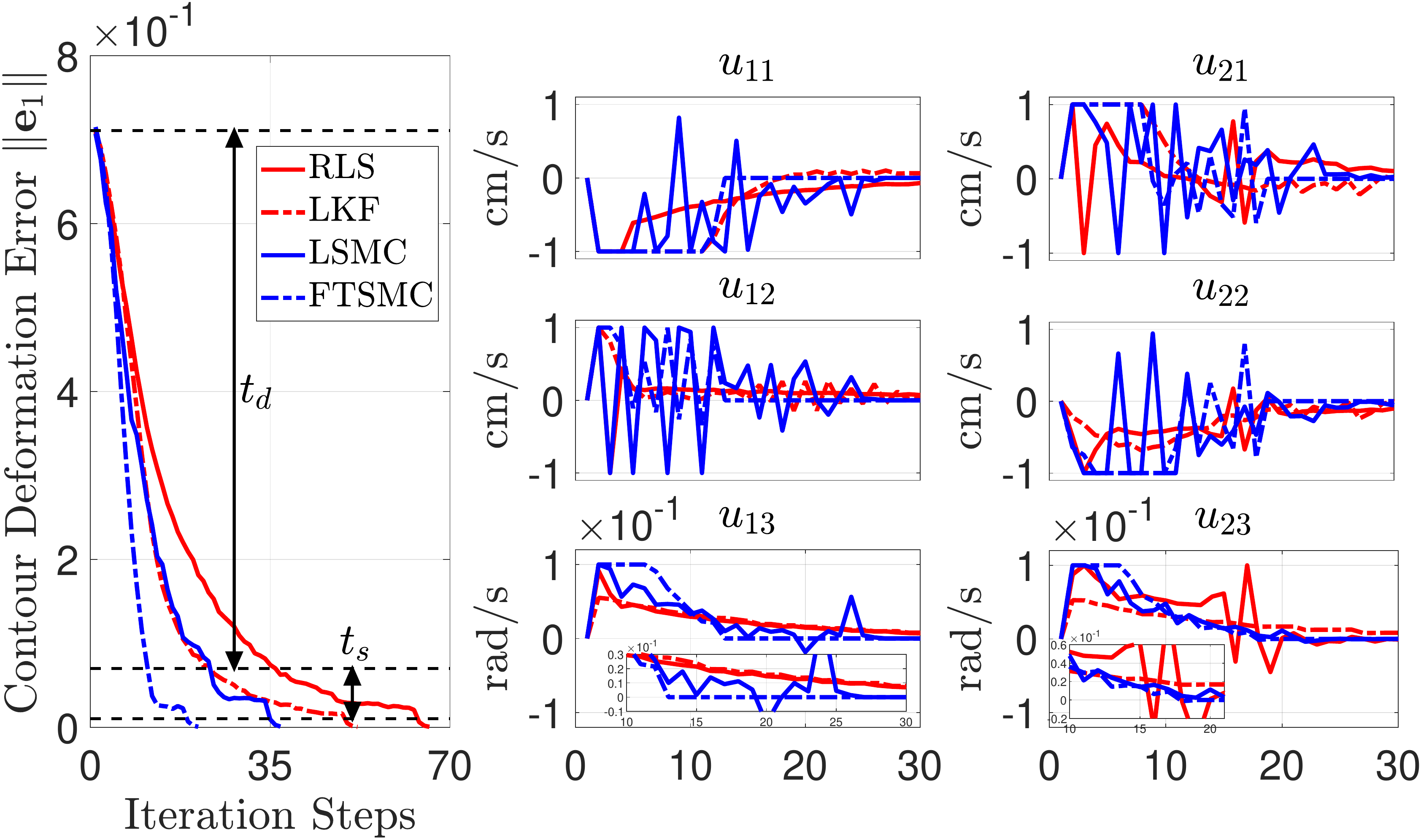}\label{fig12a2}}
\subfloat[Exp3 Result]
{\includegraphics[scale=0.17]{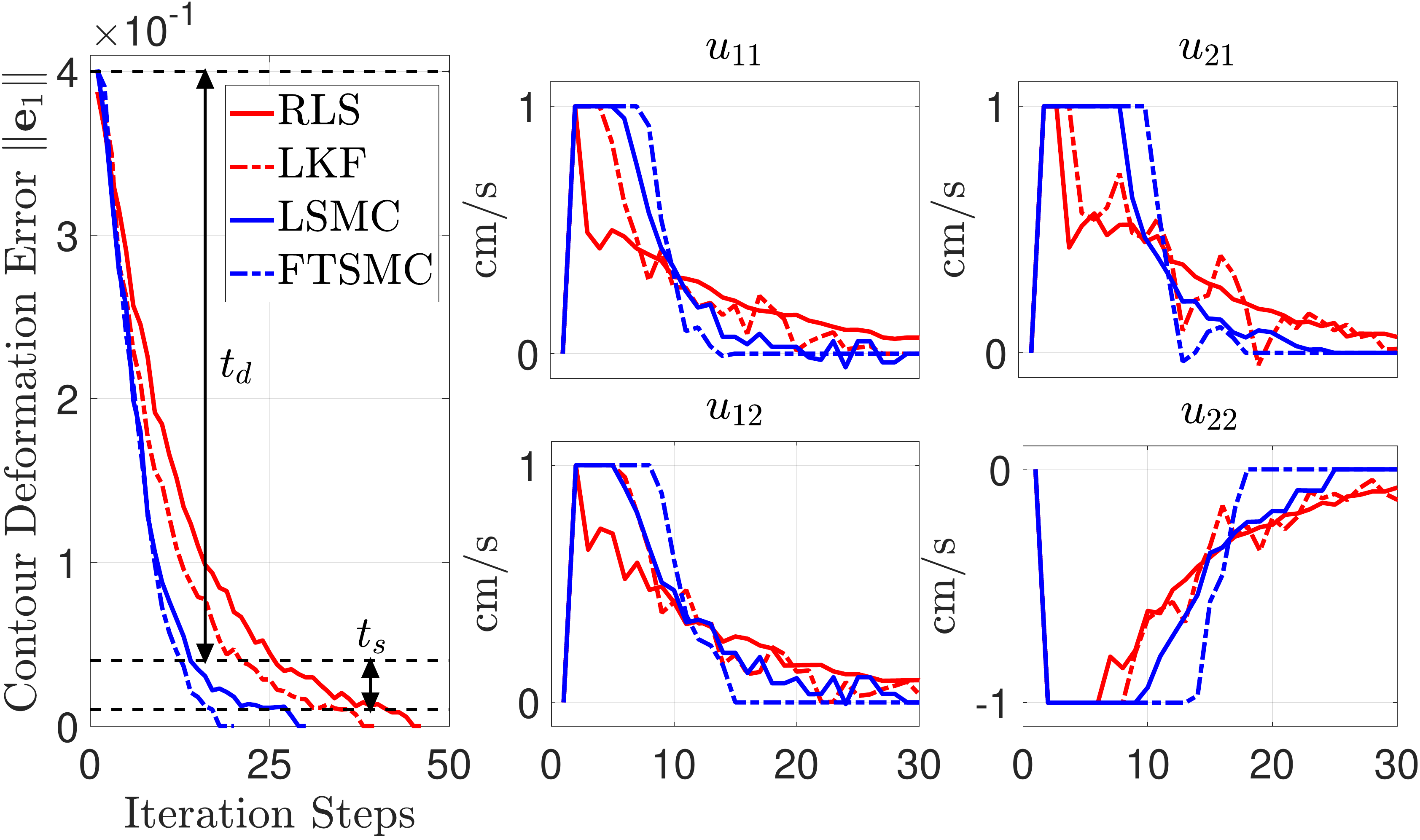}\label{fig12a3}}

\subfloat[Exp4 Result]
{\includegraphics[scale=0.17]{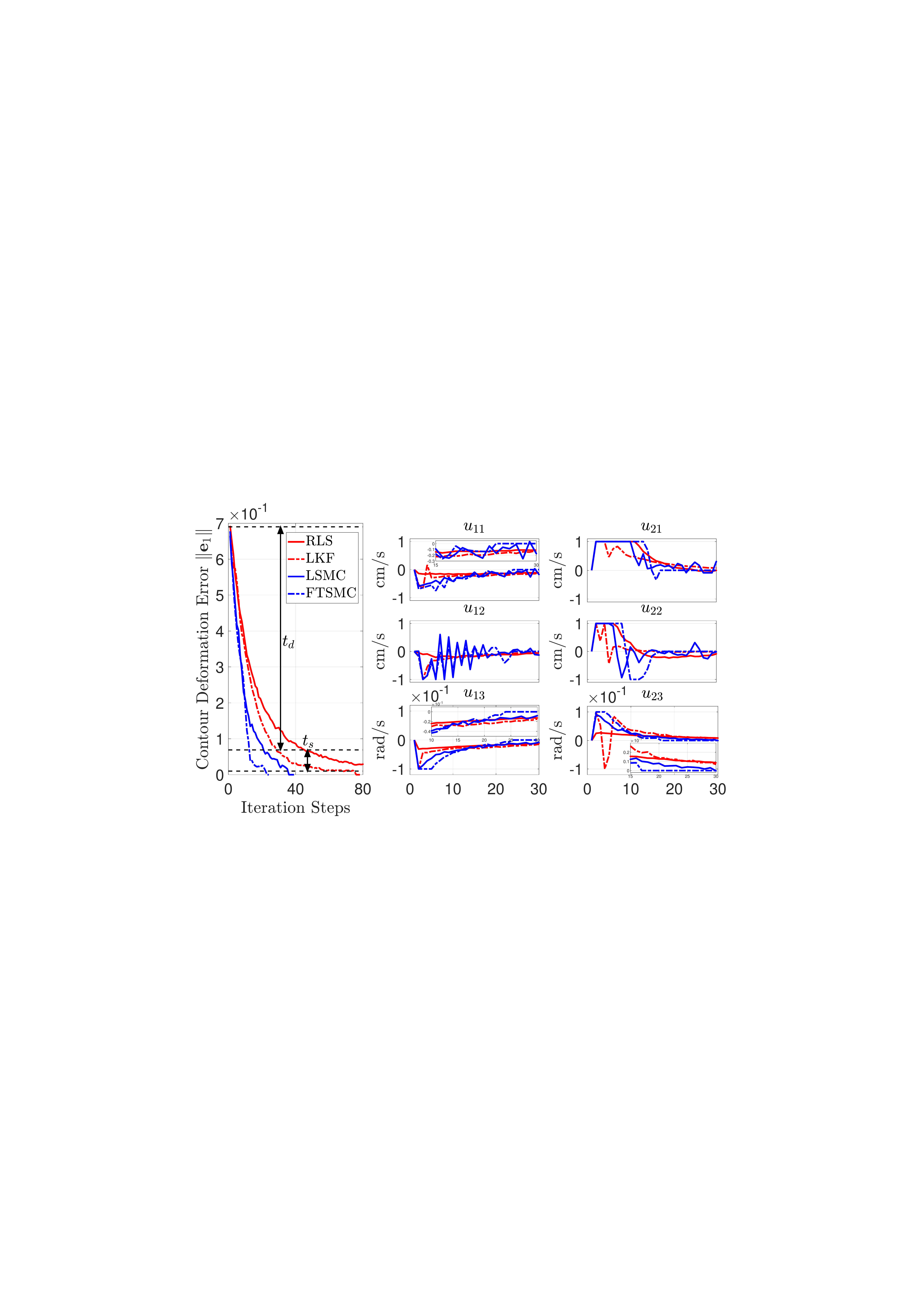}\label{fig12a4}}
\subfloat[Exp5 Result]
{\includegraphics[scale=0.17]{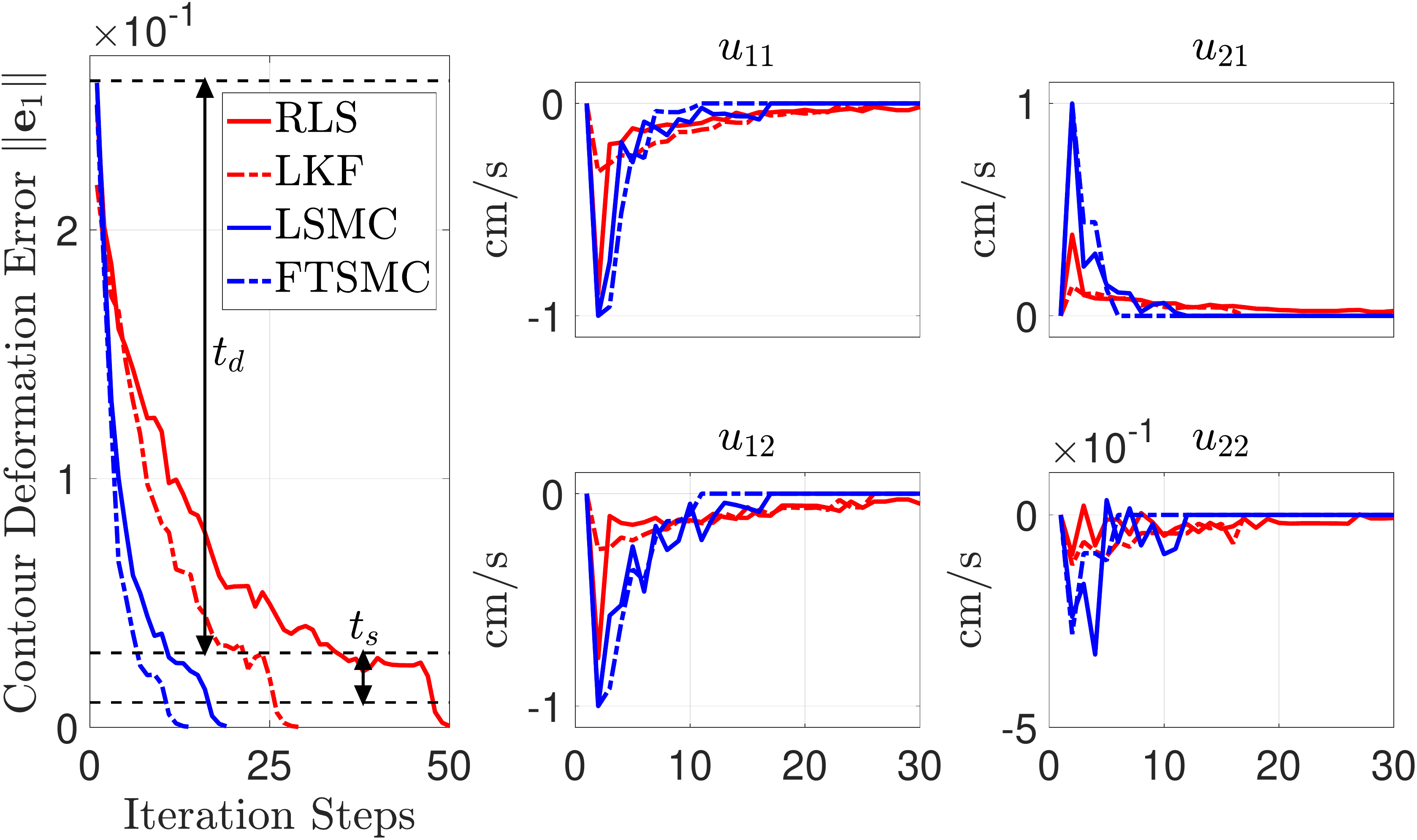}\label{fig12a5}}
\subfloat[Exp6 Result]
{\includegraphics[scale=0.17]{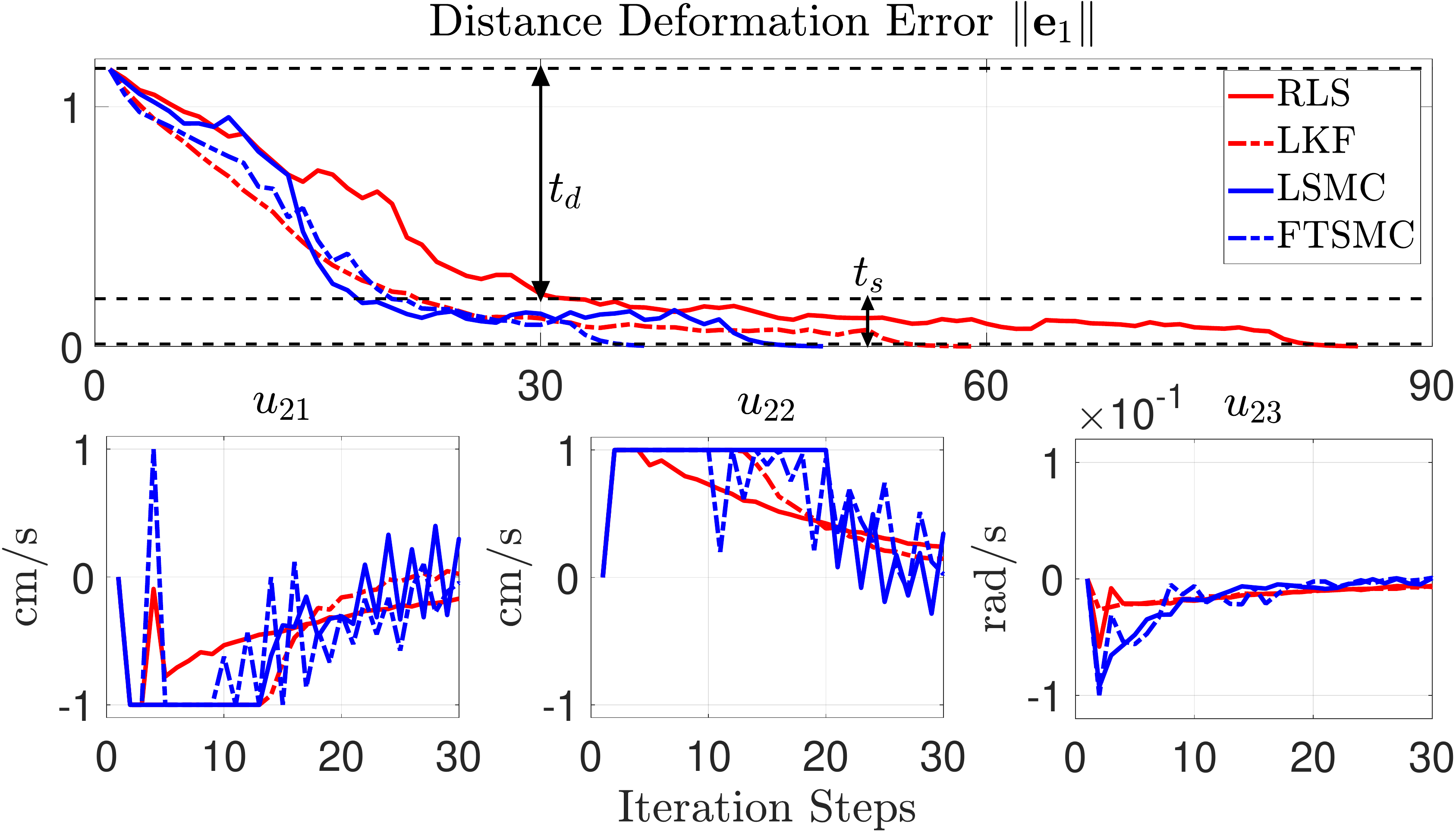}\label{fig12a6}}

\caption{
Profiles of contour deformation error $\| \mathbf{e}_1 \|$ of Exp1-Exp5, distance deformation error $\| \mathbf{e}_1 \|$ of Exp6 
and velocity command $\mathbf{u}$ among four methods, namely, 
RLS \cite{hosoda1994versatile}\cite{qi2020adaptive}, 
LKF \cite{qian2002online}\cite{qi2020adaptive}, 
LSMC \eqref{eq10}\eqref{eq15} and FTSMC \eqref{eq25}\eqref{eq31}.}
\label{fig15}
\end{figure*}

\begin{figure}
\centering
\includegraphics[scale=0.26]{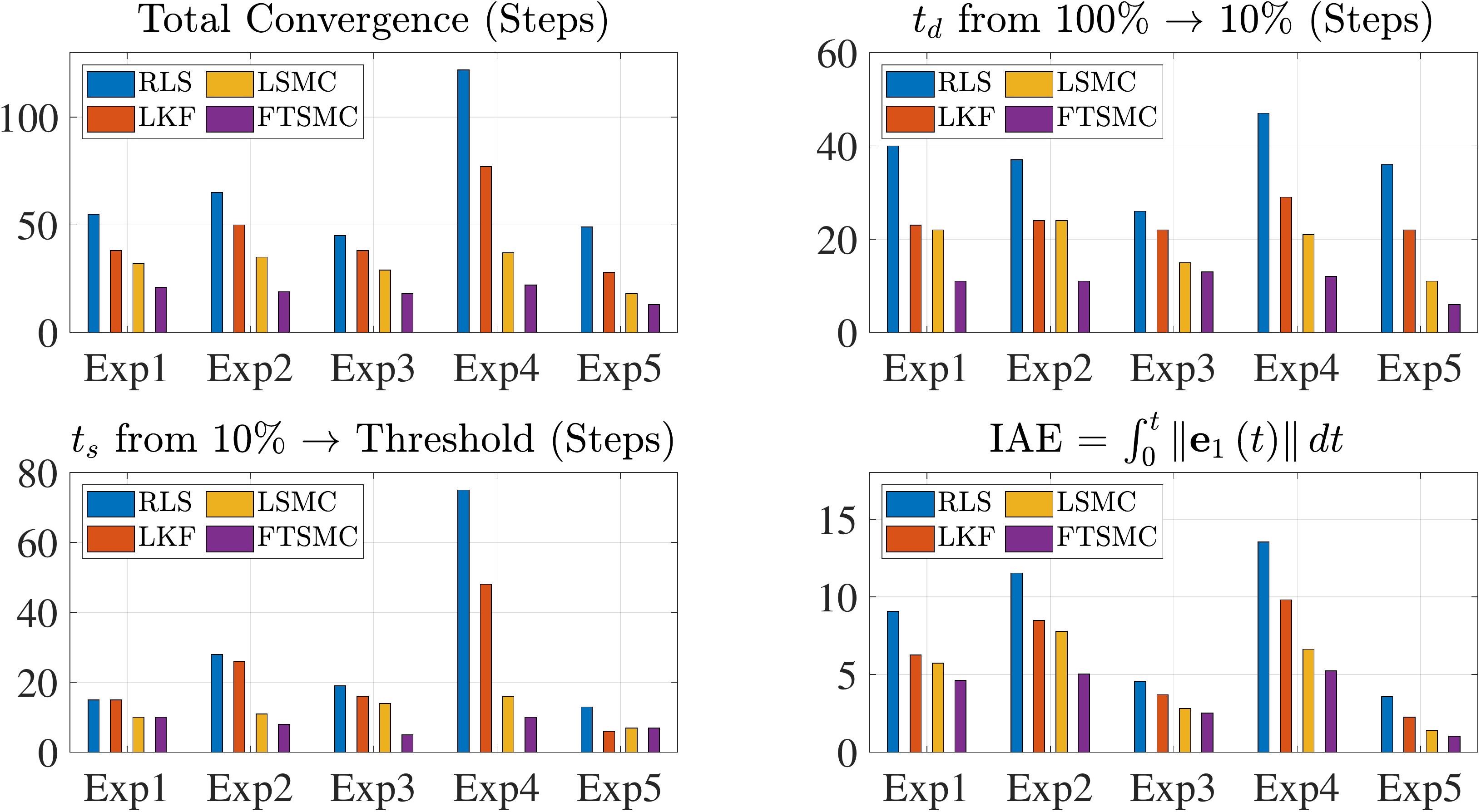}
\caption{Performance comparison in the deformation tasks (Exp1 to Exp5).}
\label{fig13}
\end{figure}

\section{CONCLUSIONS}
In this paper, we present a new visual servoing framework to control the shape of composite rigid-deformable objects with dual-arm robots.
To characterize the objects' infinite-dimensional shape, a compact feature vector is constructed with contour moments from the observed 2D image of the object.
A new finite-time sliding mode controller is proposed to automatically deform the object into a desired shape and simultaneously estimate the deformation Jacobian matrix; The stability of this method is rigorously analyzed.
To validate the performance of our new method, we report a detailed experimental study with a robotic system manipulating multiple objects.


The proposed method has some limitations.
First, this controller cannot be used for purely \emph{plastic} objects, such as certain food materials or clay.
Second, although FTSMC controller provides the best performance, due to the existence of many control parameters, sometimes it is difficult to adjust them to obtain a satisfactory performance.
The hard saturation adopted in this paper is used to limit the control input, however, we do not consider its impact on the system's stability.
For future work, our team is currently extending the current method such as performing 3D manipulation tasks.
Also, we are incorporating shape planning capabilities into the framework so as to conduct complex multi-action shaping tasks (e.g. packing non-rigid objects).

\appendices
\ifCLASSOPTIONcaptionsoff
  \newpage
\fi

\bibliography{biblio.bib,tom_biblio.bib}

\begin{thebibliography}{10}
\providecommand{\url}[1]{#1}
\csname url@rmstyle\endcsname
\providecommand{\newblock}{\relax}
\providecommand{\bibinfo}[2]{#2}
\providecommand\BIBentrySTDinterwordspacing{\spaceskip=0pt\relax}
\providecommand\BIBentryALTinterwordstretchfactor{4}
\providecommand\BIBentryALTinterwordspacing{\spaceskip=\fontdimen2\font plus
\BIBentryALTinterwordstretchfactor\fontdimen3\font minus
  \fontdimen4\font\relax}
\providecommand\BIBforeignlanguage[2]{{%
\expandafter\ifx\csname l@#1\endcsname\relax
\typeout{** WARNING: IEEEtran.bst: No hyphenation pattern has been}%
\typeout{** loaded for the language `#1'. Using the pattern for}%
\typeout{** the default language instead.}%
\else
\language=\csname l@#1\endcsname
\fi
#2}}

\bibitem{Journals:Zhu2021_RAM}
J.~Zhu, A.~Cherubini, C.~Dune, D.~Navarro-Alarcon, \emph{et~al.}, ``Challenges
  and outlook in robotic manipulation of deformable objects,'' \emph{{IEEE}
  Robot. Autom. Magazine (under review)}, pp. 1--11, 2021.

\bibitem{tokumoto2002deformation}
S.~Tokumoto and S.~Hirai, ``Deformation control of rheological food dough using
  a forming process model,'' in \emph{Proceedings 2002 IEEE International
  Conference on Robotics and Automation (Cat. No. 02CH37292)}, vol.~2.\hskip
  1em plus 0.5em minus 0.4em\relax IEEE, 2002, pp. 1457--1464.

\bibitem{han2020vision}
L.~Han, H.~Wang, Z.~Liu, W.~Chen, and X.~Zhang, ``Vision-based cutting control
  of deformable objects with surface tracking,'' \emph{IEEE/ASME Transactions
  on Mechatronics}, 2020.

\bibitem{park2005static}
E.~J. Park and J.~K. Mills, ``Static shape and vibration control of flexible
  payloads with applications to robotic assembly,'' \emph{IEEE/ASME
  Transactions on Mechatronics}, vol.~10, no.~6, pp. 675--687, 2005.

\bibitem{2011Bringing}
M.~Cusumano-Towner, A.~Singh, S.~Miller, J.~F. O'Brien, and P.~Abbeel,
  ``Bringing clothing into desired configurations with limited perception,'' in
  \emph{IEEE International Conference on Robotics \& Automation}, 2011.

\bibitem{Wang2018Adaptive}
H.~Wang, B.~Yang, J.~Wang, X.~Liang, \emph{et~al.}, ``Adaptive visual servoing
  of contour features,'' \emph{IEEE/ASME Transactions on Mechatronics}, pp.
  1--1, 2018.

\bibitem{navarro2016automatic}
D.~Navarro-Alarcon, H.~M. Yip, Z.~Wang, Y.-H. Liu, \emph{et~al.}, ``Automatic
  3-d manipulation of soft objects by robotic arms with an adaptive deformation
  model,'' \emph{IEEE Transactions on Robotics}, vol.~32, no.~2, pp. 429--441,
  2016.

\bibitem{2004Review}
D.~Zhang and G.~Lu, ``Review of shape representation and description
  techniques,'' \emph{Pattern Recognition}, vol.~37, no.~1, pp. 1--19, 2004.

\bibitem{Navarro2018Fourier}
D.~Navarro-Alarcon and Y.-H. Liu, ``Fourier-based shape servoing: A new
  feedback method to actively deform soft objects into desired {2D} image
  shapes,'' \emph{{IEEE Trans. Robot.}}, vol.~34, no.~1, pp. 272--1279, 2018.

\bibitem{qi2020adaptive}
J.~Qi, W.~Ma, D.~Navarro-Alarcon, H.~Gao, and G.~Ma, ``Adaptive shape servoing
  of elastic rods using parameterized regression features and auto-tuning
  motion controls,'' \emph{arXiv preprint arXiv:2008.06896}, 2020.

\bibitem{hu1962visual}
M.-K. Hu, ``Visual pattern recognition by moment invariants,'' \emph{IRE
  transactions on information theory}, vol.~8, no.~2, pp. 179--187, 1962.

\bibitem{2009Error}
W.~Chen, M.~Li, and X.~Su, ``Error analysis about ccd sampling in fourier
  transform profilometry,'' \emph{Optik - International Journal for Light and
  Electron Optics}, vol. 120, no.~13, pp. 652--657, 2009.

\bibitem{alambeigi2018autonomous}
F.~Alambeigi, Z.~Wang, R.~Hegeman, Y.-H. Liu, and M.~Armand, ``Autonomous
  data-driven manipulation of unknown anisotropic deformable tissues using
  unmodelled continuum manipulators,'' \emph{IEEE Robotics and Automation
  Letters}, vol.~4, no.~2, pp. 254--261, 2018.

\bibitem{2020Automatic}
R.~Lagneau, A.~Krupa, and M.~Marchal, ``Automatic shape control of deformable
  wires based on model-free visual servoing,'' \emph{IEEE Robotics and
  Automation Letters}, vol.~PP, no.~99, pp. 1--1, 2020.

\bibitem{Hu20193}
Z.~Hu, T.~Han, P.~Sun, J.~Pan, and D.~Manocha, ``3-d deformable object
  manipulation using deep neural networks,'' \emph{IEEE Robotics and Automation
  Letters}, vol.~4, no.~4, pp. 4255--4261, 2019.

\bibitem{li2018vision}
X.~Li, X.~Su, and Y.-H. Liu, ``Vision-based robotic manipulation of flexible
  pcbs,'' \emph{IEEE/ASME Transactions on Mechatronics}, vol.~23, no.~6, pp.
  2739--2749, 2018.

\bibitem{zhu2020vision}
J.~Zhu, D.~Navarro-Alarcon, R.~Passama, and A.~Cherubini, ``Vision-based
  manipulation of deformable and rigid objects using subspace projections of 2d
  contours,'' \emph{Robotics and Autonomous Systems}, 2021, accepted.

\bibitem{han2021visual}
L.~Han, H.~Wang, Z.~Liu, W.~Chen, and X.~Zhang, ``Visual tracking control of
  deformable objects with a fat-based controller,'' \emph{IEEE Transactions on
  Industrial Electronics}, 2021.

\bibitem{cherubini2021_frontneuro}
A.~Cherubini and D.~Navarro-Alarcon, ``Sensor-based control for human-robot
  collaboration: Fundamentals, challenges and opportunities,'' \emph{Front
  Neurorobotics}, vol.~14, p. 113, 2021.

\bibitem{2018An}
S.~Li, A.~Ghasemi, W.~Xie, and Y.~Gao, ``An enhanced ibvs controller of a 6dof
  manipulator using hybrid pd-smc method,'' \emph{International Journal of
  Control Automation and Systems}, vol.~16, no.~3, 2018.

\bibitem{2020Toward}
Y.~Chang, L.~Li, Y.~Wang, and K.~You, ``Toward fast convergence and
  calibration-free visual servoing control: A new image based uncalibrated
  finite time control scheme,'' \emph{IEEE Access}, vol.~PP, no.~99, pp. 1--1,
  2020.

\bibitem{ahmadi2021robust}
B.~Ahmadi, W.-F. Xie, and E.~Zakeri, ``Robust cascade vision/force control of
  industrial robots utilizing continuous integral sliding mode control
  method,'' \emph{IEEE/ASME Transactions on Mechatronics}, 2021.

\bibitem{li2010sliding}
F.~Li and H.-L. Xie, ``Sliding mode variable structure control for visual
  servoing system,'' \emph{International Journal of Automation and Computing},
  vol.~7, no.~3, pp. 317--323, 2010.

\bibitem{navarro2013model}
D.~Navarro-Alarcon, Y.-H. Liu, J.~G. Romero, and P.~Li, ``Model-free visually
  servoed deformation control of elastic objects by robot manipulators,''
  \emph{IEEE Transactions on Robotics}, vol.~29, no.~6, pp. 1457--1468, 2013.

\bibitem{van2016finite}
M.~Van, S.~S. Ge, and H.~Ren, ``Finite time fault tolerant control for robot
  manipulators using time delay estimation and continuous nonsingular fast
  terminal sliding mode control,'' \emph{IEEE transactions on cybernetics},
  vol.~47, no.~7, pp. 1681--1693, 2016.

\bibitem{polycarpou1993robust}
M.~M. Polycarpou and P.~A. Ioannou, ``A robust adaptive nonlinear control
  design,'' in \emph{1993 American control conference}.\hskip 1em plus 0.5em
  minus 0.4em\relax IEEE, 1993, pp. 1365--1369.

\bibitem{lei2012fusion}
H.~L. Lei and S.~Yang, ``Fusion of image contour moments and fourier
  descriptors for the hand gesture recognition,'' in \emph{Advanced Materials
  Research}, vol. 433.\hskip 1em plus 0.5em minus 0.4em\relax Trans Tech Publ,
  2012, pp. 5188--5192.

\bibitem{zheng2019toward}
D.~Zheng, H.~Wang, J.~Wang, X.~Zhang, and W.~Chen, ``Toward visibility
  guaranteed visual servoing control of quadrotor uavs,'' \emph{IEEE/ASME
  Transactions on Mechatronics}, vol.~24, no.~3, pp. 1087--1095, 2019.

\bibitem{utkin1977variable}
V.~Utkin, ``Variable structure systems with sliding modes,'' \emph{IEEE
  Transactions on Automatic control}, vol.~22, no.~2, pp. 212--222, 1977.

\bibitem{2015AdaptiveNN}
J.~Ma, S.~S. Ge, Z.~Zheng, and D.~Hu, ``Adaptive nn control of a class of
  nonlinear systems with asymmetric saturation actuators,'' \emph{IEEE
  Transactions on Neural Networks and Learning Systems}, vol.~26, no.~7, pp.
  1532--1538, 2015.

\bibitem{qu1998robust}
Z.~Qu, \emph{Robust control of nonlinear uncertain systems}.\hskip 1em plus
  0.5em minus 0.4em\relax John Wiley \& Sons, Inc., 1998.

\bibitem{yang2011nonsingular}
L.~Yang and J.~Yang, ``Nonsingular fast terminal sliding-mode control for
  nonlinear dynamical systems,'' \emph{International Journal of Robust and
  Nonlinear Control}, vol.~21, no.~16, pp. 1865--1879, 2011.

\bibitem{2019Finite}
C.~Xue, X.~Yu, W.~He, and C.~Sun, ``Finite-time neural impedance control for an
  uncertain robotic manipulator,'' in \emph{2019 34rd Youth Academic Annual
  Conference of Chinese Association of Automation (YAC)}, 2019.

\bibitem{feng2002non}
Y.~Feng, X.~Yu, and Z.~Man, ``Non-singular terminal sliding mode control of
  rigid manipulators,'' \emph{Automatica}, vol.~38, no.~12, pp. 2159--2167,
  2002.

\bibitem{van2008visualizing}
L.~Van~der Maaten and G.~Hinton, ``Visualizing data using t-sne.''
  \emph{Journal of machine learning research}, vol.~9, no.~11, 2008.

\bibitem{hosoda1994versatile}
K.~Hosoda and M.~Asada, ``Versatile visual servoing without knowledge of true
  jacobian,'' in \emph{Proceedings of IEEE/RSJ International Conference on
  Intelligent Robots and Systems (IROS'94)}, vol.~1.\hskip 1em plus 0.5em minus
  0.4em\relax IEEE, 1994, pp. 186--193.

\bibitem{qian2002online}
J.~Qian and J.~Su, ``Online estimation of image jacobian matrix by kalman-bucy
  filter for uncalibrated stereo vision feedback,'' in \emph{Proceedings 2002
  IEEE International Conference on Robotics and Automation (Cat. No.
  02CH37292)}, vol.~1.\hskip 1em plus 0.5em minus 0.4em\relax IEEE, 2002, pp.
  562--567.

\end{thebibliography}
\bibliographystyle{IEEEtran}

\begin{IEEEbiography}
[{\includegraphics[width=1in,height=1.25in,clip,keepaspectratio]{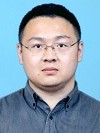}}]{Jiaming Qi}
received the M.Sc. in Integrated Circuit Engineering from Harbin Institute of Technology,
Harbin, China, in 2018. 
In 2019, he was a visiting PhD student at The Hong Kong Polytechnic University.
He is currently pursuing the Ph.D. degree with Control Science and Engineering, Harbin Institute of Technology, Harbin, China. 
His current research interests include data-driven control for soft object manipulation, vision-servoed control, robotics and control theory.
\end{IEEEbiography}

\begin{IEEEbiography}
[{\includegraphics[width=1in,height=1.25in,clip,keepaspectratio]{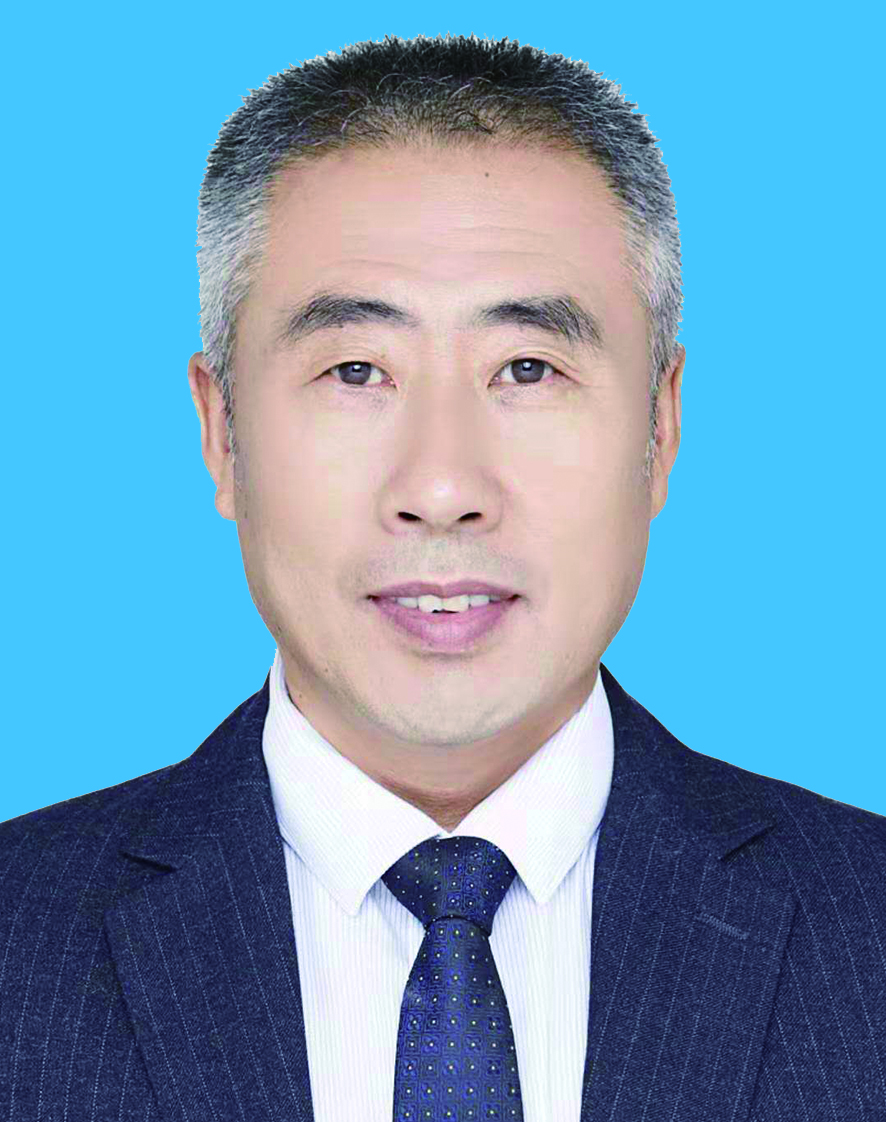}}]{Guangfu Ma}
received the Ph.D. and M.S. degrees in electrical engineering from the Harbin Institute
of Technology, Harbin, China, in 1993 and 1987, respectively. 
He was with the Harbin Institute of Technology, where he became an Associate Professor
in 1992, and a Professor in 1997, where he currently teaches and performs research in optimal control, spacecraft attitude control, and aerospace control systems. 
He is currently a Professor with the Department of Control Science and Engineering,
Harbin Institute of Technology.
\end{IEEEbiography}

\begin{IEEEbiography}
[{\includegraphics[width=1in,height=1.25in,clip,keepaspectratio]{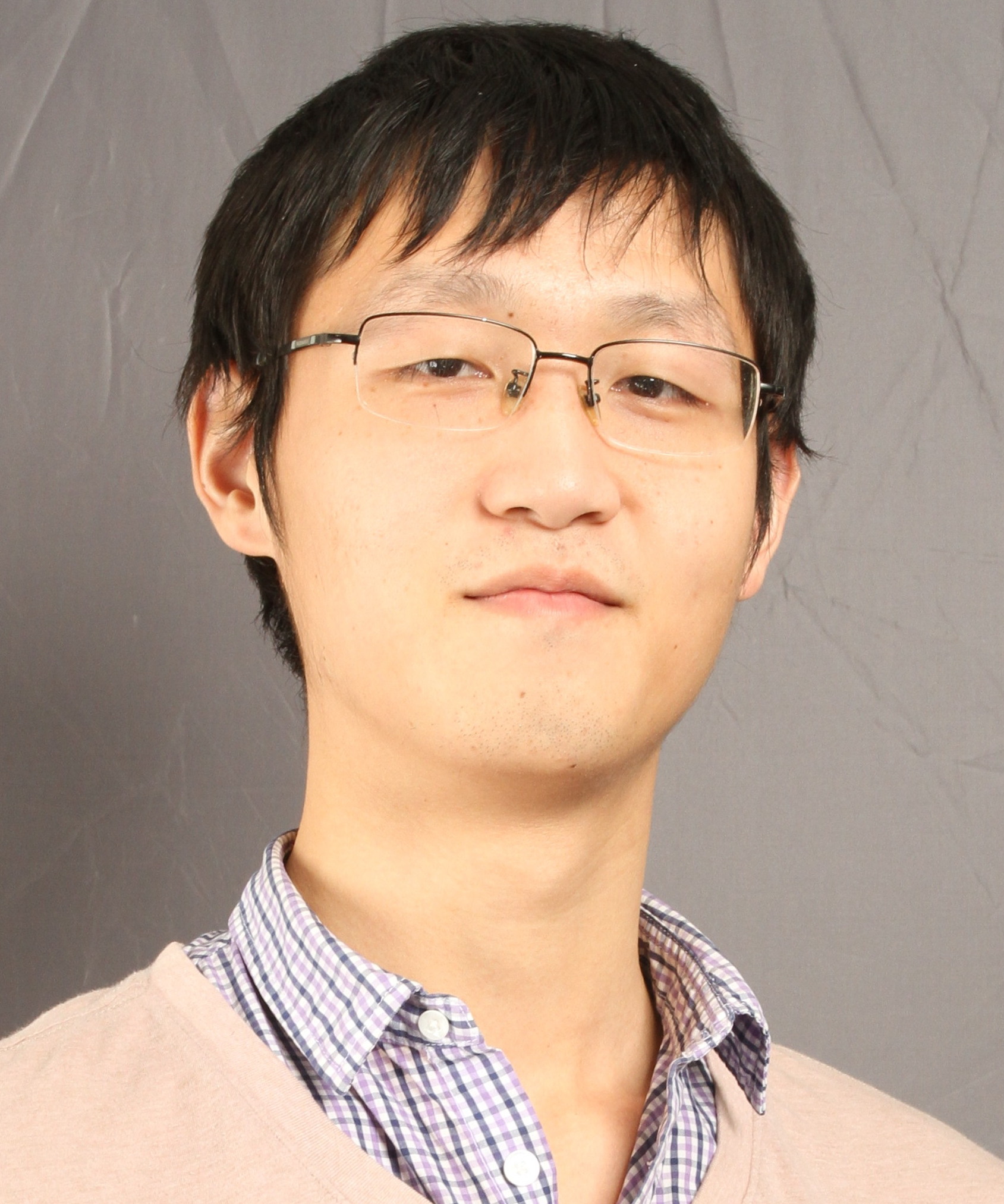}}]{Jihong Zhu}
received the M.Sc. in Systems and Control in 2015 from TU Delft and the Ph.D. in Robotics from University of Montpellier in 2020.
He conducted his doctoral research at LIRMM, France, and in 2019, he visited The Hong Kong Polytechnic University. 
He is currently a postdoc at Cognitive Robotics department, TU Delft and Honda Research Institute, Europe.
His research interests include: robot learning and manipulation of soft objects.
\end{IEEEbiography}

\begin{IEEEbiography}
[{\includegraphics[width=1in,height=1.25in,clip,keepaspectratio]{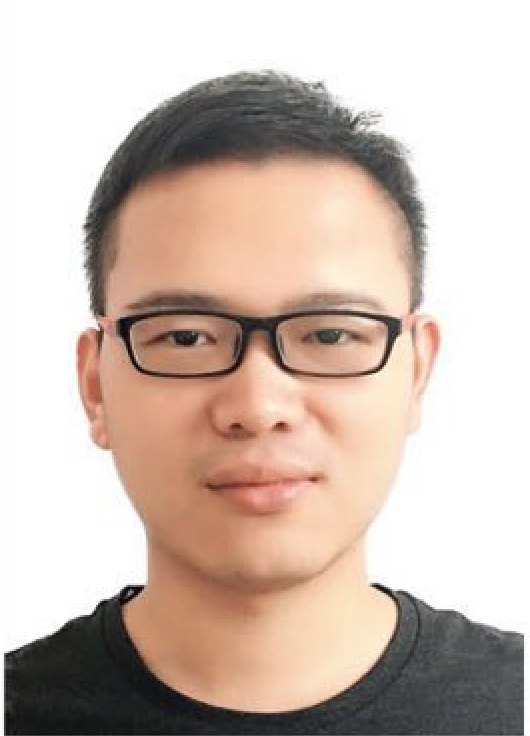}}]{Peng Zhou}
(S'20) was born in China. He received the M.Sc. degree in software engineering from the School of Software Engineering, Tongji University, Shanghai, China, in 2017 and is currently pursuing his the Ph.D. degree in robotics in The Hong Kong Polytechnic University, Kowloon, Hong Kong. His research interests include deformable object manipulation, motion planning and robot learning.
\end{IEEEbiography}

\begin{IEEEbiography}
[{\includegraphics[width=1in,height=1.25in,clip,keepaspectratio]{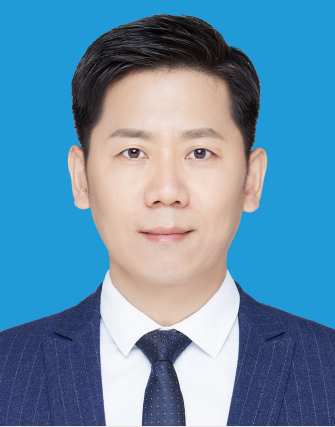}}]{Yueyong Lyu}
is now an associate research fellow in the Department of Control Science and Engineering, Harbin Institute of Technology, China. He also got his Bachelor, Master and Ph. D degree from Harbin Institute of Technology respectively in 2002, 2008 and 2013. His interests of research mainly focus on spacecraft guidance, navigation and control, especially in spacecraft formation flying, on-orbit service, and so on.
\end{IEEEbiography}

\begin{IEEEbiography}
[{\includegraphics[width=1in,height=1.25in,clip,keepaspectratio]{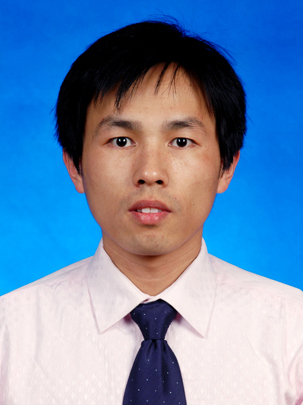}}]{Haibo Zhang}
received the Ph.D and M.S. degrees in control science and engineering from the Harbin Institute of Technology, Harbin, China, in 2013 and 2009, respectively. He performs research in space manipulator dynamics and compliance capture control, and spacecraft relative motion control. He is currently a senior engineer with Beijing Institute of Control Engineering.
\end{IEEEbiography}

\begin{IEEEbiography}
[{\includegraphics[width=1in,height=1.25in,clip,keepaspectratio]{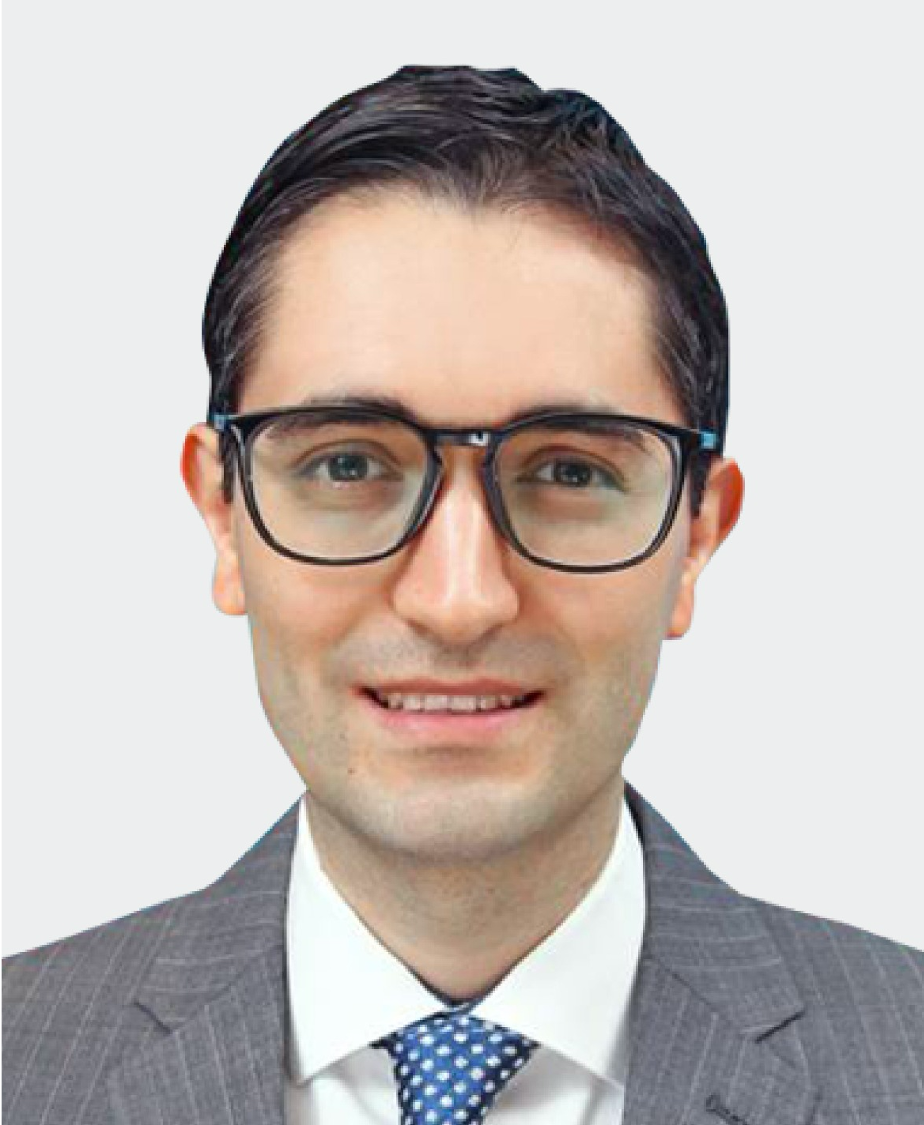}}]{David Navarro-Alarcon}
(GS’06–M’14–SM’19) received the Ph.D. degree in mechanical and automation engineering from The Chinese University of Hong Kong (CUHK), Shatin, Hong Kong, in 2014.

Since 2017, he has been with The Hong Kong Polytechnic University (PolyU), Hung Hom, Hong Kong, where he is an Assistant Professor at the Department of Mechanical Engineering, Principal Investigator of the Robotics and Machine Intelligence Laboratory, and Investigator at the Research Institute for Smart Ageing.
Before joining PolyU, he was a Postdoctoral Fellow and then Research Assistant Professor at the CUHK T Stone Robotics Institute, from 2014 to 2017. He has had visiting appointments at the University of Toulon in France and the Technical University of Munich in Germany.
His current research interests include perceptual robotics and control theory.
\end{IEEEbiography}

\end{document}